\documentclass[lettersize,journal]{IEEEtran}
\usepackage[
colorlinks=true,    
linkcolor=blue,     
citecolor=green,    
filecolor=magenta,  
urlcolor=cyan,      
pdfborder={0 0 0},  
bookmarks=true,     
pdfauthor={Your Name},
pdftitle={Your Title},
]{hyperref}

\IEEEoverridecommandlockouts
\usepackage{amsmath,amssymb,amsthm,amsfonts,subcaption,caption}
\usepackage{algorithm}
\usepackage{algorithmicx}%
\usepackage{algpseudocode}%
\usepackage{graphicx}
\usepackage{textcomp}
\usepackage{xcolor}
\usepackage{multirow}
\usepackage{multicol}
\usepackage{hyperref}
\usepackage{booktabs} 
\usepackage{threeparttable}   
\usepackage[numbers,sort&compress]{natbib}
\usepackage{cleveref}
\usepackage{float}
\usepackage{url}
\usepackage{enumitem}
\usepackage{epstopdf}
\usepackage[T1]{fontenc}
\usepackage[numbers,sort&compress]{natbib}

\newtheorem{theorem}{Theorem}

\newtheorem{definition}{Definition}
\newtheorem{assumption}{Assumption}
\newtheorem{lemma}{Lemma}
\newtheorem{remark}{Remark}
\newtheorem{example}{Example}
\newtheorem{setup}{Setup}

\newtheorem{basic-eq}{}

\def\g{\mathbf{g}}

\def\w{\mathbf{w}}
\def\h{\mathbf{h}}
\def\v{\mathbf{v}}

\def\u{\mathbf{u}}

\def\W{\mathbf{W}}

\def\B{\mathbf{B}}
\def\D{\mathbf{D}}
\def\E{\mathbb{E}}
\def\G{\mathbf{G}}
\def\H{\mathbf{H}}
\def\I{\mathbf{I}}
\def\N{{N}}
\def\P{\mathbf{P}}

\def\V{\mathbf{V}}

\def\vik{\overline{\v}_i^k}

\def\nl{\left\|}
\def\nr{\right\|}
\def\dsum{\displaystyle \sum}

\graphicspath{ {./Figures/} }
\makeatletter
\begin{document}
	\title{Decentralized Federated Learning by Partial Message Exchange}
	
	\author{Shan Sha, Shenglong Zhou, Xin Wang, Lingchen Kong, Geoffrey Ye Li, \emph{IEEE Fellow}
	\thanks{This work was supported by  the Scientific Research Innovation Capability Support Project for Young Faculty (No. SRICSPYF-ZY2025173) and the Fundamental Research Funds for the Central Universities (No. 2026JBZD001 and  No. 2023XKRC049) }%
	\thanks{Shan Sha, Shenglong Zhou, Xin Wang, and Lingchen Kong are with the School of Mathematics and Statistics, Beijing Jiaotong University, Beijing, China. E-mail: \{shansha, shlzhou, xinwang3, lchkong\}@bjtu.edu.cn. Geoffrey Ye Li are with the Department of Electrical and Electronic Engineering, Faculty of Engineering, Imperial College London, London, U.K. E-mail: geoffrey.li@imperial.ac.uk.}%
	\thanks{{\it Corresponding author: Shenglong Zhou}.}
	}

	\markboth{Journal of \LaTeX\ Class Files,~Vol.~14, No.~8, August~2021}%
	{Shell \MakeLowercase{\textit{et al.}}: A Sample Article Using IEEEtran.cls for IEEE Journals}
	
	\IEEEpubid{}
	
	\maketitle
	
	\begin{abstract}
%

	Decentralized federated learning (DFL) has emerged as a transformative server-free paradigm that enables collaborative learning over large-scale heterogeneous networks. However, it continues to face fundamental challenges, including data heterogeneity, restrictive assumptions for theoretical analysis, and degraded convergence when standard communication- or privacy-enhancing techniques are applied. To overcome these drawbacks, this paper develops a novel algorithm, PaME (DFL by Partial Message  Exchange). The central principle is to allow only randomly selected sparse coordinates to be exchanged between two neighbor nodes. 
	As a result, PaME  significantly reduces  communication costs while simultaneously limiting the exposure of data-sensitive information during transmission. The latter property is rigorously characterized by a formal reconstruction-risk theory under partial observation.
	 Moreover, the algorithm is proven to converge  in expectation to a stationary point at a linear rate, provided that the gradient is locally Lipschitz continuous and the communication matrix is doubly stochastic. These two mild assumptions not only dispenses with many restrictive conditions commonly imposed by existing DFL methods but also enables PaME to effectively address data heterogeneity. Furthermore, comprehensive numerical experiments demonstrate its superior performance compared with several representative decentralized learning algorithms.
	\end{abstract}
	

		\section{Introduction}
\IEEEPARstart{D}{ecentralized} federated learning (DFL) has emerged as a popular paradigm driven by advances in communication technologies and distributed optimization. It is widely regarded as a fundamental framework for peer-to-peer, large-scale collaborative learning and has been systematically reviewed in several recent surveys \cite{beltran2023decentralized, hallaji2024decentralized}. Its applications span a wide range of scenarios, including the Internet of Things \cite{alabadi2024innovative}, edge computing \cite{wang2024decentralized}, and privacy-sensitive industrial domains such as energy \cite{nuvvula2024federated} and healthcare \cite{alsamhi2024federated}. 

A central goal of federated learning is to balance communication efficiency, privacy, and model accuracy.  Compared to centralized federated learning (CFL), DFL eliminates the need for a central server, which otherwise constitutes a communication bottleneck and a single point of failure. Moreover, DFL can operate over diverse and possibly time-varying communication topologies, enabling the modeling of heterogeneous and dynamic environments. From a graph-theoretic perspective, CFL can be viewed as a special case of DFL with a star-shaped communication network/graph. These properties generally make DFL more communication-efficient than CFL, especially in large-scale or highly heterogeneous networks. 
From a privacy perspective, the decentralized setting provides a natural advantage over CFL because it does not require all participants to trust a single central server, which represents a critical vulnerability during training.  In contrast, DFL mitigates this risk through its flexible topology: even if an individual node in the network is corrupted, it can be easily isolated (e.g., taken offline), and moderately affects  the overall training process. Finally, according to \cite{lian2017can,shi2023improving},  DFL can empirically achieve convergence performance in terms of accuracy comparable to, or in some cases superior to, CFL. 
	
	Consequently, DFL offers a more favorable communication efficiency-privacy-accuracy trade-off than CFL. However, this improved architecture comes at the cost of additional complexity in both algorithm design and performance analysis. Requiring no central server implies the absence of a unique global solution at each iteration, and as a result, necessitates the management of local parameters for each node. This, in turn, complicates both the theoretical analysis and robust privacy guarantees. In this work, we propose a DFL algorithm with a novel message exchange mechanism that improves trade-off among communication efficiency, privacy, and accuracy.
	
\subsection{Related Work}
A typical iteration of DFL algorithms alternates between local updates and parameter exchange with neighbor nodes. For instance, as a representative method, D-PSGD~\cite{lian2017can} combines local stochastic gradient descent (SGD) with neighbor averaging and achieves convergence rates comparable to centralized approaches under some ideal conditions  while avoiding server-induced communication bottlenecks. Subsequently, a number of variants have been proposed, mainly by modifying three core components: neighbor selection, parameter aggregation, and local update. We provide a structured review according to these three stages.
	
	\subsubsection{Neighbor Selection}
	DFL algorithms rely on an underlying communication topology that specifies with which nodes information is exchanged. Common choices of topology include fully connected, ring, and grid networks \cite{zhu2022topology}. Moreover, dynamic communication graphs can be incorporated to enhance flexibility and performance, such as  stochastic communication among nodes via a sampling-based neighbor selection scheme  \cite{zhang2019decentralized},  epidemic learning to accelerate model convergence \cite{de2023epidemic}, and  an adaptive method to identify helpful neighbors \cite{wang2024smart}. 
  In general, highly connected networks facilitate faster consensus but incur higher communication overhead, whereas sparse or random topologies reduce the communication overhead at the risk of slower convergence. Furthermore, the neighbor-selection strategy can affect the strength of privacy preservation \cite{zhang2025personalized}. 
	
	\subsubsection{Parameter Aggregation}
	After receiving parameters from neighboring nodes, the next step is to efficiently exploit this information to strengthen consensus and accelerate global convergence. A common strategy is to perform a weighted average over neighboring nodes \cite{nedic2009distributed}. Moreover, parameter aggregation is typically the primary step at which communication occurs; thus, various data-compression strategies can be applied to the transmitted signals (e.g., model differences, gradients, or consensus variables) to reduce communication volume. Examples include difference compression \cite{tang2018communication}, gradient sparsification \cite{alistarh2018convergence}, and one-bit compression \cite{11145909}. To deal with transmission failures, a completion strategy has been proposed in \cite{ye2022dfl} by replacing the lost packets with local parameters at each device.
	
	Additionally, parameter aggregation is a natural stage for incorporating privacy-preserving mechanisms, which can be broadly categorized into encryption-based and differential privacy (DP)-based approaches. The former leverages secure computation and cryptographic techniques, such as secure aggregation and homomorphic encryption \cite{qian2024decentralized,saidi2025securing}. The latter is favored for its formal privacy guarantees and flexibility \cite{wang2023decentralized,yue2025differentially}. However, the DP framework typically requires noise injection, which inevitably introduces an accuracy-privacy trade-off. Therefore, recent work aims to balance this trade-off and mitigate the resulting utility degradation \cite{fukami2024dp}.
	
	\subsubsection{Local Update}
	Following the communication phase, each node performs one or more local updates using its own data. One of the most common frameworks is SGD: each node computes a local stochastic gradient based on its current parameter (usually after aggregation) and updates its model accordingly. Another important class of methods is based on the alternating direction method of multiples (ADMM)-type steps \cite{zhou2023federated,zhou2025preconditioned}, which have been employed in CFL \cite{zhou2023federated}. However, in decentralized settings, the higher per-iteration complexity has limited their practice. 
	Moreover, several variants of SGD have been developed to accelerate convergence and reduce communication overhead, such as,  momentum terms integrated into local updates to improve stability under nonconvex landscapes \cite{sun2022decentralized} and multiple local updates between communication rounds to reduce the overall communication cost \cite{liu2022decentralized}.

 \begin{table}[t]
		\centering
		\caption{Assumptions for convergence: \textcircled{1}: Convexity; 
			\textcircled{2}: Strong convexity;  
			\textcircled{3}: Lipschitz continuity of the gradient; \textcircled{4}: Bounded (stochastic) gradient or the second-moment of the (stochastic) gradient; 
			\textcircled{5}: Bounded variance; 
			\textcircled{6}: Lipschitz continuity of the gradient on a bounded region; \textcircled{7}: Unbiased gradient estimation; 
			\textcircled{8}: Unbiased compression operator;
			\textcircled{9}: Lipschitz continuity of the objective function;
			\textcircled{10}: Bounded compression-error operator;
			\textcircled{11}: Doubly stochastic communication matrix and spectral gap;
			\textcircled{12}: PL condition.\\
			Techniques for communication efficiency: \textcircled{A}:  Data compression;
			\textcircled{B}: Communication round reduction.
			\label{tab:comp-assumps}}
		\renewcommand{\arraystretch}{1.25}\addtolength{\tabcolsep}{-3pt}
		\begin{tabular}{lcccr}
			\hline  
			Algorithms&  Refs.  & Communication & Convergence rate& Assumptions \\ \hline  
			\multicolumn{5}{c}{Type I convergence: $\|\nabla f(\w^T)\|^2 = B$}\\\hline  		
			D-PSGD& \cite{lian2017can}  &none& $O(1/\sqrt{T})$
			 &\textcircled{3}\textcircled{5}\textcircled{11} \\ 
			 CHOCO-SGD& \cite{koloskovadecentralized}  & \textcircled{A}& $O(1/{T^{\frac{2}{3}}})$
			 &\textcircled{3}\textcircled{4}\textcircled{5}\textcircled{10}\textcircled{11} \\ 
			SQuARM-SGD& \cite{singh2021squarm}  & \textcircled{A}& $O(1/{\sqrt{T}})$
			&\textcircled{3}\textcircled{4}\textcircled{5}\textcircled{7}\textcircled{10}\textcircled{11} \\ 
			DFedSAM & \cite{shi2023improving} & \textcircled{B}& $O(1/\sqrt{T})$ &	\textcircled{3}\textcircled{4}\textcircled{5}\textcircled{11}\\		
			EL & \cite{de2023epidemic} & \textcircled{B} & $O(1/\sqrt{T})$ &			\textcircled{3}\textcircled{4}\textcircled{5}\textcircled{7}\\
			DCD-PSGD & \cite{tang2018communication}  & \textcircled{A} & $O(1/\sqrt{T})$ &	
			\textcircled{3}\textcircled{4}\textcircled{5}\textcircled{8}\textcircled{11}\\
			BEER & \cite{zhao2022beer} & \textcircled{A} & $O(1/T)$ &
			\textcircled{3}\textcircled{5}\textcircled{7}\textcircled{10}\textcircled{11}\\
			PaME & ours & \textcircled{A}\textcircled{B}  & $O(\varrho^T)$ &
			\textcircled{6}\textcircled{11}\\
			 \hline 
			\multicolumn{5}{c}{Type II convergence: $f(\w^T)-f^\infty = B$}\\\hline 
			DFedAvgM & \cite{sun2022decentralized} & \textcircled{A}& $O(1/T)$&\textcircled{3}\textcircled{4}\textcircled{5}\textcircled{11}\textcircled{12}\\
			DMGT-SVRG & \cite{liu2024decentralized} & \textcircled{B}& $O(\rho^T)$&\textcircled{2}\textcircled{7}\textcircled{9}\textcircled{11}\\
			BEER & \cite{zhao2022beer} & \textcircled{A} & $O(\rho^T)$ &
			\textcircled{3}\textcircled{5}\textcircled{7}\textcircled{10}\textcircled{11}\textcircled{12}\\
			NIDS & \cite{li2019decentralized} & none & $O(\rho^T)$ &
			\textcircled{2}\textcircled{3}\textcircled{11}\\
			ANQ-NIDS & \cite{michelusi2022finite} & \textcircled{A} & $O(\rho^T)$ &
			\textcircled{2}\textcircled{3}\textcircled{10}\textcircled{11}\\
			PaME & ours & \textcircled{A}\textcircled{B}  & $O(\rho^T)$ &
			\textcircled{6}\textcircled{11}\\
			\hline 
		\end{tabular}
		\label{table:compare-conditions}
	\end{table}	
		
	\subsection{Contributions}
	In this work, we address the critical challenge of balancing communication efficiency, accuracy, and privacy preservation in decentralized systems by introducing a partial message exchange (\texttt{PME}, see Fig. \ref{transmission-mech} or Algorithm \ref{algorithm-PME}) mechanism within a DFL algorithm (see Algorithm \ref{algorithm-PaME}), which hence we term PaME.  As presented in Table~\ref{table:compare-conditions}, PaME differs from existing DFL algorithms by substantially relaxing theoretical assumptions while achieving a favorable trade-off among communication efficiency, accuracy, and privacy. Our main contributions are summarized as follows.
	
	\textit{a) Best convergence under the weakest assumptions:}
	 \text PaME is proven to converge under two assumptions: \textcircled{6} Lipschitz continuity of the gradients on a bounded region and \textcircled{11}  the communication matrix to be doubly stochastic, as presented in Table~\ref{table:compare-conditions}. The former is equivalent to the locally Lipschitz continuity, and thus much weaker than \textcircled{3} the (global) Lipschitz continuity of the gradient. Moreover, we do not impose additional standard stochastic-gradient assumptions, such as bounded stochastic gradients or variance, which are typically invoked to control stochasticity and client variability. Such a success lies in the establishment of the boundedness of the iterations generated by PaME from a deterministic perspective, despite the algorithm being a stochastic-gradient-based method. Consequently, PaME can effectively handle applications with heterogeneous data.
	 
	More importantly, the convergence result reflects the central design principle of PaME: improving the trade-off among communication efficiency, privacy-related protection, and data heterogeneity. Specifically, the \texttt{PME} mechanism and time-varying neighbor participation reduce the transmitted information, the induced partial-observation structure leads to a reconstruction-risk characterization, and the penalty term explicitly controls the discrepancy among heterogeneous local models. These three aspects are not treated as separate add-on mechanisms; instead, they are jointly incorporated into the PaME dynamics, under which the convergence guarantee is established without introducing additional standard stochastic-gradient or compression assumptions. 
	
Table~\ref{table:compare-conditions} summarizes the Type-I and Type-II convergence rates achieved by various existing algorithms. Type-I convergence refers to the decay of the objective function value, i.e., $f(\w^k)-f^\infty\to 0$,  while Type-II convergence reflects the vanishing of the gradient norm, i.e., $\|\nabla f(\w^T)\|\to 0$, where $T$ is the total iteration number and $f^\infty$ is the optimal value or the limiting point of sequence $\{f(\w^k)\}$. 
Under two assumptions  \textcircled{6} and \textcircled{11}, PaME attains a linear rate for both types of convergence, i.e., $O(\varrho^T)$ and $O(\rho^T)$ with ${\varrho, \rho\in(0,1)}$. This is faster than  the sub-linear rate of $O(1/T)$.
	For Type-I convergence, the best rate was achieved by BEER \cite{zhao2022beer}, but it is only sub-linear and requires several stronger assumptions,  such as \textcircled{3}, \textcircled{5} bounded variance,  \textcircled{7} unbiased gradient estimation,  \textcircled{10} bounded compression-error operator, and \textcircled{11}. 
 For Type-II convergence, to the best of our knowledge, existing linear convergence results for DFL typically rely on \textcircled{2} strong convexity  \cite{liu2024decentralized, li2019decentralized,michelusi2022finite} or \textcircled{12} Polyak-\L{}ojasiewicz (PL) conditions \cite{zhao2022beer}, together with additional stochastic-gradient or mixing assumptions, see Table~\ref{table:compare-conditions}.
 Finally, it is noted that both $\textcircled{3} \Rightarrow \textcircled{6}$ and $\textcircled{2} \Rightarrow \textcircled{6}$. Hence, PaME achieves the fastest known convergence rate but under weakest set of assumptions.

	Additionally, since the \texttt{PME} mechanism is intrinsically a form of data compression, our analysis does not rely on the standard unbiased-compressor assumption \cite{tang2018communication}; its effect can be directly tracked through the iterate dynamics.
	
	Finally, PaME naturally accommodates time-varying communication graphs, which result in dynamic communication matrices at each communication round. Despite this, convergence of PaME is guaranteed provided that only the initial communication matrix is doubly stochastic. This contrasts with many existing DFL methods, which typically require the communication matrices to remain static and doubly stochastic at each communication round \cite{sun2022decentralized, liu2024decentralized, zhao2022beer, li2019decentralized, michelusi2022finite}. This indicates that PaME achieves the convergence in a more complex scenario, where the communication matrices can be dynamic, sparse, and  non-doubly stochastic.

			\begin{figure}[t]
				\centering
				\includegraphics[width=.49\textwidth]{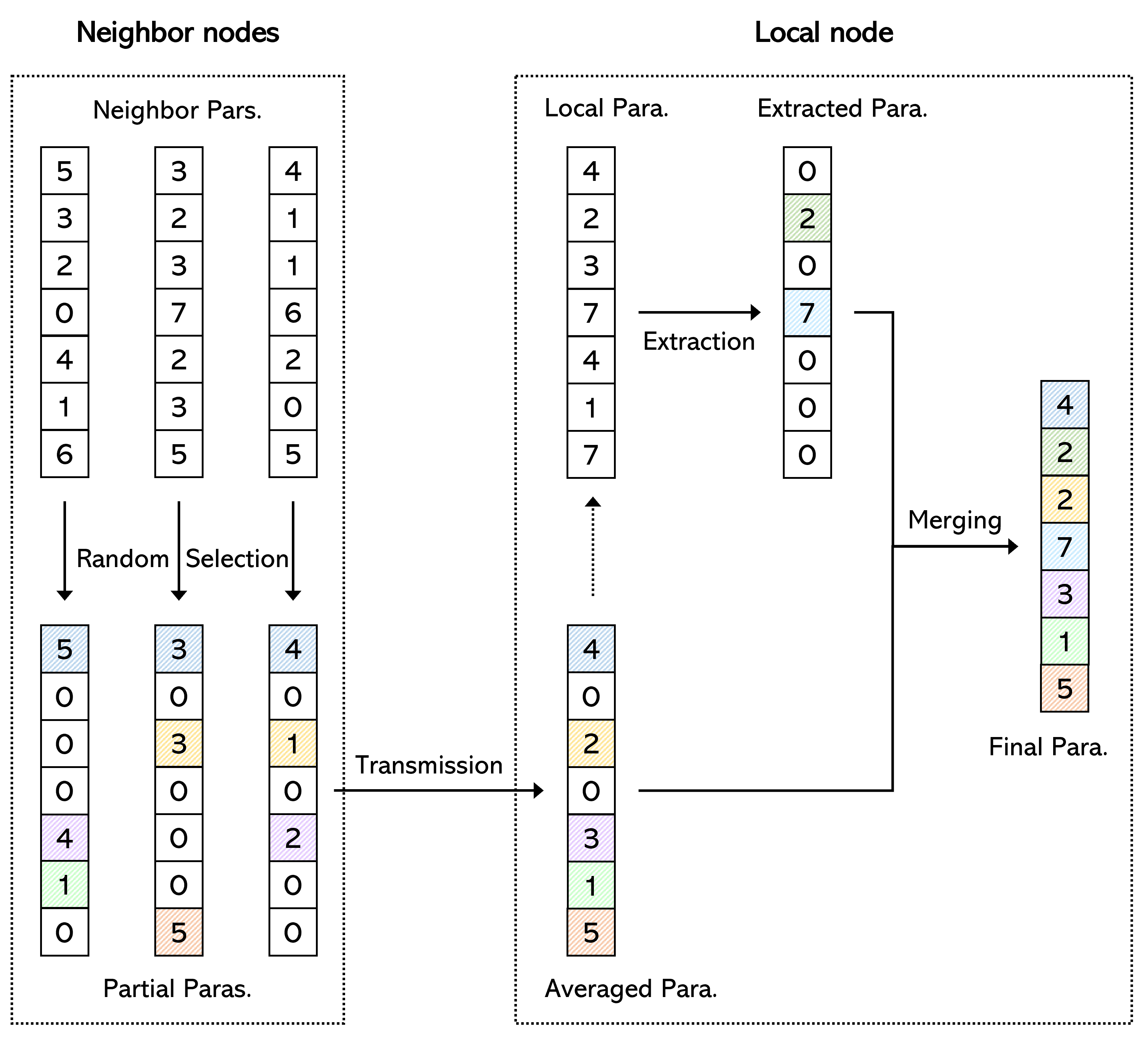}
				\caption{Partial message exchange (\texttt{PME}) : Every neighbor of a local node $i$ randomly selects partial coordinates (or messages) and transmits them to node $i$. Node $i$ averages the received incomplete parameters and fills in the missing coordinates using the coordinates in its own local parameter.}
				\label{transmission-mech} \vspace{-1mm}
			\end{figure}
			
	\textit{b) Communication efficiency:} One of the major achievements in this paper lies in the development of the \texttt{PME} mechanism, as outlined in Fig. \ref{transmission-mech}, which enables PaME to reduce the communication cost significantly. Moreover, PaME  supports asynchronous updates and multiple local steps between communication rounds, which further increases its flexibility and enables more efficient use of limited communication resources. Overall, the communication efficiency stems from two factors: reduced transmitted content per communication round and reduced communication rounds.
	
	It is worth mentioning that PaME reduces communication in a different manner in comparison with compressed DFL methods such as CHOCO-SGD\cite{koloskovadecentralized} and DCD-PSGD\cite{tang2018communication}. They are mainly built upon decentralized SGD or gossip-averaging frameworks, where compression is applied to the information exchanged during the gossip process. Specifically, CHOCO-SGD communicates compressed residuals between local variables and public copies, while DCD-PSGD communicates compressed model differences. In contrast, PaME is developed under an inexact alternating direction method (ADM)-based decentralized optimization framework, and the transmitted object is a partial model message consisting of randomly selected coordinates of the current model parameter (see Fig. \ref{transmission-mech}), rather than a stochastic gradient, a model difference, or a residual. The missing coordinates are treated as unobserved neighboring states instead of zero updates. Accordingly, PaME employs a receiver-side coordinate-wise normalization factor to construct the aggregated variable, rather than relying on an unbiased compression operator or a public-copy or error-feedback correction mechanism. Thus, while the random coordinate sampling in PME is related to existing sparsification techniques, the key distinction is its integration with the penalty-based PaME dynamics and its coordinate-wise normalized partial-state aggregation.

	\textit{c) Privacy preservation:}
The \texttt{PME} mechanism and the partial device participation strategy reduce the exposure of data-sensitive information during communication. Specifically, random neighbor selection and random coordinate sampling ensure that each node reveals only selected coordinates of its local parameter to a subset of neighboring nodes. Such randomized partial observation is consistent with the intuition that incomplete communication records make adversarial tracking or sample reconstruction more difficult \cite{zhang2020private}. To move beyond qualitative arguments, inspired by the inverse-problem perspective for analyzing data reconstruction attacks and defenses in \cite{liu2025data}, we introduce a passive partial-observation adversary model and establish a formal finite-window reconstruction-risk characterization in the DFL setting, see Section \ref{subsec:privacy-reconstruction-risk}. In particular, the information available to the adversary is represented by the partially observed data-sensitive Jacobian, and the corresponding effective observable information ratio is quantified to characterize the reconstruction difficulty induced by \texttt{PME}. Moreover, other privacy-preserving mechanisms, such as differential privacy \cite{hu2023federated} and secure aggregation \cite{biswas2024secure}, can be incorporated into our framework depending on the privacy requirements of specific applications.

	\textit{d) Robustness:} Within the PaME framework, a partial device participation strategy is employed, namely, each local node communicates only with a subset of its selected neighbors. As a result, this design mitigates the impact of stragglers because neighbors with unreliable or slow communication links (i.e., stragglers) can be contacted less frequently or omitted in specific rounds. Moreover, each local node  communicates independently according to its own interval, without being affected by other nodes' communication schedules, thereby naturally inducing a partially synchronized training regime and yielding more stable and robust learning dynamics.
	
	\textit{e) Superior numerical performance:} Extensive numerical experiments are conducted to validate the effectiveness of PaME in comparison with several established DFL algorithms on both synthetic and real-world datasets under data distributions ranging from homogeneous to highly heterogeneous settings. The experimental results demonstrate that PaME achieves a more favorable trade-off between communication efficiency and model accuracy across these scenarios, thereby confirming its robustness to data heterogeneity.
	
	\subsection{Organization}
	This paper is organized as follows. Section \ref{Preliminaries} introduces the mathematical notation and formulates the DFL model. Section \ref{DFL} details the proposed algorithm and analyzes its advantages regarding communication efficiency and privacy preservation. Section \ref{Theory} presents the theoretical analysis, including convergence guarantees and complexity analysis. Section \ref{Numerical} reports experimental results on synthetic and real datasets. Finally, Section \ref{Conclusion} provides concluding remarks.

 \section{Preliminaries}
	
	\label{Preliminaries}
	
	We begin this section by presenting the mathematical notation used throughout the paper, followed by the formulation of the optimization model for DFL.
	
	\subsection{Notation}

	Scalars, vectors, and matrices are written in lowercase, bold lowercase, and bold uppercase letters, respectively, e.g., $n$, $\sigma$,  $\sigma_i$, and $\sigma_i^k$ are scalars, $\w$, $\w_i$, $\w_i^k$, $\mathbf{g}_i$, $\v$, and $\overline{\v}_i$ are vectors, and $\W, \W^\alpha$, and $\V$ are matrices. Let ${[m] := \{1,2,\cdots,m\}}$ 
	and $\mathbb{R}^n$ be the $n$-dimensional Euclidean space equipped with inner product
	$
	{\langle \w, \v \rangle := \sum_{t=1}^n w_t v_t},
	$
	where '$:=$' means 'is defined as'. Let $\|\cdot\|$ denote the Euclidean norm for vectors and spectral norm for matrices, and $\|\cdot\|_F$ denote the Frobenius norm for matrices.
	Write $w_t, w_{it}$, and $w_{it}^k$ to represent the $t$th entry of  vectors $\w$, $\w_i$, and $\w_i^k$, respectively. 
	 The cardinality of a set $\Omega$ is denoted by $|\Omega|$. Let  $\E(\cdot)$ be the expectation operator.
	Finally, we use the compact notation
	\begin{equation*} 		
		\begin{aligned}
			\mathbf{W} &= (\mathbf{w}_1, \mathbf{w}_2,\cdots,\mathbf{w}_m), ~~~		
			&{\mathbf{V}} &= (\overline{\mathbf{v}}_1,  \overline{\mathbf{v}}_2,\cdots,\overline{\mathbf{v}}_m), \\	
			\mathbf{G} &= (\mathbf{g}_1, \mathbf{g}_2,\cdots,\mathbf{g}_m),	
			~~~  &\boldsymbol{\sigma} &= (\sigma_1,\sigma_2,\cdots,\sigma_m),	
		\end{aligned}	
	\end{equation*} 
	and apply the same convention to define $\W^\alpha$, $\V^\alpha$, $\G^\alpha$, and
	$\boldsymbol{\sigma}^\alpha$, where $\alpha$ is an integer representing the $\alpha$th iteration.

	\subsection{Decentralized Federated Learning}
	Given a network of $m$ nodes indexed by $[m]$,  each node $i$ has a private dataset $\mathcal{D}_i$ and minimizes a local loss function
	${f_i(\cdot) := f_i(\cdot;\mathcal{D}_i)}$, where ${f_i: \mathbb{R}^n \to \mathbb{R}}$ is continuously differentiable and bounded from below.
	Then the DFL optimization problem can be formulated as
	\begin{equation}\label{fl}
		\min_{\mathbf{w} \in \mathbb{R}^n}
		f(\mathbf{w}) := \sum_{i =1}^m f_i(\mathbf{w}),
	\end{equation}
	where ${\mathbf{w} \in \mathbb{R}^n}$ denotes the global parameter to be learned.
	
	Under the decentralized setting, each node maintains a local model parameter $\mathbf{w}_i$, and model consensus is achieved through peer-to-peer communication. We consider a general communication topology in which node $i$ can exchange messages only with its neighbor nodes ${j \in \N_i}$, where ${\N_i \subseteq [m]}$ denotes the neighbor set of node $i$. This communication pattern induces an undirected communication graph ${\mathrm{G} = ([m], \mathrm{E})}$, where ${\mathrm{E} := \{(i,j): j \in \N_i, i \in [m]\}}$ is the edge set. To ensure that all nodes eventually reach a common model, namely, ${\mathbf{w}_1=\cdots=\mathbf{w}_m=\mathbf{w}}$,  assume graph $\mathrm{G}$ is connected.  
	
	 Given such a communication topology, model \eqref{fl} can be equivalently reformulated as the following decentralized optimization problem,
			\begin{equation}\label{dfl}
					\min _{\mathbf{W}} ~ \sum_{i=1}^m f_i\left(\mathbf{w}_i\right), ~\text{ s.t. } ~ \mathbf{w}_i=\mathbf{w}_j, ~(i, j) \in \mathrm{E}.
			\end{equation}
	The work in this paper is carried out based on the above model.
	\section{DFL via Inexact ADM}
	\label{DFL}
	In this section, we first introduce PaME  and then highlight  several advantageous properties.

	\subsection{Algorithm Design}

			\begin{algorithm}[!t] 
 \caption{DFL by Partial Message Exchange (PaME)  \label{algorithm-PaME}}
\begin{algorithmic}[1]

	\State	Initialize {$\mathbf{w}^0_i=0$}, two integers {$ \kappa_i>0$} and {$s_i>0$}, {$\sigma_i^0>0$}, {$\gamma_i>1$} for each {$i\in[m]$}.
		
		\For{iteration $k=0,1,2,3,\cdots $}			
		\For{node $i=1,2,\cdots,m $}
		\If{$k \in\mathcal{K}_i :=\{0,\kappa_i,2\kappa_i,3\kappa_i,\cdots \}$} 
		
			\State Random neighbor selection: ${N}_i^{k}\subseteq {N}_i$. 
		\State Neighbor number update:	 ~$m_i^{k}=\left|{N}_i^k\right|$.			
		\State Partial message exchange:~ $$\vik=\texttt{PME}(\w_i^k,~\{\w_j^k: j\in{N}_i^{k}\}).$$

							
		
		\Else

		
		\State Local parameter tracking:  ~~$\vik= \w_i^k$.
		
		\State Neighbor number tracking:	 $m_i^{k}=m_i^{k-1}$.							
		
		
    \EndIf
    
    \State Random sub-batch data sampling: {$\mathcal{B}_i^{k} \subseteq \mathcal{D}_i$}.
		
		\State Local parameter updating: 
		\begin{eqnarray}\label{SRDFL-localstep}
		\mathbf{w}_{i}^{k+1}=\vik - \frac{\nabla f_i(\vik;\mathcal{B}_i^{k})}{\sigma_i^k m_i^k}.
		\end{eqnarray} 
		
		\State Hyper-parameter increasing: $\sigma_i^{k+1}=\gamma_i \sigma_i^{k}$.	
     \EndFor
 \EndFor
\end{algorithmic}
				
	\end{algorithm}
		
			Instead of solving problem (\ref{dfl}) directly, we aim to solve the following penalized formulation,
			\begin{equation*}\label{pen-model}
					\min _{\mathbf{W}} ~ \sum_{i =1}^m\Big(f_i\left(\mathbf{w}_i\right)+\frac{\sigma_i}{2} \sum_{j \in \N_i}\left\|\mathbf{w}_i-\mathbf{w}_j\right\|^2\Big),  
			\end{equation*}
			where each ${\sigma_i>0}$ is a penalty constant. According to \cite{nocedal2006numerical}, by driving ${\sigma_i}$ to $\infty$, one can penalize the constraint violations with increasing severity. Therefore, it makes sense to consider $m$ increasing sequence $\{\sigma_i^k\}$, ${i\in [m]}$ and to seek an approximate minimizer for each $k$.  Specifically, based on current point ${\mathbf{W}^k=\left(\mathbf{w}_1^k, \mathbf{w}_2^k, \cdots, \mathbf{w}_m^k\right)}$, one can find an approximate minimizer by solving the following problem,			
					\begin{equation*}\label{pen-model-1}
					\min _{\mathbf{W}} ~ \sum_{i =1}^m\Big(f_i\left(\mathbf{w}_i\right)+\frac{\sigma_i^k}{2} \sum_{j \in \N_i^k}\left\|\mathbf{w}_i-\mathbf{w}_j^k\right\|^2\Big),  
			\end{equation*}
where ${\N_i^k\subseteq\N_i}$ is a randomly selected subset. Solving the above problem is equivalent to address the following $m$ problems independently, 			 
			\begin{equation}\label{subproblem-wk1}
				\min _{\mathbf{w}_i}~ f_i\left(\mathbf{w}_i\right)+ \frac{\sigma_i^k}{2} \sum_{j \in \N_i^k}\left\|\mathbf{w}_i-\mathbf{w}_j^k\right\|^2,~~ i \in [m].
			\end{equation}
			However, it is still time-consuming to solve the above problem exactly, particularly for complex $f_i$ or high-dimensional settings. An alternative is to solve the above problem inexactly using the linearization of $f_i$ at some point. Based on this,  we consider two cases to address problem \eqref{subproblem-wk1}. 
			\begin{itemize}[leftmargin=10pt]
				\item  When ${k\in\mathcal{K}_i :=\{0,\kappa_i,2\kappa_i,3\kappa_i,\cdots \}}$, we require node ${i}$ to communicate with its selected neighbor nodes ${j\in\N_i^k}$, where   $\kappa_i$ is a positive integer and ${N}_i^{k}\subseteq {N}_i$ is a subset consisting of randomly selected neighbor nodes at the $k$th iteration.  Those neighbor nodes transmit their partial messages $\v_j^k$ defined in \eqref{sparse-generator-0}  to node $i$ based on the rule of {\tt PME} (i.e., Line $1-4$ in Algorithm \ref{algorithm-PME}). Then node $i$ averages received message ${\{\v_j^k: j\in{N}_i^{k}\}}$ and its previous parameter $\w_i^k$ to derive an aggregated parameter $\overline{\mathbf{v}}_i^k$ following the rule of \eqref{pame-0}. Then we linearize $f_i$ at $\overline{\mathbf{v}}_i^k$ and solve problem \eqref{subproblem-wk1} by  
				\begin{equation*}
					\begin{aligned}		
						&\mathbf{w}_i^{k+1}\\
						=&\operatorname*{arg\,min}_{\mathbf{w}_i }   \left\langle\nabla f_i\left(\overline{\mathbf{v}}_i^k; \mathcal{B}_i^k\right), \mathbf{w}_i\right\rangle  + \frac{\sigma_i^k}{2} \sum_{j \in N_i^k} \left\|\mathbf{w}_i-\mathbf{w}_j^k\right\|^2\\
						=&\vik-\dfrac{1}{\sigma_i^k |N_i^k|} \nabla f_i(\vik; \mathcal{B}_i^k) ,
						 					\end{aligned}
				\end{equation*}
				where $\mathcal{B}_i^k$ is a randomly selected sub-batch data from $\mathcal{D}_i$.
		
				\item  When ${k\notin\mathcal{K}_i}$ where there is no communication between node $i$ and its neighbor nodes, we linearize $f_i$ at previous point $\w_i^k$ and approximately solve problem \eqref{subproblem-wk1} by 
				\begin{equation*} 
					\begin{aligned}		
						&\w_i^{k+1}\\
						=&\operatorname*{arg\,min}_{\w_i }  \left\langle\nabla f_i\left(\w_i^k; \mathcal{B}_i^k\right), \mathbf{w}_i \right\rangle + \frac{\sigma_i^k}{2} \sum_{j \in N_i^k} \left\|\mathbf{w}_i-\mathbf{w}_i^k\right\|^2 \\
						=& \w_i^{k}-\dfrac{1}{\sigma_i^k |N_i^k|} \nabla f_i\left({\w}_i^{k}; \mathcal{B}_i^k\right).
					\end{aligned}
				\end{equation*}
In this case, we keep $N_i^k=N_i^{k-1}$.			
			\end{itemize}

			Overall, by incorporating the above two cases into the DFL framework, we develop the proposed algorithm, termed PaME (DFL \textbf{Pa}rtial \textbf{M}essage \textbf{E}xchange), as outlined in Algorithm~\ref{algorithm-PaME}. We highlight its main advantages as follows.
	
	\begin{algorithm}[!t] 
 \caption{$\vik=\texttt{PME}(\w_i^k,~\{\w_j^k: j\in{N}_i^{k}\})$  \label{algorithm-PME}}
\begin{algorithmic}[1]
 \For{every node $j\in {N}_i^{k}$} 
\State Random subset selection: ${T_j^k\subseteq[n]}$ with ${|T_j^k|= s_j}$.   

 \State Sparse vector $\v_j^k$ generation:
 \begin{equation}\label{sparse-generator-0}
v_{j\ell}^k=\begin{cases}
 {w}_{j\ell}^k, &\text{if}~\ell\in T_j^k,\\
  0, &\text{if}~\ell\in T_j^k.\\
 \end{cases}
 \end{equation}
 \State Partial message $\v_j^k$ exchange to node $i$.
 \EndFor
 
 \State  Node $i$ averages received messages by
  \begin{equation}\label{pame-0}
\overline{v}_{i\ell}^k=\begin{cases}
 \dfrac{\sum_{j\in {N}_i^{k}}v_{j\ell}^k }{\lambda_{i,\ell}^k }, &\text{if}~\lambda_{i,\ell}^k\neq0,\\[1ex]
  w_{i\ell}^k, &\text{if}~\lambda_{i,\ell}^k=0,\\
 \end{cases}
 \end{equation}
 for each $\ell\in[n]$, where 
   \begin{equation}\label{def-lambda-itk-0}
\lambda_{i,\ell}^k =\left|\left\{j\in {N}_i^{k}:~ v_{j\ell}^k \neq 0 \right\}\right|.
 \end{equation}
\end{algorithmic}
\end{algorithm}
\subsection{Partial Message Exchange (\texttt{PME})}			
			In Algorithm~\ref{algorithm-PaME}, when ${k\in\mathcal{K}_i}$,  node ${i}$ communicates with its selected neighbors via {\tt PME} outlined in Algorithm~\ref{algorithm-PME}. An illustrative example is provided in Fig.~\ref{transmission-mech}. The key features of \texttt{PME} are twofold.
\begin{itemize}[leftmargin=13pt]
\item[1)] Node ${j\in N_i^k}$  transmits a sparse vector $\v_j^k$, namely, $s_j$ randomly selected coordinates in ${\w_j^k\in\mathbb{R}^n}$ and $(n-s_j)$ zeros, to node $i$. This can lead significant  transmission volume reduction. Specifically, transmitting this sparse vector $\v_j^k$ using double-precision floating-point format requires 
\begin{equation}\label{reduced-bit}
\left(64s_j+n-s_j\right)=\left(63s_j+n\right)  \text{bits}.
\end{equation}
In contrast, transmitting a dense vector ${\w_j^k\in\mathbb{R}^n}$ with the same format generally requires $(64n)$ bits. When ${s_j\ll n}$, $${(63s_j+n ) \ll 64n}.$$ For example, $1.63\times10^4\ll 6.4\times10^5$ if ${n=100s_j=10^4}$.

We shall point out that if ${w_{j\ell}^k=0}$ for some ${\ell\in T_j^k}$ in  \eqref{sparse-generator-0} (which is unlikely to happen in practice), then to distinguish this useful $0$ and $0$ for ${\ell\notin T_k^k}$, the neighbor could use a character, such as `$\star$', to replace this useful $0$.   
			For example, at $k$th iteration,
		$$\w_5^k=\left[
3,
6,
0,
6
\right]^\top,~~T_5^k=\left\{
3,4
\right\},~~\v_5^k=\left[
0,
0,
\star,
6
\right]^\top.$$	
In $\v_5^k$, there are two useful coordinates, namely, $v_{53}^k=0$ and $v_{54}^k=6$. Using such a strategy, the transmission volume would not be affected as we only need $8$bits ($<64$bits) to transmit this character `$\star$'.

\item[2)] 	The averaging mechanism in \eqref{pame-0} is novel and  yields an unbiased estimator. It averages the $\ell$th entry using $\lambda_{i,\ell}^k$ rather than $|{N}_i^{k}|$, where $\lambda_{i,\ell}^k$ is defined in \eqref{def-lambda-itk-0} and counts the number of ${\{ \v_j^k:j\in N_i^k\}}$ with non-zero entry $v_{j\ell}^k$ (`$\star$' is treated as a  nonzero). In this way,   given ${\lambda_{i,\ell}^k>0}$, the $\ell$th entry of $\overline{\v}_i^k$  yields an unbiased estimation of the average of selected neighbors' parameters ${\{\w_{j}^k: j\in N_i^k\}}$ under some certain sampling rules, namely,
$$ \mathbb{E}\Big(\overline{v}_{i\ell}^k~|~\lambda_{i,\ell}^k>0 \Big) = \frac{1}{|N_i^k|}\sum_{j\in N_i^k}w_{j\ell}^k.$$ 
In contrast, directly averaging ${\{\v_j^k: j\in N_i^k\}}$ using $|{N}_i^{k}|$, namely $\sum_{j\in N_i^k}\v_j^k/|{N}_i^{k}|$, will yields a biased estimation, 
 as shown in the following theory whose proof can be found in Section I of the Supplemental Material.
\end{itemize}	

\begin{theorem}\label{theorem1-unbiase}
Let $\overline{\w}$ be the average of $q$ vectors $\w_1$, $\w_2$, $\cdots$, ${\w_q \in \mathbb{R}^n}$. For each ${i\in[q]}$, independently construct a sparse vector ${\v_i\in\mathbb{R}^n}$  by uniformly selecting $s$ coordinates of $\w_i$ without replacement from $[n]$. Define indicator variables by
\begin{equation*}
\delta_{i\ell}=\begin{cases}
1,&\text{if $w_{i\ell}$ is selected},\\
0,& \text{otherwise},
\end{cases}\qquad\delta_\ell:=\sum_{i=1}^q\delta_{i\ell}.
\end{equation*}
and two averages $\overline{\v}$ and $\widetilde{\v}$ by
\begin{equation*}
\overline{v}_\ell=\begin{cases}
\dfrac{1}{\delta_\ell}\displaystyle{\sum_{i=1}^q} v_{i\ell},&\text{if}~\delta_\ell>0,\\
0,& \text{otherwise},
\end{cases}\qquad\widetilde{v}_\ell=\frac{1}{q}\sum_{i=1}^q v_{i\ell}
\end{equation*}
for any $\ell\in[n]$. 
Then
\begin{equation}\label{Ev-Ew}
\mathbb{E}\left(\overline{v}_\ell~|~\delta_\ell>0\right) = \overline{w}_\ell,\quad
\mathbb{E}\left(\widetilde{v}_\ell~|~\delta_\ell>0\right)= c\overline{w}_\ell.
\end{equation}
where $c:=(s/n)/(1-(1-s/n)^q)$.
\end{theorem}
We  provide an example to illustrate the above theorem and to detail the computation of $\vik=\texttt{PME}({N}_i^{k})$ in accordance with Algorithm~\ref{algorithm-PME}.  Suppose that at $k$th iteration,	  $$ {{N}_i^{k}=\{2,4,5\}},~~{s_2=s_4=s_5=2}$$ 
and  $\w_i^k$ and $(\w_j^k, T_j^k,\v_j^k), j\in{N}_i^{k}$ are given by
$$
\begin{aligned}
\w_i^k&=\left[
2,
8,
3,
6
\right]^\top,&&&\\
\w_2^k&=\left[
2,
8,
1,
4
\right]^\top, &T_2^k=\left\{
1,4
\right\},~~~&\v_2^k=\left[
\underline{2},
0,
0,
\underline{4}
\right]^\top, \\
\w_4^k&=\left[
4,
7,
2,
5
\right]^\top,~&T_4^k=\left\{
3,4
\right\},~~~&\v_4^k=\left[
0,
0,
\underline{2},
\underline{5}
\right]^\top, \\
\w_5^k&=\left[
3,
6,
0,
6
\right]^\top,&T_5^k=\left\{
3,4
\right\},~~~&\v_5^k=\left[
0,
0,
\underline{\star},
\underline{6}
\right]^\top, 
\end{aligned} $$ 
where  $T_j^k$ is randomly selected to such that $|T_j^k|=2$  and $\v_j^k$ are obtained by \eqref{sparse-generator-0}. According to \eqref{def-lambda-itk-0},  $$\lambda_{i,1}^k=1,~~\lambda_{i,2}^k=0,~~\lambda_{i,3}^k=2,~~\lambda_{i,4}^k=3,$$ namely, in $\v_2^k, \v_4^k, \v_5^k$, the $1,2,3,$ and $4$th  coordinates contain $1,0,2,$ and $3$ effective entries, respectively, as indicated by the underlined elements. Then from \eqref{pame-0}, we have
$$ 
\overline{\v}_i^k= \left[
  \frac{2}{1},~
 (\w_i^k)_2,~
\frac{2+\star}{2},~
 \frac{4+5+6}{3} 
\right]^\top=\left[ 
2,
8,
1,
5 
\right]^\top,
$$
where `$\star$' is treated back to $0$. The mean of the selected neighbors' parameters and sparse parameters are computed by
$$ \begin{aligned}
\widetilde{\v}_i^k&= \dfrac{1}{|N_i^k|}\sum_{j\in N_i^k}\v_j^k=\left[
\frac{2}{3},
0,
\frac{2}{3},
5
\right]^\top,\\
 ~\overline{\w}_i^k&= \dfrac{1}{|N_i^k|}\sum_{j\in N_i^k}{\w}_j^k=\left[
3,
7,
1,
5
\right]^\top.
\end{aligned}$$
Obviously, $\overline{\v}_i^k$ is a closer estimator to $\overline{\w}_i^k$ than $\widetilde{\v}_i^k$.

			\subsection{Communication Efficiency}
		The communication efficiency of PaME in Algorithm \ref{algorithm-PaME} arises from three factors. 
		
		\begin{itemize}[leftmargin=11pt]
		\item Communication between node $i$ and its selected neighbors ${\N_i^k \subseteq \N_i}$ occurs only at iterations ${k \in \mathcal{K}_i}$, rather than at every step. Consequently, a larger period $\kappa_i$ permits multiple local updates between communication events, thereby reducing the communication rounds. This periodic communication strategy is well-established in  \cite{zheng2016asynchronous,yu2019parallel,wang2021cooperative,mcmahan2017communication,li2019convergence,li2020federated,zhou2023fedgia}.
	
	\item Even at ${k \in \mathcal{K}_i}$, node $i$ only communicates with a selected subset of its neighbors, rather than all neighbors. 

\item As discussed earlier, each node ${j\in [m]}$ only transmits partial messages with ${(63s_j+n)}$bits content, instead of the full model parameter with $(64n)$bits content. This substantially reduces the transmitted message size and, in turn, significantly improves communication efficiency.
\end{itemize}	

This communication saving is different from directly applying a generic compressor to decentralized SGD. Existing compressed DFL methods usually reduce communication by compressing residuals, gradients, or model differences in gossip-type updates \cite{koloskovadecentralized,tang2018communication}. In contrast, PaME transmits partial model states and aggregates the observed coordinates through the coordinate-wise normalization factor $\lambda_{i,\ell}^k$. The missing entries are therefore interpreted as unobserved neighboring states, rather than zero update directions. This design is natural for the inexact penalty-based framework, since the local update requires neighboring model states to penalize consensus violations, instead of compressed update directions or residual errors. In this sense, PME directly sparsifies the information needed by the penalty term, without introducing an additional compression-and-correction layer on top of a gossip-SGD recursion. This makes the reduced communication compatible with the inexact penalty-based update and the subsequent convergence analysis.

\subsection{Privacy Analysis}
\label{subsec:privacy-reconstruction-risk}
We provide a partial-observation reconstruction-risk analysis to characterize the privacy-related effect of the \texttt{PME} mechanism by considering a passive
honest-but-curious adversary. It follows the PaME protocol but attempts to infer
the private data of a node from the messages available during communication.
Specifically, for node \(i\), the adversary can observe the coordinate set
selected by \texttt{PME} and the corresponding transmitted entries of the
message, and it knows the public algorithmic information, such as the protocol,
model architecture, communication topology, and parameter settings. 
However, it
does not have access to the private mini-batches, the untransmitted coordinates,
or the full stochastic gradient. 
This is a conservative threat model, as the adversary is assumed to possess
extensive public information that can facilitate data inference, which is common
in reconstruction attacks in federated learning \cite{zhu2019deep,sun2021soteria,jeter2026securing}. 

Formally, the data reconstruction is deemed  as a finite-window
inverse problem. In PaME, the private data affect the communicated message only
through the local stochastic update, which is then partially observed due to the coordinate-selection mechanism in \texttt{PME}. For node \(i\), let
\(\mathcal S_{i,T}\) denote the local mini-batch window over \(T\) communication
rounds, and let \(\mathcal F_{i,T}\) be the corresponding data-to-message map
that collects the data-dependent local updates over this window. The
coordinate-selection and partial-observation effect induced by \texttt{PME} is
represented by the observation operator \(\D_{i,T}\). Thus, the information
available to the adversary can be written as
\begin{equation}
	\label{eq:pame-partial-observation}
	y_{i,T}
	=
	\D_{i,T}\mathcal F_{i,T}(\mathcal S_{i,T})
	+
	\boldsymbol{\xi}_{i,T},
\end{equation}
where \(\boldsymbol{\xi}_{i,T}\) denotes observation uncertainty, including
stochasticity and residual modeling errors. Thus, the adversary observes only
the coordinate-restricted data-dependent update \(y_{i,T}\), together with the
corresponding visible coordinate sets, rather than the full local update.
Based on this partial observation, the adversary aims to reconstruct the
private batch window \(\mathcal S_{i,T}\). Therefore,
\eqref{eq:pame-partial-observation} naturally induces an inverse problem:
recovering the unknown local data sequence from its partial, noisy, and
coordinate-restricted image under the data-dependent update mapping.
A more detailed interpretation of \eqref{eq:pame-partial-observation} is
provided in Section~II of the Supplemental Material. Inspired by the inverse-problem perspective for analyzing data
reconstruction attacks and defenses in~\cite{liu2025data}, we quantify the
difficulty of this reconstruction task through the reconstruction risk, induced
by the partial observation in PaME. Specifically, we estimate  lower bound $R_L^{\rm PaME}$ and upper bound $R_U^{\rm PaME}$ of the reconstruction risk, $$R^{\rm PaME}:=\mathbb E
\|\operatorname{vec}(\widehat{\mathcal S}_{i,T})
-
\operatorname{vec}(\mathcal S_{i,T}^{0})\|,$$ where $\mathcal S_{i,T}^{0}$ and   $\widehat{\mathcal S}_{i,T}$  are the true mini-batch window  and the recovered mini-batch window by the adversary,  and $\operatorname{vec}$ represents the vectorization operator.

\begin{theorem}[Finite-window reconstruction risk under partial observation]
	\label{the:pame-reconstruction-risk}
	Suppose that \(\mathcal F_{i,T}\) is locally differentiable around the true
	batch window \(\mathcal S_{i,T}^0\). Let
	\[
	J_{i,T}
	:=
	\nabla_{\mathcal S}\mathcal F_{i,T}(\mathcal S)
	\big|_{\mathcal S=\mathcal S_{i,T}^0}
	\]
	be the data-sensitive Jacobian. Define the effective observable information
	ratio as
	\begin{equation}
		\label{eq:pame-effective-ratio}
		\rho_{i,T}
		:=
		\frac{\|\D_{i,T}J_{i,T}\|_F^2}{\|J_{i,T}\|_F^2}.
	\end{equation}
	Let \(d_{\mathcal B}\) be the dimension of the vectorized local batch window,
	\(M_{i,T}\) be the complete-observation information size, and
	\(\sigma_{\rm obs}\) be the scale of the observation uncertainty. 
	Following the inverse-problem formulation of data reconstruction attacks and
	the reconstruction-risk scaling established for noisy-gradient reconstruction
	in prior work~\cite{liu2025data}, we adopt the following complete-observation
	scaling as the baseline for the corresponding local inverse problem:
\begin{equation}
\label{full-reconstruction-risk}
	R_L^{\rm full}
	\gtrsim
	\sigma_{\rm obs}
	\sqrt{\frac{d_{\mathcal B}}{M_{i,T}}},
	\qquad
	R_U^{\rm full}
	\lesssim
	b
	\sqrt{\frac{d_{\mathcal B}}{M_{i,T}}}.
	\end{equation}
	Then
	\begin{equation}
		\label{range-rho-iT}
		0\le \rho_{i,T}\le 1.
	\end{equation}
	Moreover, the reconstruction risk under the PaME
	partial-observation model satisfies the lower bound
	\begin{equation}
		\label{eq:pame-partial-risk-lower}
		R_L^{\rm PaME}
		\gtrsim
		\sigma_{\rm obs}
		\sqrt{
			\frac{d_{\mathcal B}}{\rho_{i,T}M_{i,T}}
		}.
	\end{equation}
	In addition, a masked reconstruction attack using only the visible
	coordinates has reconstruction error bounded by
	\begin{equation}
		\label{eq:pame-partial-risk-upper}
		R_U^{\rm PaME}
		\lesssim
		b
		\sqrt{
			\frac{d_{\mathcal B}}{\rho_{i,T}M_{i,T}}
		}
		+
		\sigma_{\rm obs},
	\end{equation}
	where \(b\) is the mini-batch size.
\end{theorem}

The proof of Theorem \ref{the:pame-reconstruction-risk} is provided in Section I-B of
the Supplemental Material. In what follows, we make several comments regarding this result.

\begin{remark}
 The quantities \(R_L\) and \(R_U\) characterize the reconstruction risk of the adversary in the finite-window inverse problem. Specifically, \(R_L\) denotes a lower bound on the unavoidable reconstruction error, i.e., the intrinsic difficulty of recovering the private batch window \(\mathcal S_{i,T}\) from the available observation $y_{i,T}$, while \(R_U\) corresponds to the reconstruction error achieved by a feasible masked attack that uses only the observable coordinates. 
Since $R^{\rm PaME}_L\leq R^{\rm PaME} \leq R^{\rm PaME}_U$,   an increase in either the lower or upper reconstruction-risk bound
indicates that the adversary is forced to incur a larger reconstruction error,
which corresponds to stronger protection against data reconstruction.

Compared with the complete-observation baseline, PaME replaces the full data-sensitive Jacobian \(J_{i,T}\) with its coordinate-restricted counterpart \(\D_{i,T}J_{i,T}\). Consequently, the effective information size is reduced from \(M_{i,T}\) to \(\rho_{i,T}M_{i,T}\), where \({0\le \rho_{i,T}\le 1}\). When \({\rho_{i,T}<1}\), the lower bound in (\ref{eq:pame-partial-risk-lower}) is enlarged by a factor of \(1/\sqrt{\rho_{i,T}}\) compared with the complete-observation case (\ref{full-reconstruction-risk}), indicating that any reconstruction attack necessarily suffers a larger error under partial observation. Similarly, the upper bound in (\ref{eq:pame-partial-risk-upper}) shows that even an attack tailored to the visible coordinates can only reconstruct the batch window with an error scaling according to the reduced information size \(\rho_{i,T}M_{i,T}\), together with an additional observation-uncertainty term. Therefore, the PME mechanism improves privacy from the reconstruction-risk perspective by reducing the amount of data-sensitive information exposed to the adversary.
\end{remark}

	\subsection{Robustness}
	PaME is relatively robust due to two strategies:
	partial device participation and partial synchronization. 
	\begin{itemize}[leftmargin=11pt]
		\item When communication occurs, each node $i$ selects a subset of neighbor nodes ${\N_i^{k} \subseteq \N_i}$ to participate in training, which helps mitigate the impact of stragglers. Following the approach in~\cite{li2019convergence}, node $i$ sets a threshold ${m_i^k \in [1,|\N_i|)}$ and forms $\N_i^{k}$ from the first $m_i^k$ neighbors that respond. Once $m_i^k$ sparse parameters are received, node $i$ proceeds without waiting for the remaining neighbors, which are treated as stragglers in that iteration. In practical deployments, a neighbor $j$ with unreliable or severely delayed communication can therefore be treated as a persistent straggler and excluded from $\N_i^{k}$, i.e., ${j \notin \N_i^{k}}$. This mechanism is closely related to partial device participation strategies in CFL~\cite{li2020secure,zhou2023federated,li2022federated}.
	Moreover, Algorithm~\ref{algorithm-PaME} imposes no specific structural constraint on the communication topology during training. The neighbor sets, $\N_i^{k}$, may vary over time, and at each communication round node $i$ only needs to synchronize with the currently selected subset $\N_i^{k}$. 
	
	\item Each node $i$ communicates independently according to its own interval $\kappa_i$, without being affected by other nodes' communication schedules, thereby naturally inducing a partially synchronized training regime and yielding more stable and robust learning dynamics.
	\end{itemize}

	\section{Theoretical Analysis}
	\label{Theory}
	This section provides the theoretical results of the proposed algorithm, including the convergence and complexity.

	\subsection{Assumptions and Setup}
	\label{hyper-setting} 
	To establish the convergence of PaME, we introduce two assumptions, before which, let the communication matrix, ${\B\in\mathbb{R}^{m\times m}}$, be defined  by\begin{equation*}
		B_{ji}:=
		\left\{
		\begin{array}{ll}
			\dfrac{1}{m_i}, & \text{if } j\in N_i,\\[1ex]
			0, & \text{otherwise},
		\end{array}
		\right.
	\end{equation*}
	\begin{assumption} 
		\label{assump-double-stochasitic}
		 $\B$ is doubly stochastic and satisfies
	\begin{equation}\label{max-lambda-0}
	\zeta :=\max\{|\lambda_2(\B)|, |\lambda_m(\B)|\}<1.
\end{equation}
where $\lambda_i(\B)$	 is the $i$th largest eigenvalue of $\B$.
	\end{assumption}
	Regarding Assumption~\ref{assump-double-stochasitic}, it has been extensively adopted to ensure the convergence of DFL algorithms, see those in Table~\ref{tab:comp-assumps}. We further clarify its practical meaning as follows. 
	\begin{itemize}[leftmargin=11pt]
	\item The doubly stochasticity of the initial communication matrix \(\B\) does not necessarily require global knowledge of the entire network for common decentralized overlay graphs, such as rings, grids, complete graphs, and random regular graphs \cite{wang2019matcha,santucci2023doubly,li2025decentralized}.
	\item This condition is imposed only on the initial matrix for convergence analysis. During the execution of PaME, the neighbor set \(N_i^k\) may vary at each communication round \(k\in\mathcal K\), inducing dynamic matrices \(\B^k\). Hence, the actual communication matrices can be sparse, time-varying, and non-doubly stochastic, which is more flexible than many existing DFL methods requiring static or per-round doubly stochastic matrices~\cite{sun2022decentralized, liu2024decentralized, zhao2022beer, li2019decentralized, michelusi2022finite}. 
	\item In the subsequent numerical experiments, the communication graphs are constructed only under the connectivity requirement, rather than by enforcing exact doubly stochasticity at every communication round. The observed convergence behavior further suggests that Assumption~\ref{assump-double-stochasitic} mainly serves as a sufficient condition for clean theoretical analysis, while PaME can operate effectively under more flexible decentralized communication patterns.
	\end{itemize}
	
	\begin{assumption}
		For each ${i \in [m]}$,  $\nabla f_i$ is Lipschitz continuous with ${\alpha_i > 0}$ on $\mathbb{N}(2\delta)$ for a given $\delta\in(0,\infty)$, where 
		$$\mathbb{N}(2\delta):= \{\w \in \mathbb{R}^n : \|\w\|_\infty \le 2\delta\}$$
		and $\|\w\|_\infty$ denotes the infinity norm of $\w$.
		\label{assump-gradientlip}
	\end{assumption}

 Regarding Assumption \ref{assump-gradientlip}, the Lipschitz continuity of the gradient (often referred to as $L$-smoothness) on a bounded region is equivalent to local Lipschitz continuity. Thus, it is a local version of $L$-smoothness. This constitutes a substantially weaker condition compared to the global $L$-smoothness, convexity, or bounded gradient assumptions commonly imposed in standard DFL frameworks (see Table~\ref{table:compare-conditions}). As a result, our theoretical findings are established under the mildest conditions among existing DFL convergence guarantees.
 
For the purpose of simplifying the convergence analysis, we adopt the following parameter settings.
\begin{setup}\label{setup}
Parameters in Algorithm~\ref{algorithm-PaME} are chosen as follows.
	\begin{itemize}[leftmargin=14pt]
		\item[1)] Set ${\W^0 = \mathbf{0}}$ and ${\kappa_i = k_0}$ for all ${i\in [m]}$. This is adopted for analytical convenience without loss of generality. In fact, one can always let $k_0$ be the least common multiple of $\{\kappa_1,\kappa_2,\cdots,\kappa_m\}$, then the subsequent analysis remains similar to the case of ${\kappa_i = k_0}$.  
		\item[2)] Set ${m_i^k = t_i=\lfloor \nu_i |N_i|\rfloor }$ for all ${k \ge 0}$ and ${i \in [m]}$, where ${\nu_i \in (0, 1]}$ is the participation rate and $\lfloor a\rfloor$ is the floor of $a$. That is, at every iteration, node $i$ selects the same number of neighbor nodes to join in the training.    
\item[3)] Set ${\sigma_i^0 = \sigma^0}$ and ${\gamma_i=\gamma}$, ${s_i=s}$ for all ${i\in [m]}$. Again, this is adopted for analytical convenience without loss of generality. In fact, for different ${\sigma_i^0}$ and ${\gamma_i}$,  we can conduct similar analysis by considering 
$$\begin{array}{c}
\sigma^0=\min_i\sigma_i^0,~~
 \gamma=\min_i\gamma_i,~~
 s=\min_is_i.
\end{array}$$
Moreover, in Algorithm \ref{algorithm-PME}, each node ${j\in N_i^k}$ independently constructs a set ${T_j^k}$ by uniformly selecting $s$ entries of $[n]$ without replacement. 	
		\item[4)] In the sequel, given $\zeta$ defined in (\ref{max-lambda-0}) and integers $k_0$ and $(t_1,\cdots,t_m)$,  initialize  the following parameters
		\begin{equation}\label{initial-p-gamma-nu-0}
\gamma\in\left(1,\zeta^{-2/k_0}\right),~	p:=\frac{s}{n}\in(0,1)	,~\nu_i \in (0, 1],
		\end{equation}	
such that
	\begin{equation}\label{cond-p-zeta-gamma-0}
		(1-p)^{t_i} (1+ {\zeta})^2 +2p \sum_{j\in N_i}\nu_j
		<
		\left( {\gamma^{-k_0/2} }- {\zeta}\right)^2,
	\end{equation}
for any $i\in[m]$.	Moreover, choose 
			\begin{equation*}\label{def-sigma-gamma}
  \sigma^0\geq \sigma ,
		\end{equation*}	
		where  ${\sigma\in(0,\infty)}$ is a given constant relying on $\alpha_i$, $t_i$, $\delta$, and $\gamma$. Its explicit form are given in Supplemental material.
	\end{itemize}  
	\end{setup}	
We would like to point out that since $s$ (i.e., $p$) and $\nu_i$ can be chosen flexibly, there are always many choices of these parameters, for instance, taking $p$ close to $1$ and $\nu_i$ close to $0$, that satisfy condition \eqref{cond-p-zeta-gamma-0}.
	\subsection{Sequence Convergence}
	For notational simplicity, we define two gaps by
	\begin{equation*}
	\begin{aligned}
		 \Delta \w_i^k := \w_i^k - \w_i^{k-1},\qquad \Delta \overline{\v}_i^k := \overline{\v}_i^k - \overline{\v}_i^{k-1},
	\end{aligned}
	\end{equation*}
and  a merit function by
	\begin{equation*}
	\begin{aligned}
		{H}^k &:= \mathbb{E}\sum_{i=1}^{m}\left( f_i(\w_i^k) + \dfrac{\sigma_i^k t_i}{2}\|\w_i^k - \overline{\v}_i^k\|^2\right),\\
		\widetilde{H}^{{k}}&: = H^{k}+ C\gamma^{-{k}}+ D\eta^{-k},\\
		\boldsymbol{\varpi}^{k}  &:=\frac{1}{m}\sum_{i=1}^m\mathbf{w}_i^{k},
	\end{aligned}\end{equation*}	
where $\gamma$ is given in (\ref{initial-p-gamma-nu-0}), ${\eta>1}$, $C>0$, and $D>0$. The explicit forms of  $(\eta,C,D)$ are given in Supplemental material and do not rely on the generated sequence.
The first following result establishes the boundedness of the sequence generated by PaME, as well as the monotonic decreasing property of a sequence associated with ${H}^k$. 	
	\begin{theorem}\label{the:desent}
		Let $\{(\W^{k}, {\mathbf{V}}^{k})\}$ be the sequence generated by Algorithm~\ref{algorithm-PaME} with Setup \ref{setup}. Then the following statements hold under Assumptions \ref{assump-gradientlip} and \ref{assump-double-stochasitic}.
	\begin{itemize}[leftmargin=12pt]
		\item[1)] 
		For any ${{k} \ge 0}$ and ${i\in[m]}$,  $ \w_i^{k}\in\mathbb{N}(2\delta)$ and $ \overline{\v}_i^{k}\in \mathbb{N}(2\delta)$.
		\item[2)] 
		For any ${k} \ge 0$,
		\begin{eqnarray*} \label{descent-L}
		\widetilde{H}^{{k}} - \widetilde{H}^{{k}+1} \ge \E \sum_{i=1}^{m}\frac{\sigma_i^k t_i}{8}\Big(\|\Delta \w_i^{{k}+1} \|^2 
		+ \|\Delta \overline{\v}_i^{{k}+1} \|^2 \Big), 
		\end{eqnarray*}
	\end{itemize}
	\end{theorem}
The proof of Theorem \ref{the:desent} is provided in Section III-D of the Supplemental Material. It is worth noting that the boundedness of the generated sequence is a crucial property, as it allows us to relax several commonly imposed boundedness assumptions, such as bounded (stochastic) gradients, bounded second moments of the (stochastic) gradients, and bounded variance, as summarized in Table~\ref{tab:comp-assumps}. This observation partially explains why the subsequent convergence results can be established under mild assumptions.	We provide the main convergence result.
	\begin{theorem}\label{the:convergence}
	 Let $\{(\W^{k}, {\mathbf{V}}^{k})\}$ be the sequence generated by Algorithm~\ref{algorithm-PaME} with Setup \ref{setup}. Then the following statements hold under Assumptions \ref{assump-double-stochasitic} and \ref{assump-gradientlip}.
	\begin{itemize}[leftmargin=14pt]
			\item[1)] 
			For any $i \in [m]$,
			\begin{equation*}\label{gap-vanishing}
			 \lim\limits_{{k} \rightarrow \infty}\E\nl \Delta \w_i^{k}\nr 
			= \lim\limits_{{k} \rightarrow \infty}\E\nl \Delta \overline{\v}_i^{k}\nr 
			= \lim\limits_{{k} \rightarrow \infty}\E\|\w_i^{k} - \overline{\v}_i^{k}\|=0.
			\end{equation*}
			\item[2)] Sequence $\{\boldsymbol{\varpi}^{{k}}\}$ converges to $\boldsymbol{\varpi}^{\infty}$ in the sense of $L^2$ convergence and expectation, namely,
			 \begin{equation*} \label{E-w-k-L2}
 \lim_{k\to\infty} \mathbb{E}\|\boldsymbol{\varpi}^{{k}}-\boldsymbol{\varpi}^\infty\|^2=0,\quad  \lim_{k\to\infty} \E\boldsymbol{\varpi}^{{k}}=\E\boldsymbol{\varpi}^{\infty}.
\end{equation*}   Sequence  $\{(\E\W^{k}, \E{\mathbf{V}}^{k})\}$ converges and satisfies
				\begin{equation*}
			 \lim\limits_{{k} \rightarrow \infty} \E\W^{k} 
			= \lim\limits_{{k} \rightarrow \infty} \E\V^{k}
			= (\E\boldsymbol{\varpi}^{\infty},\cdots,\E\boldsymbol{\varpi}^{\infty})=:\W^\infty.
			\end{equation*}	 
			\item[3)] 
			Sequence  $\{(H^{k},\widetilde{H}^{k},\E f(\boldsymbol{\varpi}^{{k}}))\}$ converges  and satisfies
			\begin{equation*}
			  \lim\limits_{{k} \rightarrow \infty} \widetilde{H}^{k} =\lim\limits_{{k} \rightarrow \infty} H^{k}  = \lim\limits_{{k} \rightarrow \infty} \E f(\boldsymbol{\varpi}^{{k}})=  \E f(\boldsymbol{\varpi}^{\infty}).
			\end{equation*} 			 
		\end{itemize}
	\end{theorem}
	The proof of Theorem \ref{the:convergence} is provided in Section III-E of the Supplemental Material. We now establish the convergence rate under the same assumptions and parameter setup.
	\subsection{Convergence Rate}
	\begin{theorem}\label{the:complexity}
		Let $\{(\W^{k}, {\mathbf{V}}^{k})\}$ be the sequence generated by Algorithm~\ref{algorithm-PaME} with Setup \ref{setup}.  Under  Assumptions \ref{assump-double-stochasitic} and \ref{assump-gradientlip}
		\begin{equation*}\begin{aligned}
		&\E\|\W^{k} - \W^\infty\|_F^2= O(\gamma^{-{k}}),\\
		& \E\|\V^{k} - \W^\infty\|_F^2 = O(\gamma^{-{k}}),\\
		& \E| f(\boldsymbol{\varpi}^{k}) -  f(\boldsymbol{\varpi}^\infty) | =O(\gamma^{-{k/2}}).
		\end{aligned}
		\end{equation*}		
	\end{theorem}
The proof of Theorem \ref{the:complexity} is provided in Section III-F of the Supplemental Material.  Theorem  \ref{the:complexity} establishes that both  sequences $\{(\W^{k}, {\mathbf{V}}^{k})\}$ and  $\{f(\boldsymbol{\varpi}^{k}) \}$ converge to their respective limits at a linear rate in the sense of $L^2$ convergence.  

To see the optimality of limiting point $\boldsymbol{\varpi}^{\infty}$ for problem \eqref{fl}, we need to specify a particular $\gamma$ such that
\begin{equation}\label{condition-gamma-K-0}
\lim_{T\to\infty} \frac{\gamma-1}{1-\gamma^{-T}} = 0.
\end{equation}
This is a mild condition because given $T$, there are many choices of $\gamma$ satisfying this condition, such as
\begin{align*}
\gamma=1+ T^{-a},\quad \forall~ a>0,\\
\gamma=1+ c^{-T},\quad \forall~ c>1.
\end{align*}
\begin{theorem}\label{the:pame-gradient-stationarity}
	Let $\{(\W^{k}, {\mathbf{V}}^{k})\}$ be the sequence generated by Algorithm~\ref{algorithm-PaME} with Setup \ref{setup} and $T$ be the total number of iterations. Choose $\gamma$ to satisfy  (\ref{condition-gamma-K-0}). Then under  Assumptions \ref{assump-double-stochasitic} and \ref{assump-gradientlip},  it holds
	\begin{equation*}\begin{aligned}
			&\E\nl\nabla f(\boldsymbol{\varpi}^{\infty})\nr^2= 0,\\
			& \mathbb E
			\nl
			\nabla f(\boldsymbol{\varpi}^{T})
			\nr^2
			=
			O(\gamma^{-T}).
		\end{aligned}
	\end{equation*}		
\end{theorem}
The proof of Theorem \ref{the:pame-gradient-stationarity} is provided in Section III-G of the Supplemental Material. It complements Theorems~\ref{the:convergence} and~\ref{the:complexity} by providing a first-order stationarity characterization.
Specifically, the first equality shows that the limit point \(\boldsymbol{\varpi}^{\infty}\) is a stationary point in the \(L^2\)-sense for problem \eqref{fl}, while the second equality further shows that the global gradient norm at \(\boldsymbol{\varpi}^{T}\) decays linearly at the rate \(O(\gamma^{-T})\).
Thus, the convergence of PaME is not only to a consensus limit, but also to a first-order stationary point of the global objective.

We emphasize that the theoretical guarantees, including Theorems~\ref{the:convergence}, \ref{the:complexity} and \ref{the:pame-gradient-stationarity}, are derived under the randomness induced by the {\tt PME} mechanism in Algorithm~\ref{algorithm-PaME}. If a deterministic transmission scheme is employed instead, all convergence results hold with probability one, thereby recovering a strong deterministic convergence theory for the standard DFL setting.

	\section{Numerical Experiments}
	\label{Numerical}
	In this section, we present numerical experiments
	to evaluate the performance of PaME.  All experiments are implemented using Python 3.7.
	
	\subsection{Testing Example}
	\begin{example}{(Linear Regression)}\label{eg5.1} The loss function of each node ${i\in[m]}$ with a local dataset $\mathcal{D}_i$ takes the following form,
		$$f_i(\mathbf{w})=
 		\frac{1}{2m_i} \sum_{(\mathbf{a},b)\in\mathcal{D}_i} \Big(  \langle \mathbf{a}, \mathbf{w}  \rangle   - b  \Big)^2,  
		$$
		where feature ${\mathbf{a}\in \mathbb{R}^n}$, observation ${b \in \mathbb{R}}$, ${m_i=|\mathcal{D}_i|}$. 
	To assess the effectiveness of PaME, we assume the existence of a `ground truth' solution ${\w^* \in \mathbb{R}^n}$ with $1\%$ non-zeros entries. Then ${b=\langle \mathbf{a}, \mathbf{w}^*\rangle+0.5 e}$, where $e$ is the noise. All entries of $\mathbf{a}$ and $e$ are identically and independently distributed from a standard normal distribution  while the non-zero entries of $\w^*$ are uniformly generated from $[0.5,2] \cup[-2,-0.5]$.
		
	\end{example}
	\begin{example}{(Logistic Regression)}\label{eg5.2} The loss function of each node ${i\in[m]}$ with a local dataset $\mathcal{D}_i$  takes the following form,
		$$
		f_i(\mathbf{w}) = 
		\frac{1}{m_i} \sum_{(\mathbf{a},b)\in\mathcal{D}_i} \left(\ln \left(1 + \mathrm{e}^{\left\langle \mathbf{a}, \mathbf{w} \right\rangle}\right) - b \left\langle \mathbf{a}, \mathbf{w} \right\rangle \right) + \frac{\lambda}{2}\|\mathbf{w}\|^2,
		$$
		where feature ${\mathbf{a}\in \mathbb{R}^n}$, label ${b \in \{0,1\}}$, ${m_i=|\mathcal{D}_i|}$, and ${\lambda > 0}$ (e.g., ${\lambda = 0.001}$ in our numerical experiments). We assume a ground-truth parameter vector ${\w^* \in \mathbb{R}^n}$ with $50\%$ nonzero entries. 
		Let $\mathbf{a}$  be generated in the same way as in Example~\ref{eg5.1}, and $b$ is obtained by applying the sigmoid function to $\langle \mathbf{a}, \mathbf{w}^*\rangle$ to produce a probability in $[0,1]$.
		
	\end{example}

		\begin{example}{(Convolutional Neural Network, CNN)}\label{eg5.3}
			Consider a 10-class image classification problem trained using a CNN. Each node $i \in [m]$ possesses a local dataset $\mathcal{D}_i$ and minimizes the empirical cross-entropy loss:
			$$
			f_i(\mathbf{w}) = -\frac{1}{m_i} \sum_{(\mathbf{a}, b)\in \mathcal{D}_i}  \sum_{c\in[10]}  \mathbb{I}(b=c) \ln(h(\mathbf{w}; \mathbf{a})_c),
			$$
			where $(\mathbf{a}, b)$ denotes a training sample (e.g., an image and its corresponding label), ${m_i=|\mathcal{D}_i|}$, $\mathbb{I}(\cdot)$ is the indicator function, and $h(\mathbf{w}; \mathbf{a})_c$ is the predicted probability of class $c$ given input $\mathbf{a}$ parametrized by $\w$.
			
			We utilize the Fashion-MNIST dataset \cite{xiao2017fashion}, which consists of 70,000 grayscale images (${28 \times 28}$ pixels) depicting 10 categories of fashion items, split into 60,000 training images and 10,000 testing images.
		To verify the generalizability of PaME under data heterogeneity, we employ data partitioning strategies wherein each node is allocated samples from a varying number of classes, thereby simulating diverse heterogeneous data distributions.
			
		\end{example}
		
	\begin{example}{(ResNet-20)}\label{eg5.4}
			Consider extending the multi-class classification problem and objective function in Example \ref{eg5.3}
			by replacing the CNN backbone with ResNet-20 \cite{he2016deep} and evaluate on CIFAR-10 \cite{krizhevsky2009learning}, a more challenging natural-image benchmark.
			CIFAR-10 consists of 60,000 $32\times 32$ color images in 10 classes (e.g., airplane, automobile, bird),
			with 50,000 training images and 10,000 test images.
			Compared with Example \ref{eg5.3} (Fashion-MNIST), CIFAR-10 typically poses a harder learning task due to the complexity of color natural images and higher intra-class variability.
			To simulate heterogeneous (non-IID) data across nodes,
			we adopt a Dirichlet partitioning strategy \cite{shi2023improving}, where the class distribution for each client is sampled from a Dirichlet distribution $\mathrm{Dir}(\beta)$. 
			
			To further examine the scalability and robustness of PaME, we also evaluate the ResNet-20 backbone on Tiny-ImageNet~\cite{le2015tiny}, which is a larger and more complex natural-image benchmark than CIFAR-10. Tiny-ImageNet dataset contains 200 object classes, with images resized to \(64\times64\), and includes substantially richer inter-class diversity and visual patterns. Compared with CIFAR-10, the larger number of classes and higher image complexity make Tiny-ImageNet a more challenging testbed for decentralized training. Therefore, this experiment is designed to further verify the performance of PaME on larger-scale and more complex data under different levels of data heterogeneity. Specifically, we again adopt the Dirichlet partitioning strategy to generate non-IID data distributions across nodes, where a smaller concentration parameter indicates stronger heterogeneity.

		\end{example}

	\subsection{Implementation}
	The basic setup for PaME in Algorithm \ref{algorithm-PaME} is given as follows. Given a randomly generated graph ${G=([m],E)}$, extract neighbor sets ${\left\{\N_i: i \in [m]\right\}}$. For each node ${i \in [m]}$,   subset $\N_i^{k}$ is randomly selected from $\N_i$ such that ${\left|\N_i^{k}\right|=\lceil \nu_i\left|\N_i\right|\rceil}$. The parameters are initialized as follows: $\nu_i=\nu$, $s_i=s$, $\gamma_i=\gamma$, and $\sigma_i^0=\sigma^0$ for all ${i\in[m]}$, where $\nu$, $s$, $\gamma$, and $\sigma^0$ are specified in Table~\ref{tab:numerical} unless stated otherwise.  When analyzing the effect of a particular parameter, its value is adjusted and described explicitly.  Moreover, for each node ${i\in[m]}$, integer $\kappa_i$ is randomly selected from a predefined interval that varies across examples, as reported in Table \ref{tab:numerical}. Therefore, all nodes operate with distinct communication periods, resulting in  a partially synchronized training regime, a more realistic training deployment. Finally, we terminate the algorithm if $$
	\operatorname{std}\left\{f(\boldsymbol{\varpi}^{k-2}), f(\boldsymbol{\varpi}^{k-1}), f(\boldsymbol{\varpi}^{k})\right\}<10^{-3},
	$$ where `std' represents the standard deviation.	
	\begin{table}[th]
	\centering
	\caption{Choices of parameters. 
		\label{tab:numerical}}
	\renewcommand{\arraystretch}{1.15}\addtolength{\tabcolsep}{6pt}
	\begin{tabular}{lccccr}
		\hline 
		&  $\nu$ & $s/n$   & $\gamma$ & $\sigma^0$ & $\kappa_i$ \\\hline   		
		Example \ref{eg5.1} &0.2& 0.2   & 1.005 & 1.0
		& [3,7]\\ 
		Example \ref{eg5.2} &0.2& 0.2  & 1.005 & 1.0
		& [3,7] \\
		Example \ref{eg5.3} &0.5& 0.1 & 1.001 & 5.0
		& [5,10] \\
		Example \ref{eg5.4} &0.5& 0.1  & 1.001 & 5.0
		& [5,10] \\
		\hline 
	\end{tabular}
\end{table}

	\begin{figure}[th]
	\centering
	\begin{subfigure}{0.24\textwidth}
		\centering	
		\includegraphics[width=\textwidth]{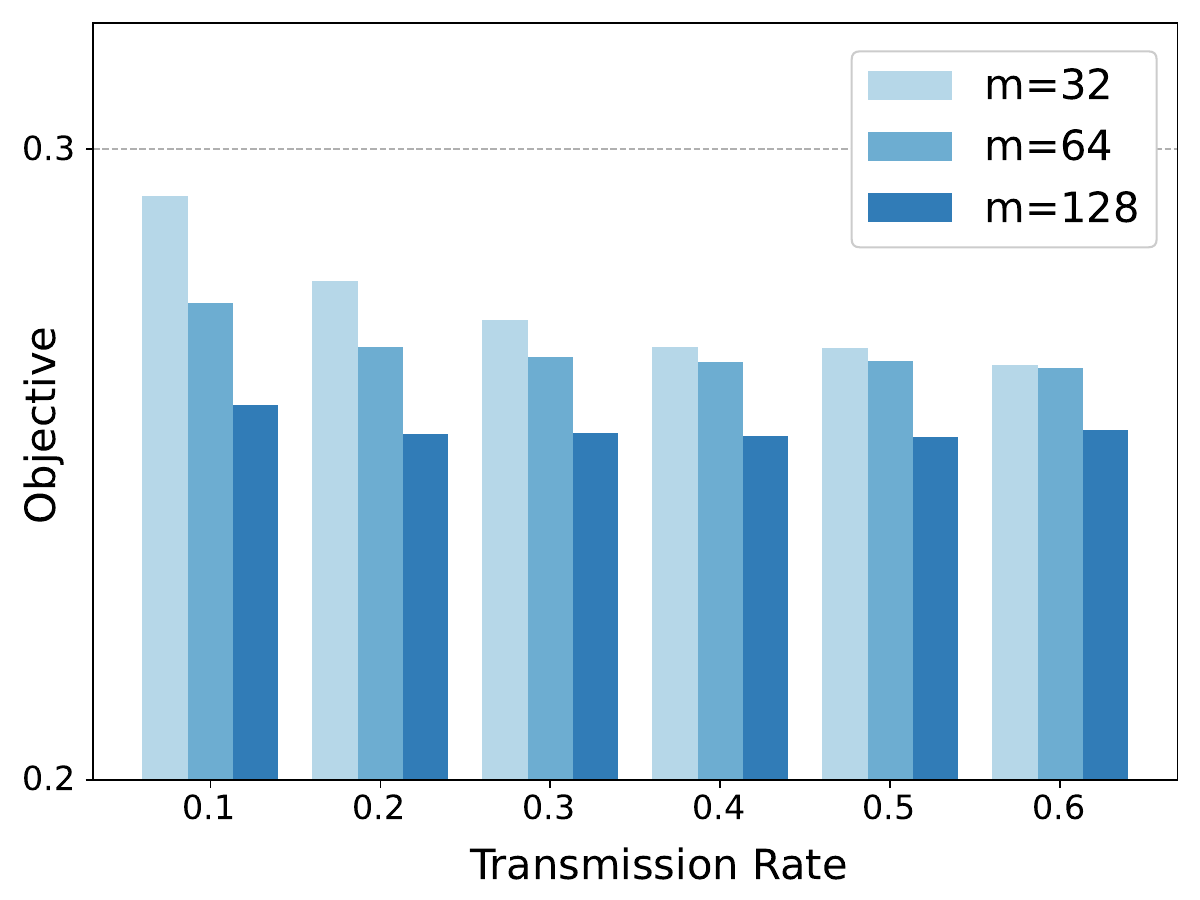}
		\caption{Effect of transmission rate.}
		\label{fig:transmission}
	\end{subfigure}
	\begin{subfigure}{0.24\textwidth}
		\centering
		\includegraphics[width=\textwidth]{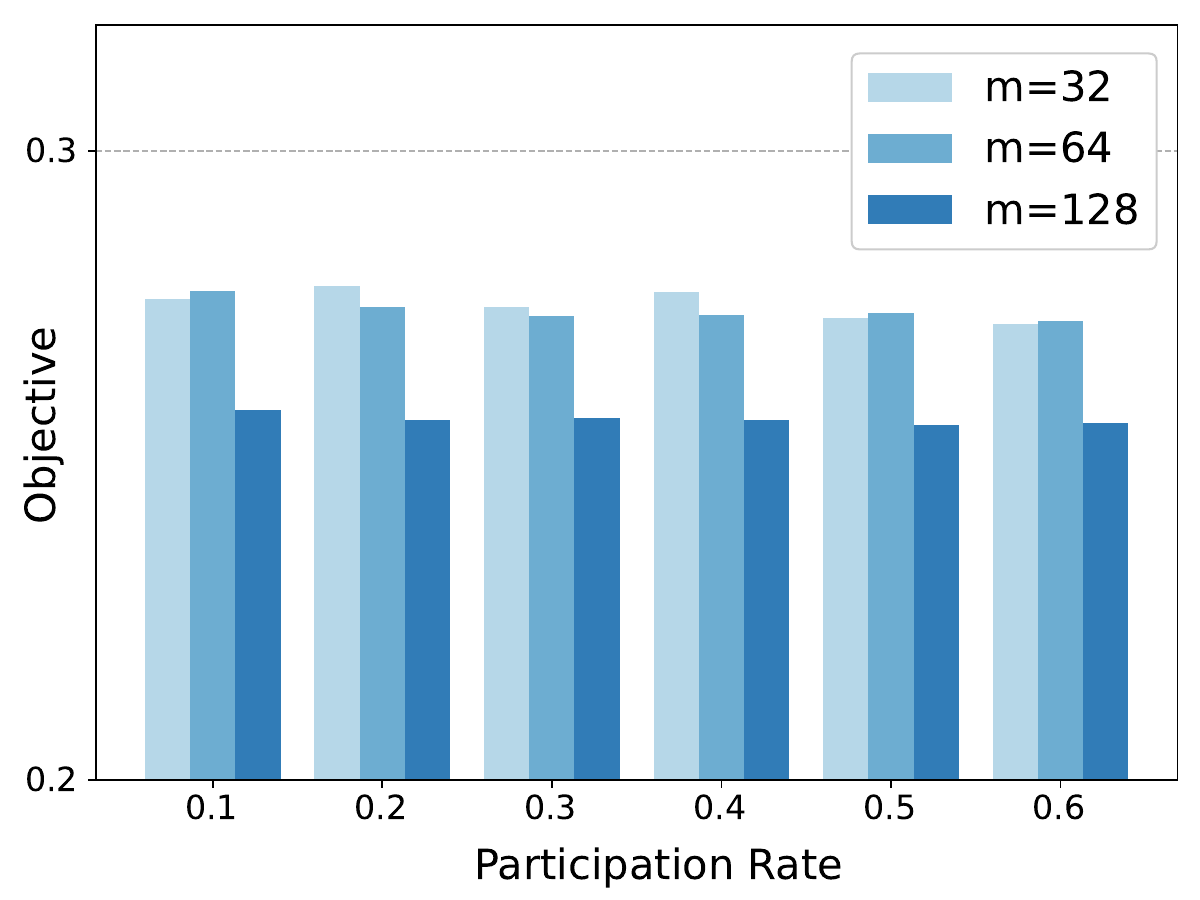}
		\caption{Effect of participation rate.}
		\label{fig:participation}
	\end{subfigure}\\[2ex]
	
	\begin{subfigure}{0.24\textwidth}
		\centering	
		\includegraphics[width=\textwidth]{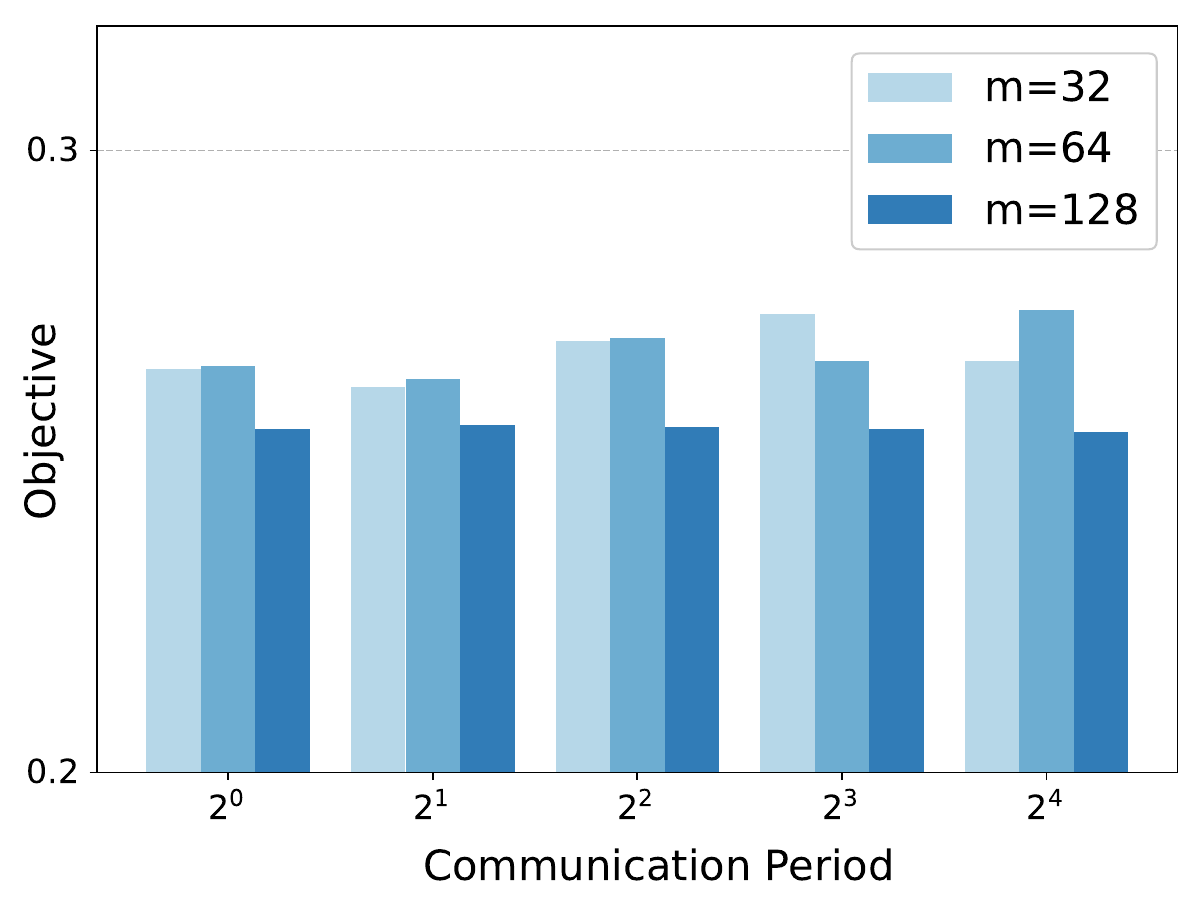}
		\caption{Effect of CP (homogeneous).}
		\label{fig:ex1crhomo}
	\end{subfigure}
	\begin{subfigure}{0.24\textwidth}
		\centering
		\includegraphics[width=\textwidth]{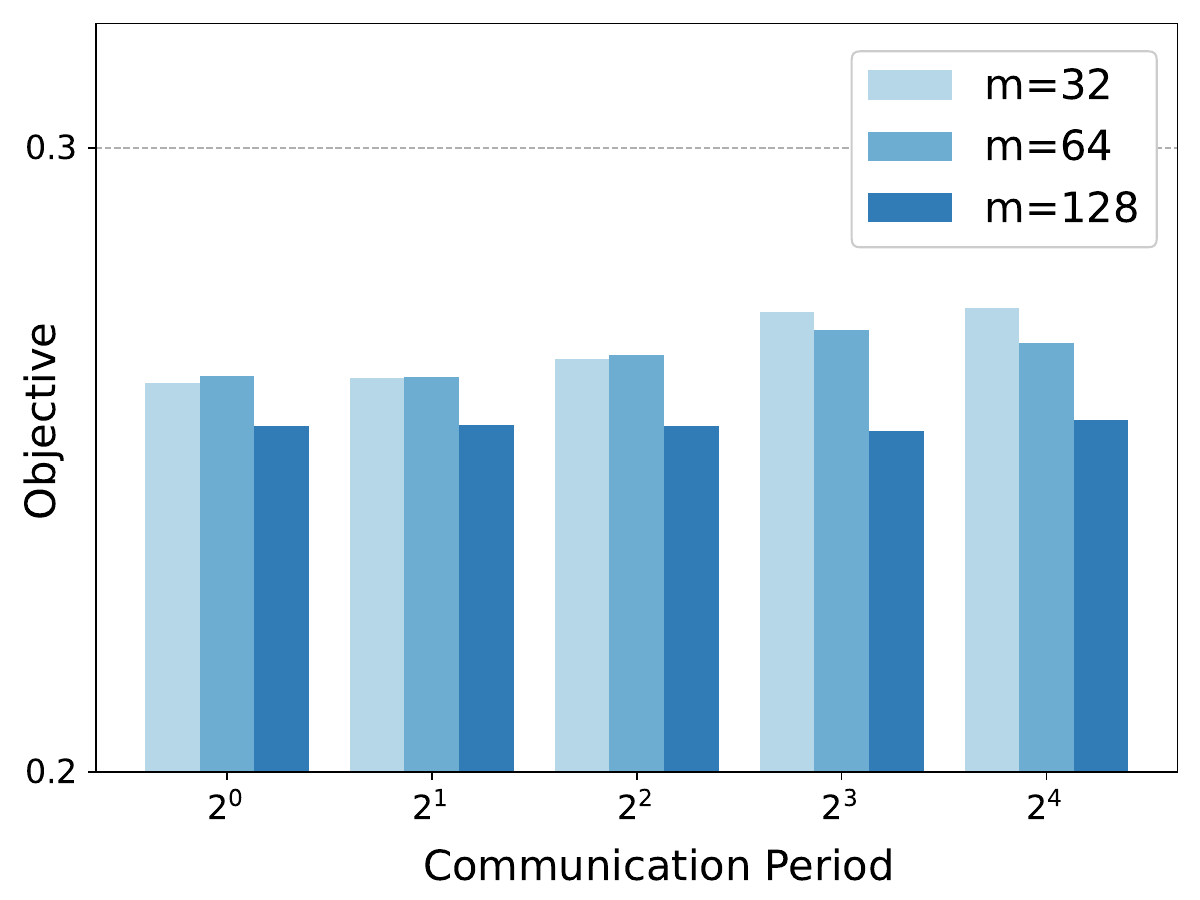}
		\caption{Effect of CP (heterogeneous).}
		\label{fig:ex1crheter}
	\end{subfigure}
	\caption{Self-comparison of PaME for Example \ref{eg5.1}.}\vspace{-3mm}
	\label{ex1-srateandprate}
\end{figure}

		\subsection{Self-Comparison of PaME}
	To assess the performance of the proposed algorithm PaME under different settings in Example \ref{eg5.1}, we conduct a comprehensive comparison in four dimensions: transmission rate $s/n$, participation rate $\nu$, communication period $\kappa_i$, and graph connectivity $\max_i|N_i|$.

	 {\it 1) Effect of transmission rate:}
		Fig.~\ref{fig:transmission} illustrates the effect of varying the transmission rate ${{s}/{n} \in \{0.1,0.2,\cdots,0.6\}}$ across different numbers of nodes ${m \in \{32,64,128\}}$, with the participation rate fixed at ${\nu = 0.2}$. The results indicate that while higher transmission rates generally yield lower final objective values (i.e., improved accuracy), a high transmission rate is not strictly necessary. Notably, with a transmission rate of only $0.1$, the final objective already approximates that obtained under full transmission. Furthermore, a larger number of nodes $m$ tends to improve performance under low transmission rates. This is likely due to the fact that a larger $m$ involves more participating nodes, thereby facilitating better information mixing. Based on the corresponding convergence curves in Fig.~\ref{ex1-sratecurve}, the marginal gain becomes negligible once the rate exceeds approximately $0.2$, although higher transmission rates accelerate convergence. This demonstrates that the proposed mechanism achieves satisfactory convergence while reducing communication costs by at least $80\%$.
		
			\begin{figure*}[t]
		\centering
		\includegraphics[width=.328\textwidth]{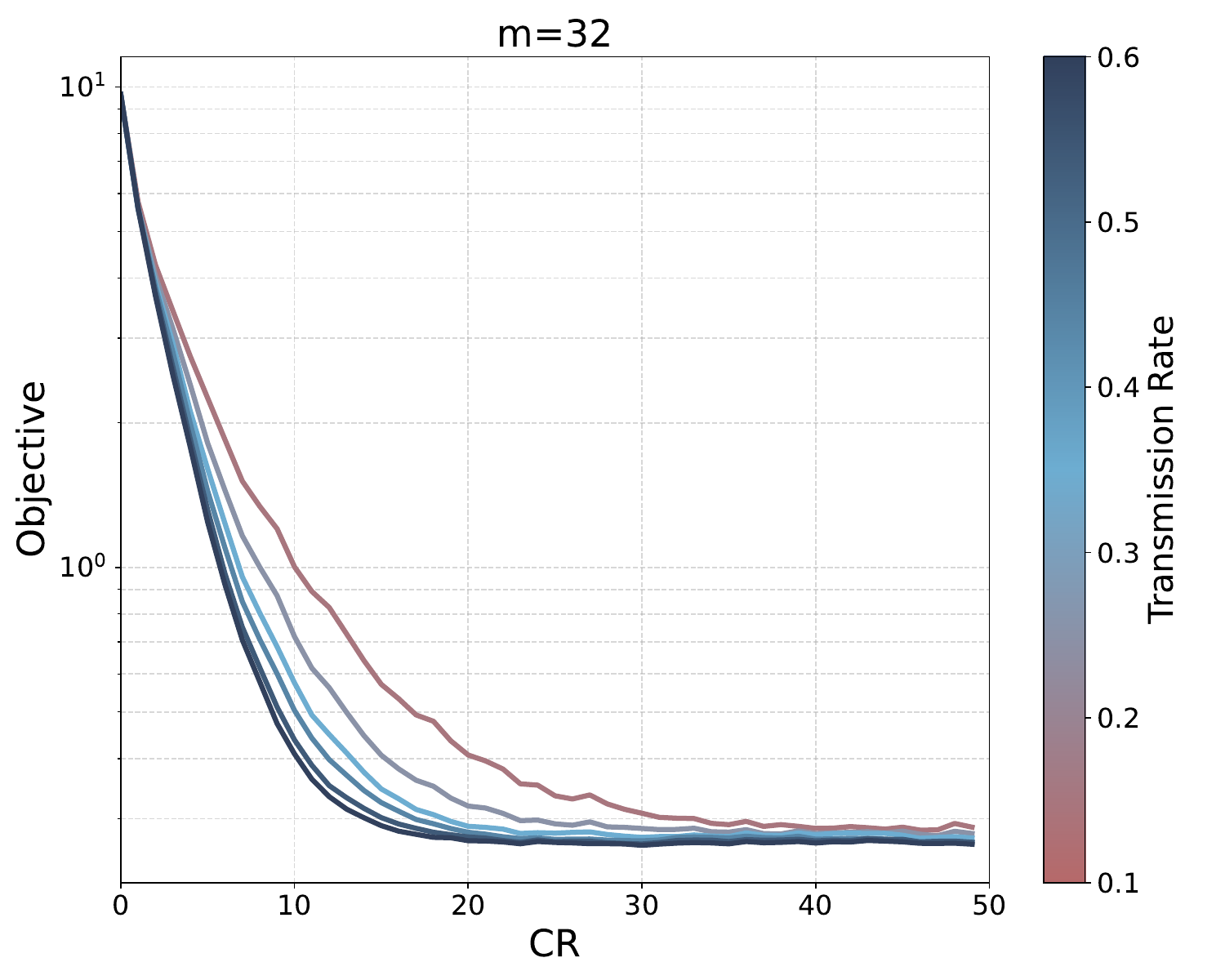}~
		\includegraphics[width=.328\textwidth]{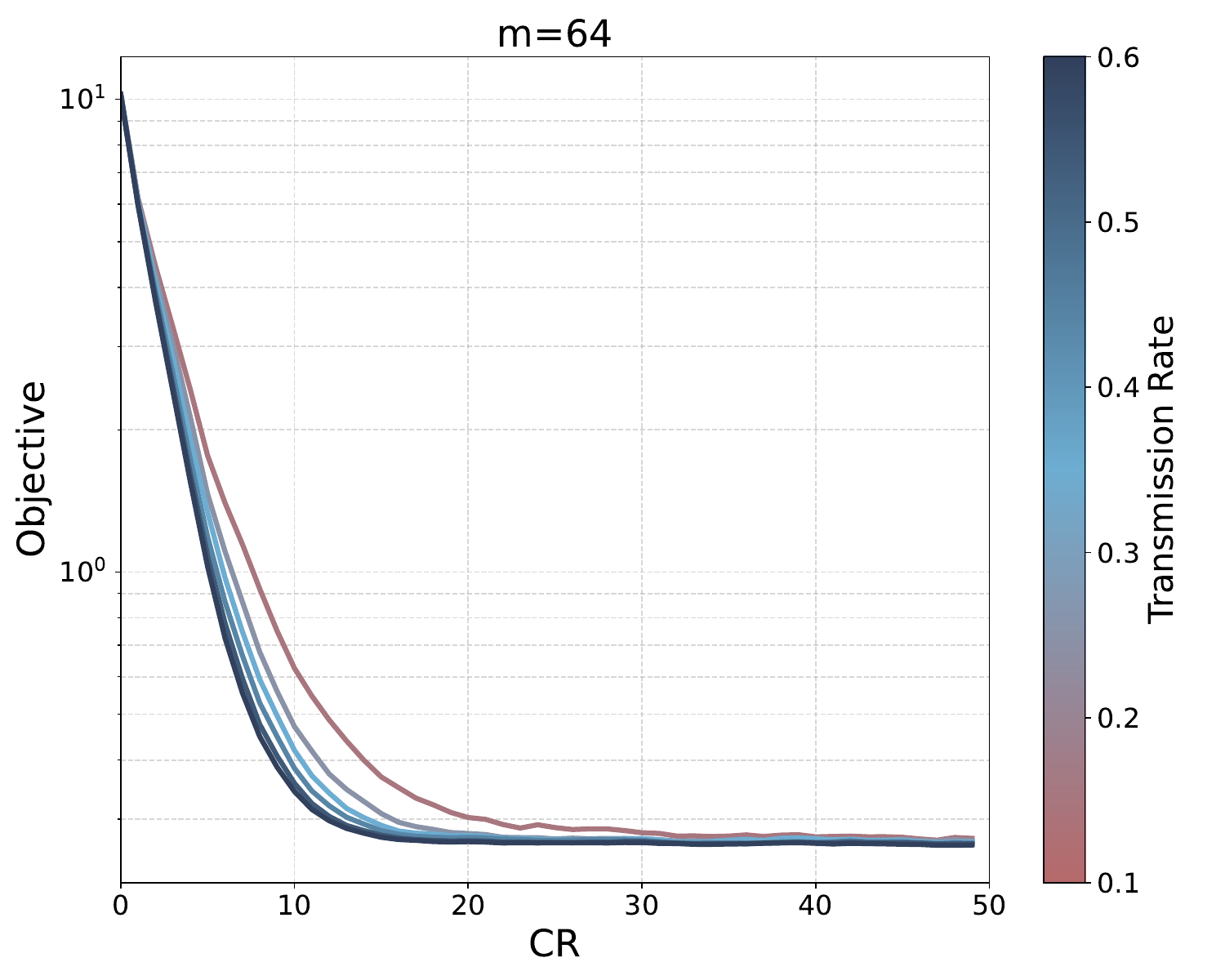}~
		\includegraphics[width=.328\textwidth]{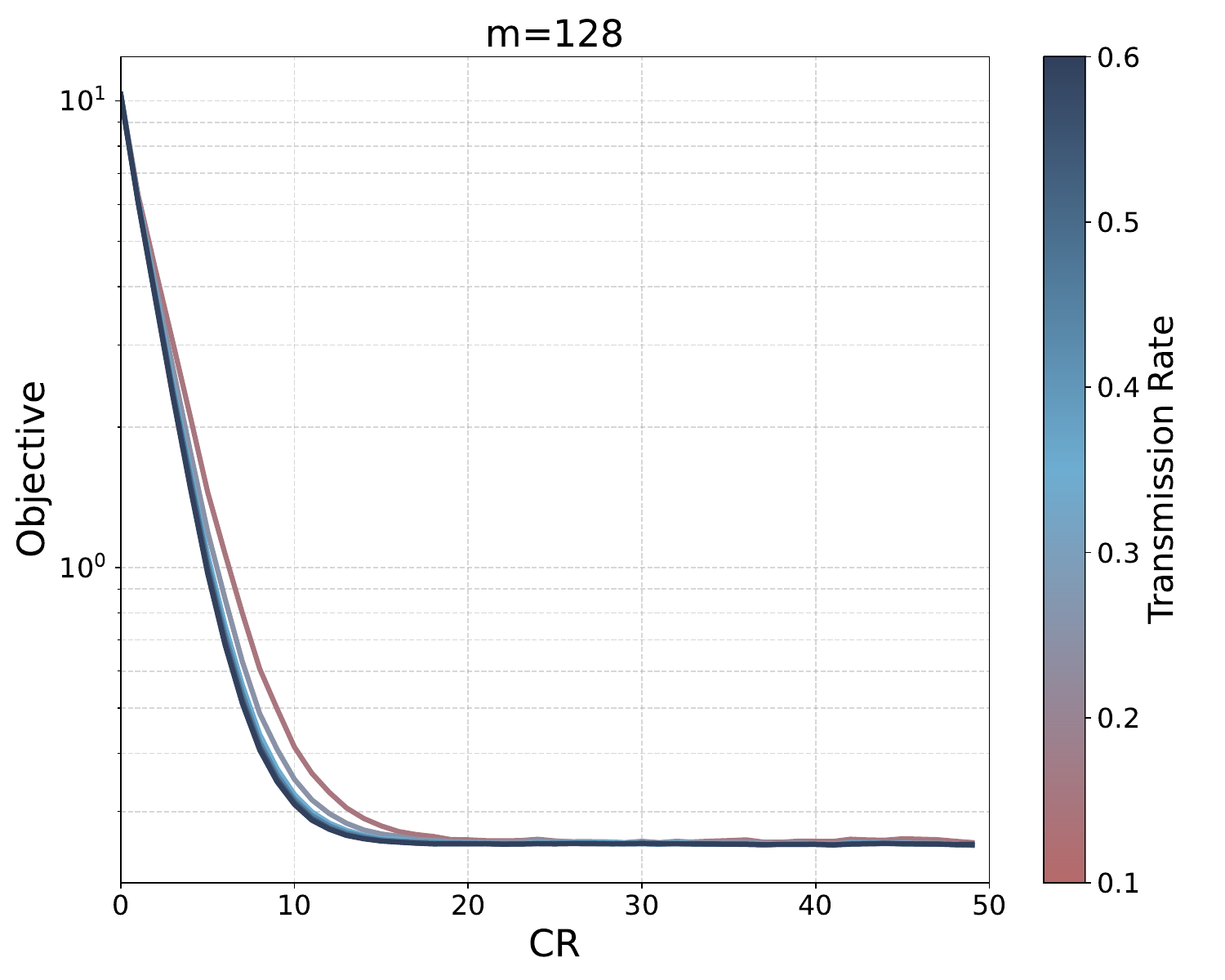}
		\caption{Convergence curves of different transmission rate for PaME solving Example \ref{eg5.1}.}\vspace{-2mm}
		\label{ex1-sratecurve}
	\end{figure*}
	 
	\begin{figure*}[t]
		\centering
		\includegraphics[width=.328\textwidth]{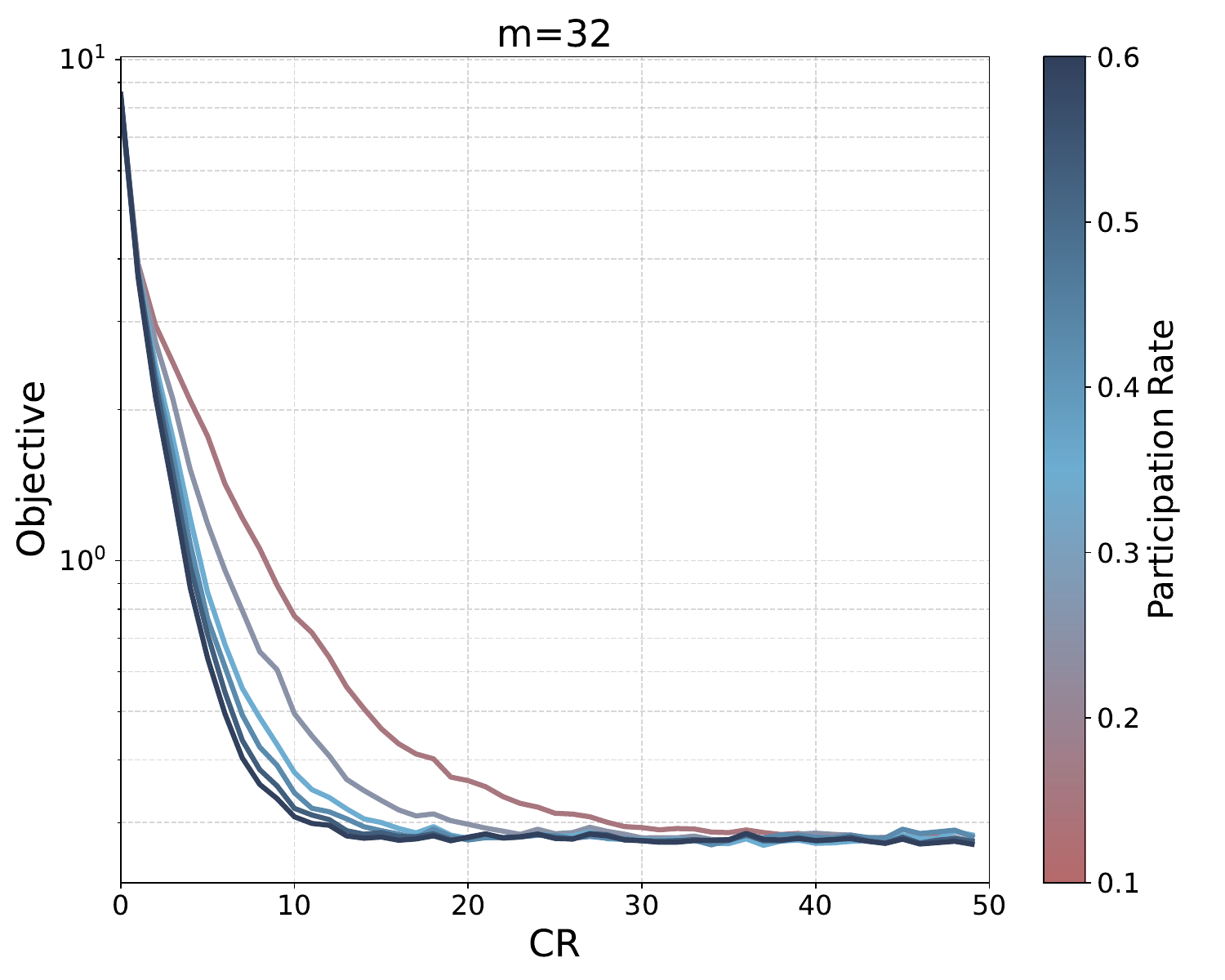}~
		\includegraphics[width=.328\textwidth]{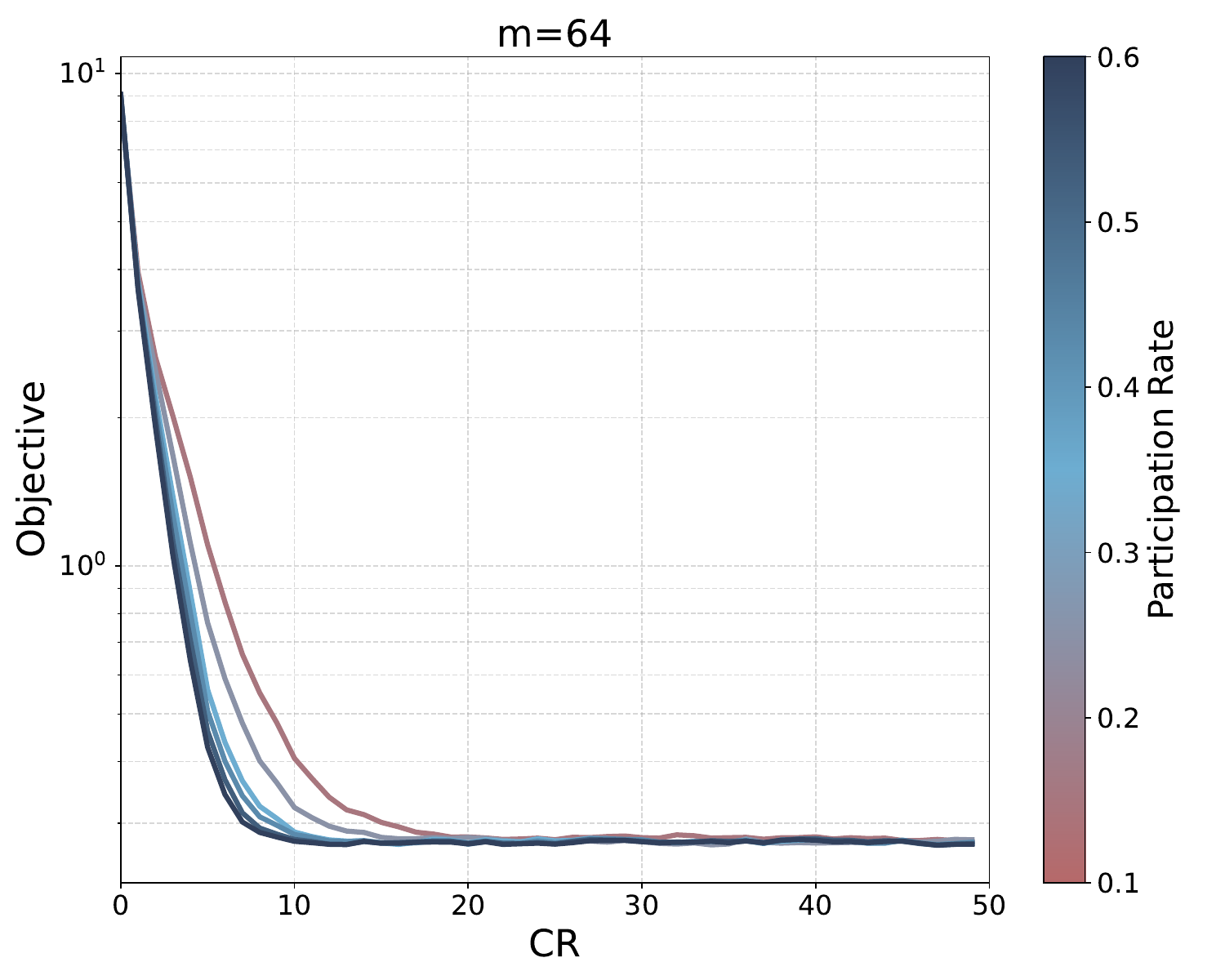}~
		\includegraphics[width=.328\textwidth]{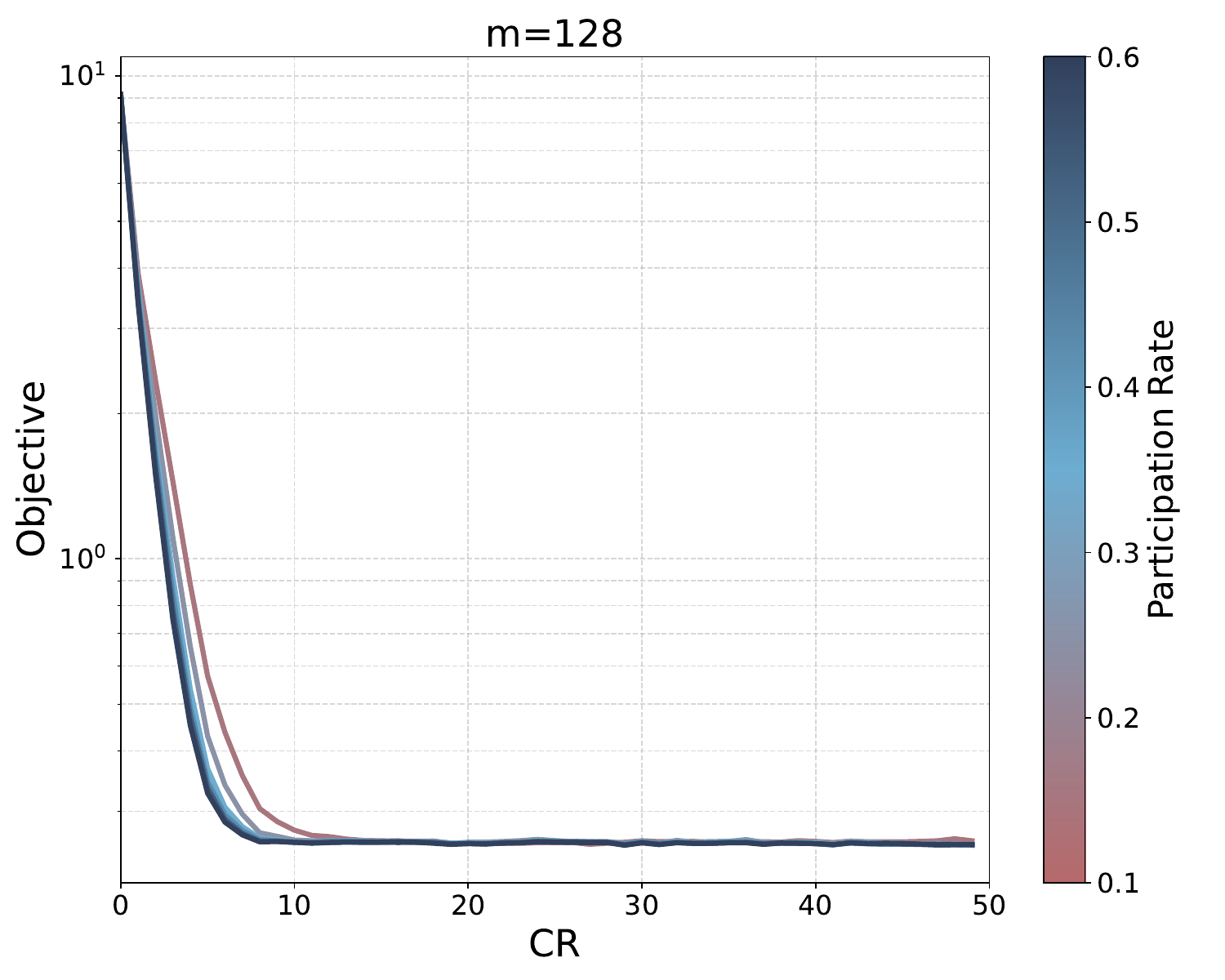}
		\caption{Convergence curves of different participation rate for PaME solving Example \ref{eg5.1}.}\vspace{-2mm}
		\label{ex1-pratecurve}
	\end{figure*} 
	
	{\it 2) Effect of participation rate:} Fig.~\ref{fig:participation} illustrates the impact of varying participation rate ${\nu \in \{0.1,0.2,\cdots,0.6\}}$ across different numbers of nodes ${m \in \{32,64,128\}}$, with the transmission rate fixed at ${s/n = 0.2}$. From the figure, while higher participation rates generally yield better model accuracy, the impact on the final objective value is relatively modest provided the rate is not extremely low; notably, all configurations eventually achieve convergence. The convergence trajectories in Fig.~\ref{ex1-pratecurve} further confirm that higher participation rates accelerate convergence. However, this improvement exhibits clear diminishing returns: increasing $r$ from $0.1$ to $0.2$ significantly boosts convergence speed, whereas further increases yield only marginal gains. Given that higher participation implies a heavier communication load, these findings suggest the existence of a favorable trade-off point where a relatively low participation rate suffices for reasonably fast convergence.

	{\it 3) Effect of communication period (CP):} Fig.~\ref{fig:ex1crhomo} and Fig.~\ref{fig:ex1crheter} illustrate the impact of the  CP under homogeneous and heterogeneous settings. In the homogeneous setting, all nodes adopt a uniform period $\kappa_i = k_0 \in \{2^0, 2^1, \cdots, 2^4\}$, whereas in the heterogeneous setting, nodes operate with distinct communication intervals whose median values correspond to those in the homogeneous case. The results indicate that the CP has minimal influence on the final objective value  although larger periods in the heterogeneous setting may induce slight instability. The convergence trajectories in Fig.~\ref{ex1-CR-homo} and Fig.~\ref{ex1-CR-heter} further reveal that while shorter CPs generally accelerate convergence, this effect becomes less pronounced as the number of nodes $m$ increases. Furthermore,  heterogeneous CPs consistently yield slightly slower convergence than the homogeneous setting across all configurations.

	{\it 4) Effect of graph connectivity:} Fig.~\ref{ex1-degree_vs_s} examines the impact of graph connectivity by fixing the number of nodes to ${m=64}$ and the participation rate to ${\nu=0.4}$  while varying the transmission rate and the normalized graph degree (i.e., the maximum number of neighbors per node, normalized to $[0,1]$). The three heatmaps report the final objective value, the number of iterations required for convergence, and the corresponding wall-clock time. The first heatmap indicates that increasing either $s$ or the graph degree generally lowers the final objective, corresponding to improved model performance. Notably, unless the graph is extremely sparse or the transmission rate is minimal, the final objective quickly saturates and becomes nearly insensitive to further increases in either parameter. This suggests that PaME can achieve near-optimal performance even with a relatively low transmission rate, provided that the communication graph maintains moderate connectivity, thereby reinforcing its communication efficiency. The iteration and time heatmaps further corroborate this trend: except for the most challenging regime characterized by very sparse connectivity and small $s$, PaME demonstrates consistent convergence behavior (requiring approximately $200$ iterations) and comparable runtime across a broad range of settings.
	
{\it 5) Effect of larger $m$:} Table~\ref{tab:pame-self-comparison-msv}
	further evaluates PaME under larger network sizes, namely $m=1000$ and
	$m=2000$, to complement the graph-based experiments with $m\leq 128$.
	Here, `MSE'' denotes the mean squared error obtained after $100$ training iterations, and `Iter.'' denotes the number of iterations required to reach the prescribed MSE level of $0.5$.
	The left part of the table reports the result with different
	transmission ratio $s$, while the right part reports the results with different neighbor participation rate $v$. The results show that PaME remains stable and effective when the number of nodes is increased to the thousand scale. In particular, increasing $s$ consistently reduces both the final MSE and the number of required iterations. These results indicate that the proposed partial message exchange mechanism can still maintain favorable convergence behavior under substantially larger decentralized networks.

	\begin{figure*}[t]
		\centering
		\includegraphics[width=.325\textwidth]{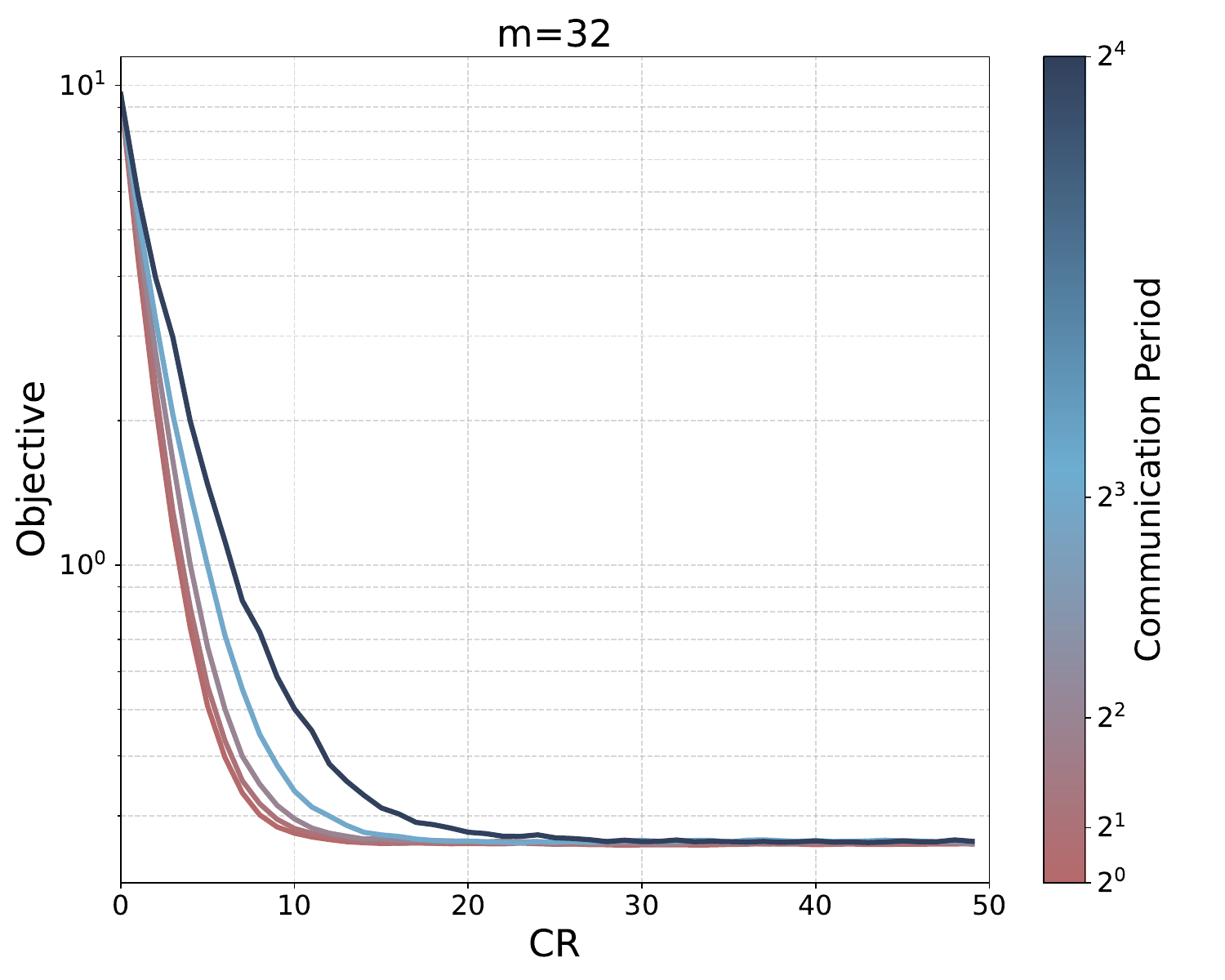}~
		\includegraphics[width=.325\textwidth]{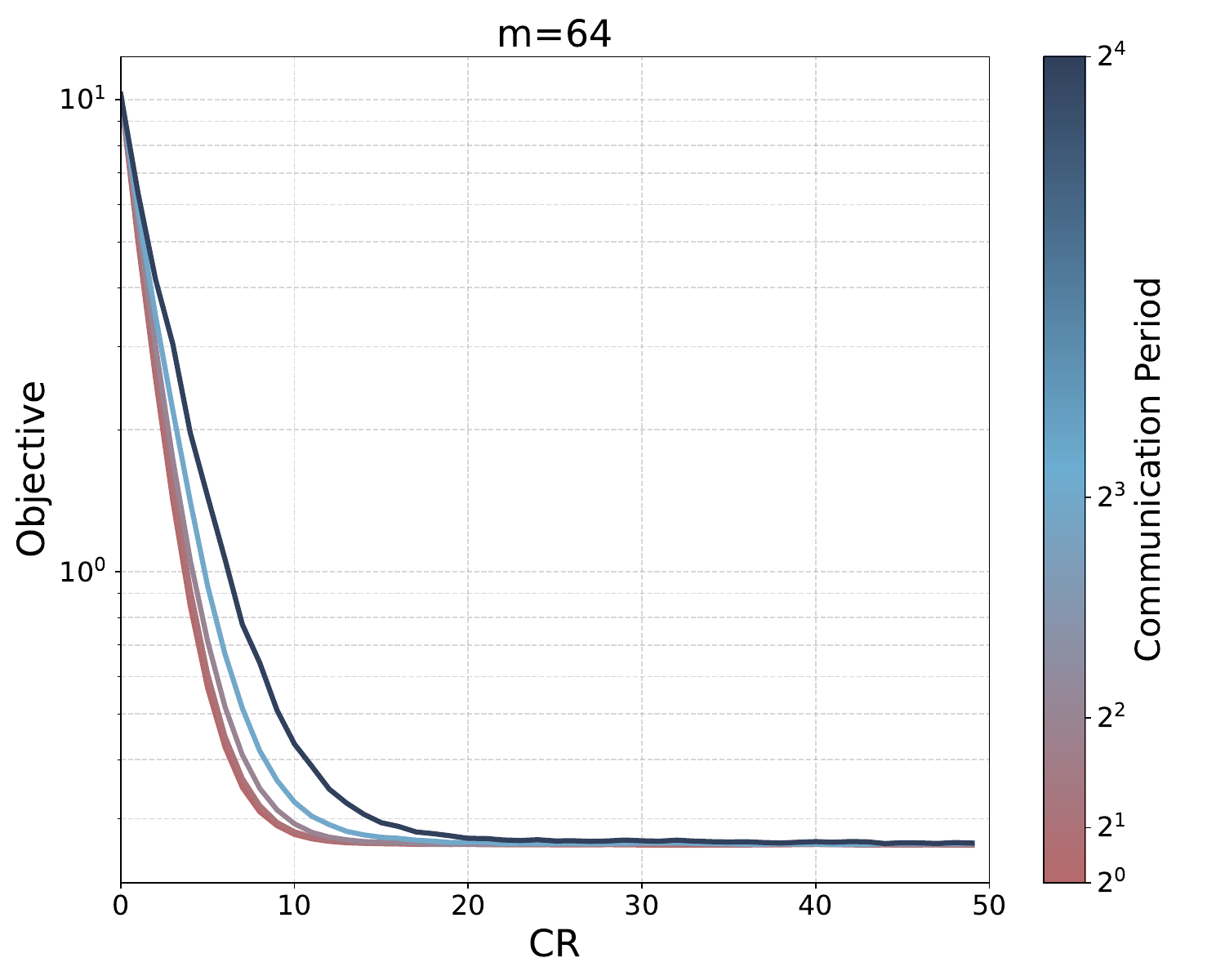}~
		\includegraphics[width=.325\textwidth]{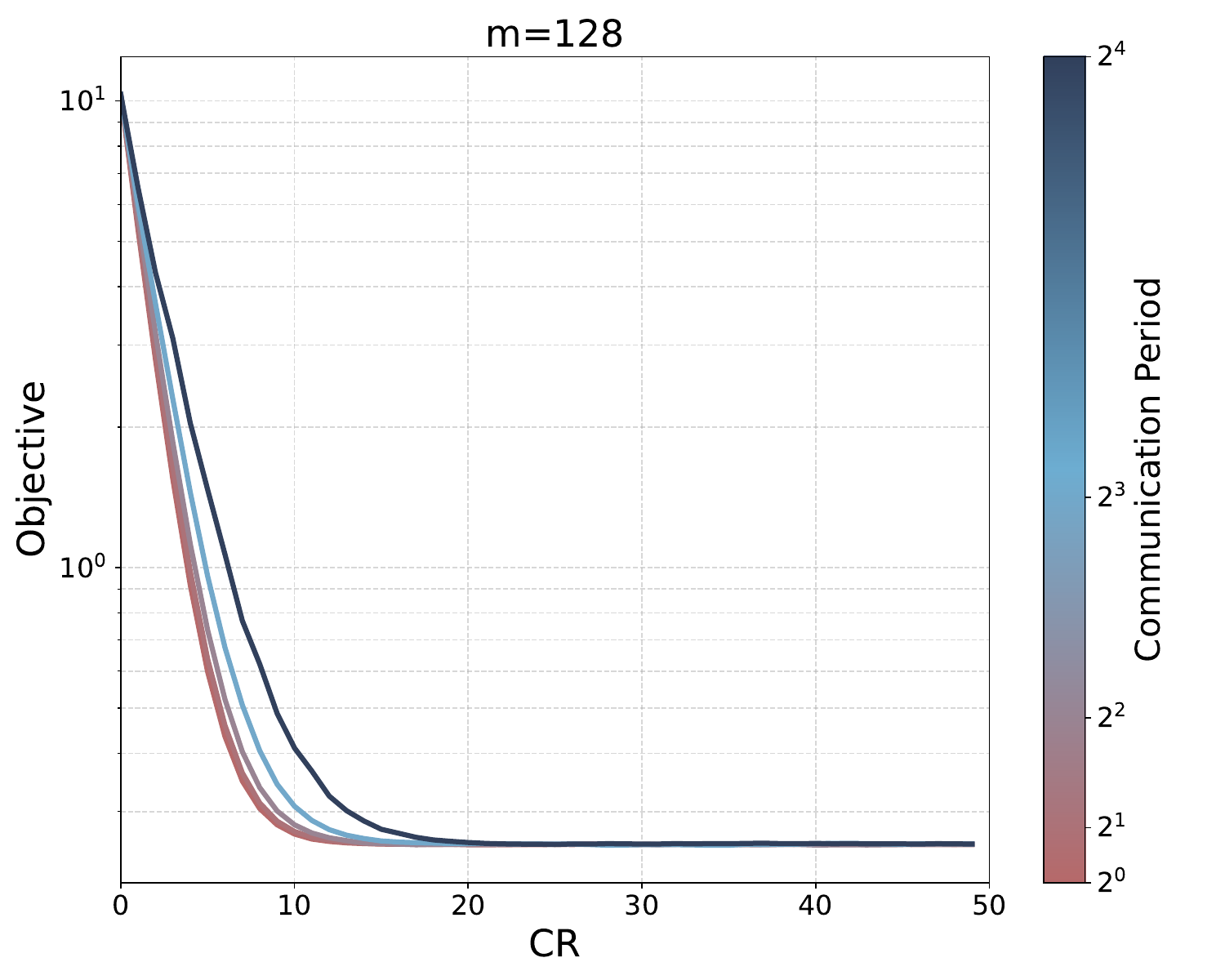}
		\caption{Convergence curves of different CPs (homogeneous) for PaME solving Example \ref{eg5.1}.}\vspace{-2mm}
		\label{ex1-CR-homo}
	\end{figure*} 
	
		\begin{figure*}[t]
		\centering
		\includegraphics[width=.325\textwidth]{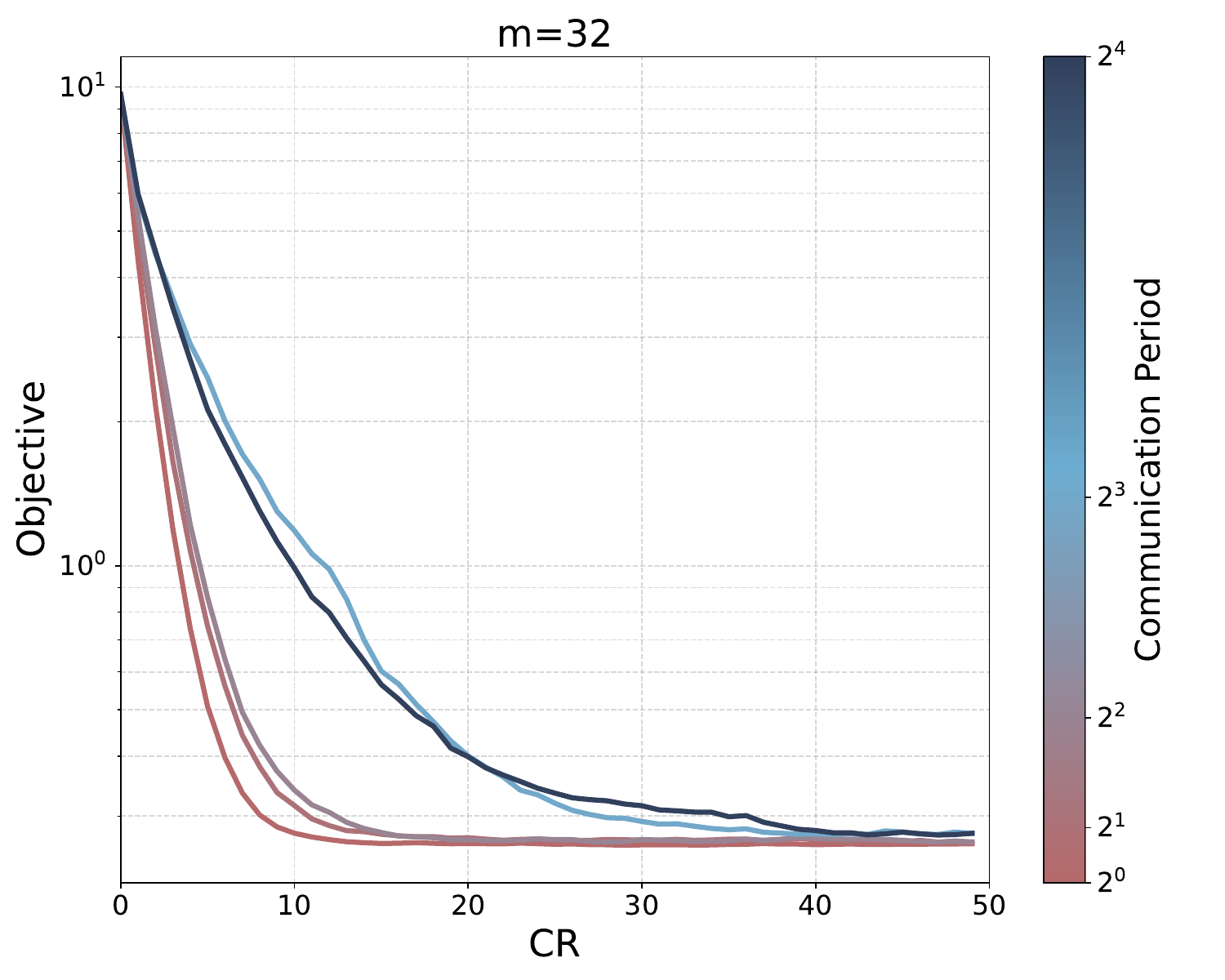}~
		\includegraphics[width=.325\textwidth]{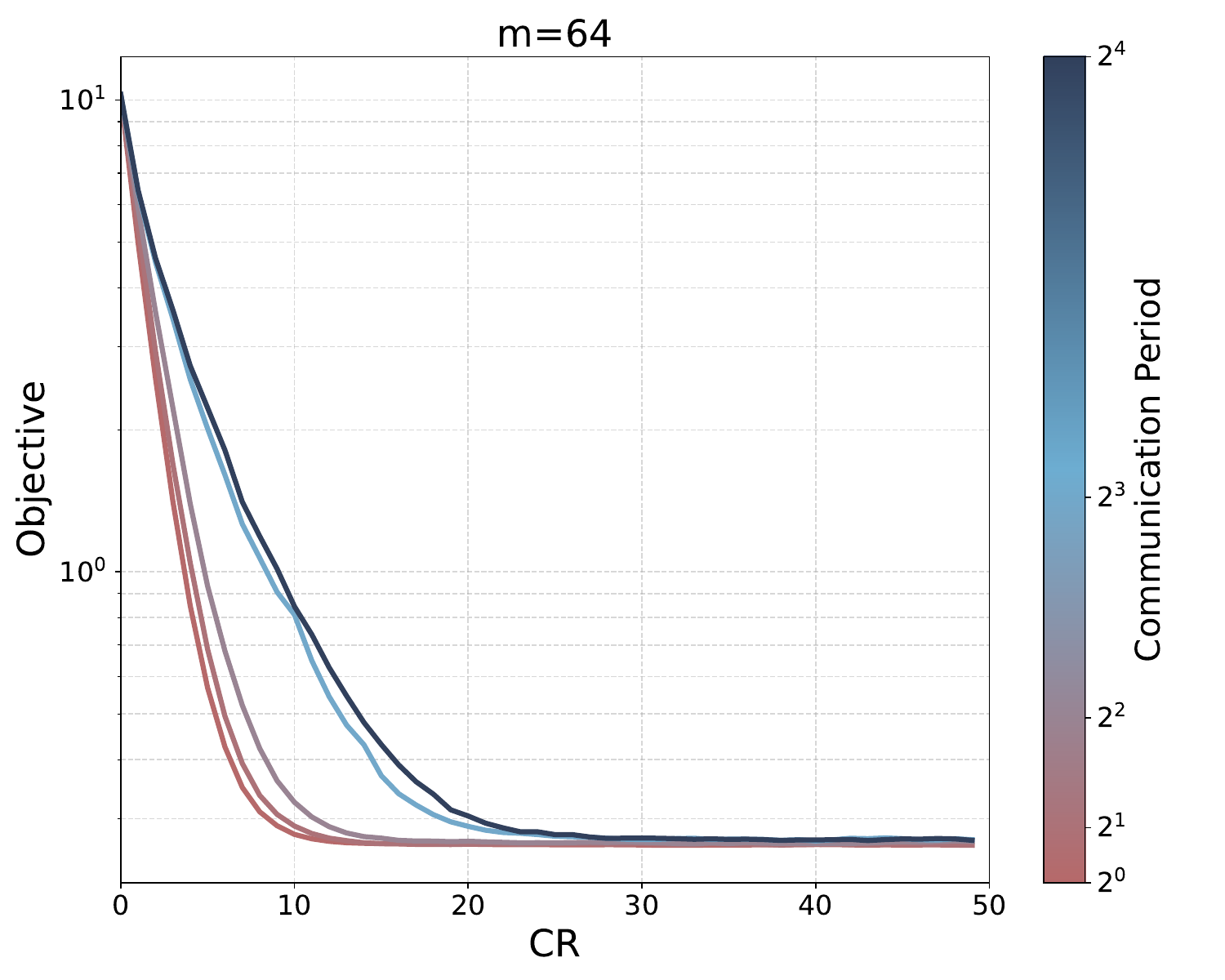}~
		\includegraphics[width=.325\textwidth]{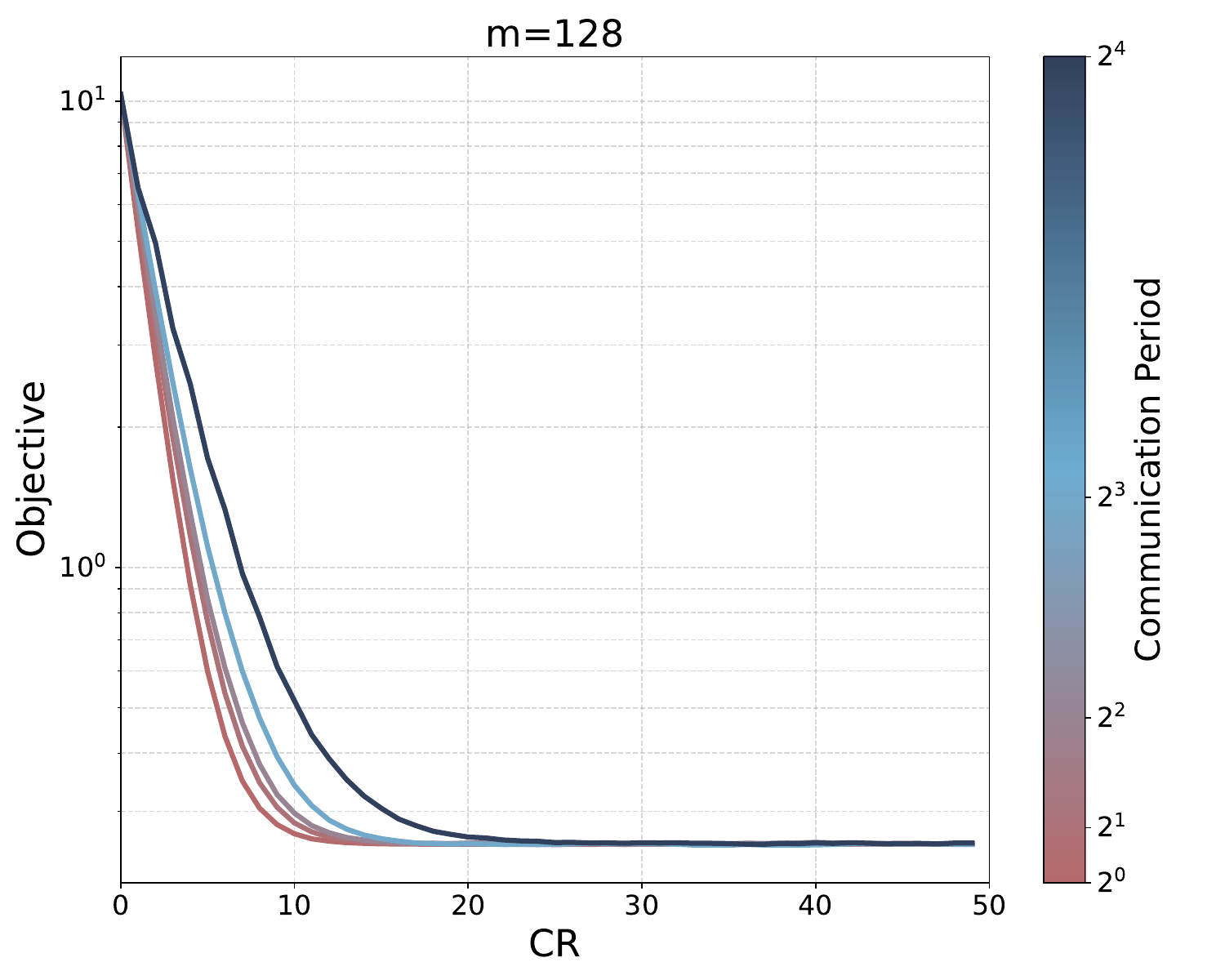}
		\caption{Convergence curves of different CPs (heterogeneous) for PaME solving Example \ref{eg5.1}.}\vspace{-2mm}
		\label{ex1-CR-heter}
	\end{figure*} 
	
		\begin{figure*}[t]
		\centering
		\includegraphics[width=.325\textwidth]{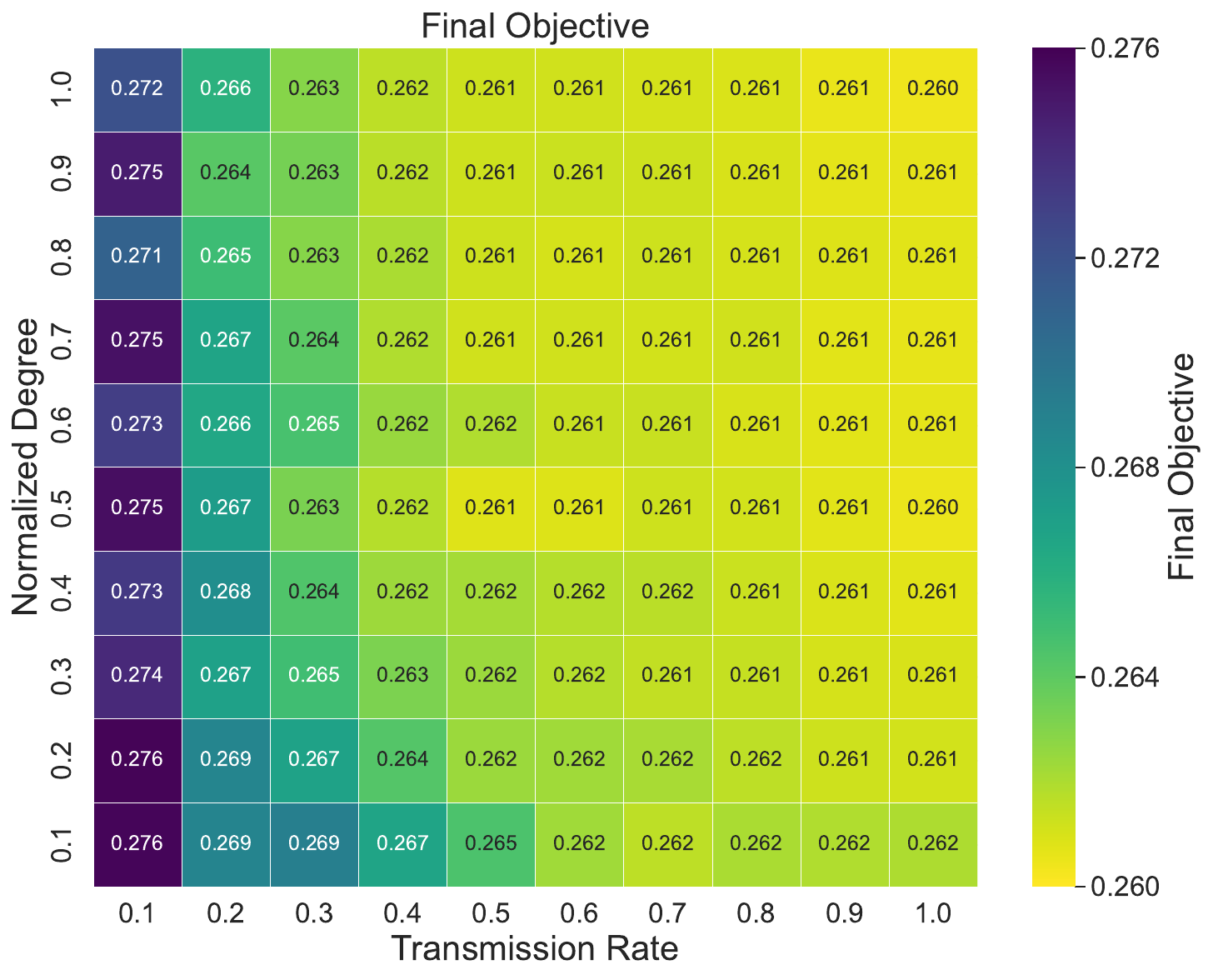}~
		\includegraphics[width=.325\textwidth]{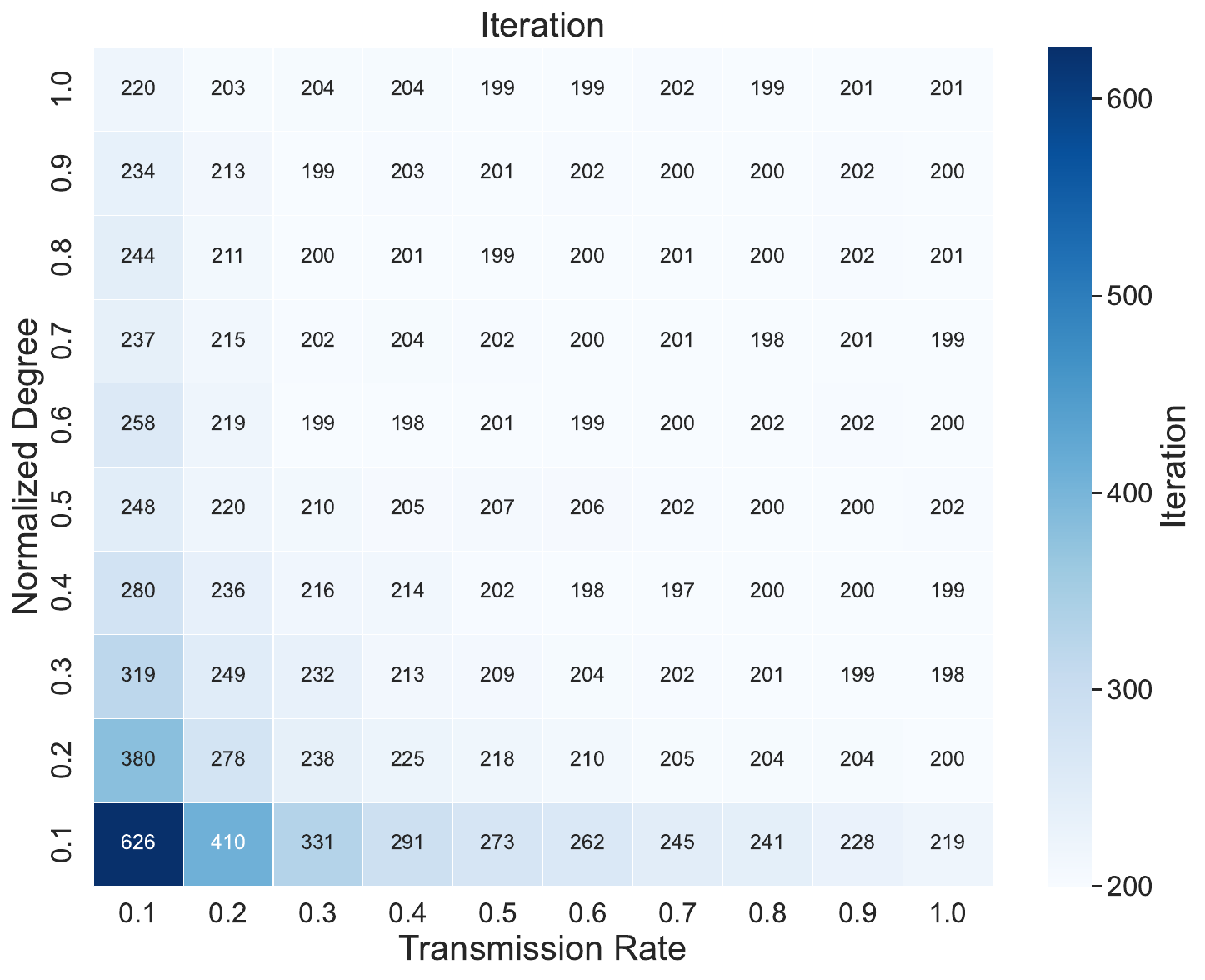}~
		\includegraphics[width=.325\textwidth]{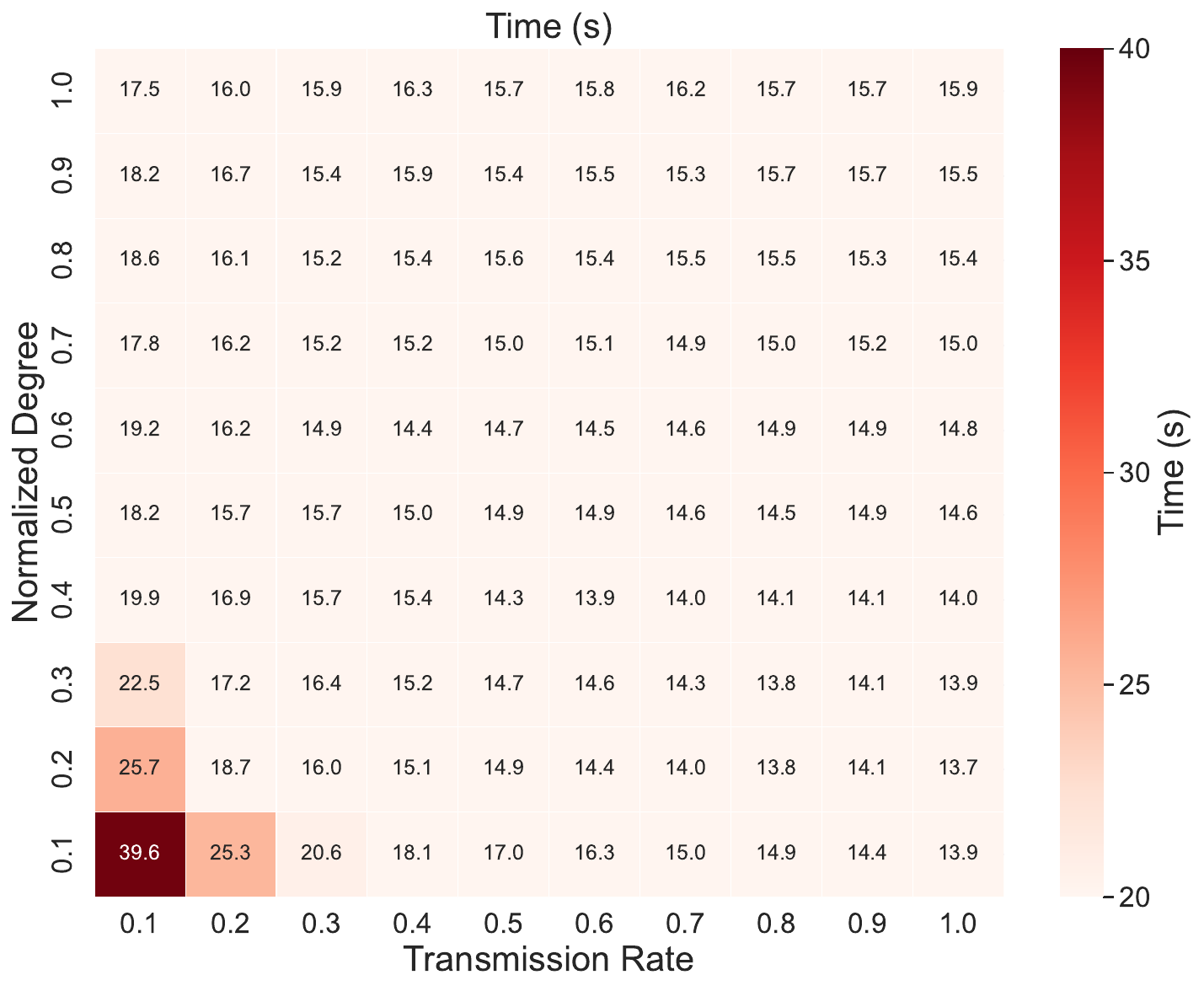}
		\caption{Heatmap of final objective, iteration, and runtime for PaME solving Example \ref{eg5.1}.}\vspace{-2mm}
		\label{ex1-degree_vs_s}
	\end{figure*} 

	\begin{table}[!t]
		\renewcommand{\arraystretch}{1.25}
		\addtolength{\tabcolsep}{-2pt}
		\caption{Effect of larger $m$ for PaME solving Example \ref{eg5.1}.
			\label{tab:pame-self-comparison-msv}}
		\vspace{-2mm}
		\begin{center}
			\begin{tabular}{lcc c cc c lcc c cc}
				\hline
				& \multicolumn{2}{c}{$m=1000$}
				&& \multicolumn{2}{c}{$m=2000$}
				&& 
				& \multicolumn{2}{c}{$m=1000$}
				&& \multicolumn{2}{c}{$m=2000$} \\
				\cline{2-3} \cline{5-6} \cline{9-10} \cline{12-13}
				$s$
				& MSE & Iter.
				&& MSE & Iter.
				&& $v$
				& MSE & Iter.
				&& MSE & Iter. \\
				\hline
				0.1 & 0.462 & 96 && 0.136 & 64
				&& 0.1 & 0.136 & 63 && 0.029 & 45 \\
				
				0.2 & 0.136 & 63 && 0.029 & 45
				&& 0.2 & 0.030 & 46 && 0.013 & 36 \\
				
				0.3 & 0.054 & 51 && 0.015 & 39
				&& 0.3 & 0.016 & 40 && 0.011 & 32 \\
				
				0.4 & 0.029 & 45 && 0.012 & 36
				&& 0.4 & 0.013 & 36 && 0.011 & 320 \\
				
				0.5 & 0.019 & 41 && 0.011 & 30
				&& 0.5 & 0.011 & 34 && 0.010 & 29 \\
				
				0.6 & 0.015 & 39 && 0.011 & 31
				&& 0.6 & 0.011 & 32 && 0.010 & 28 \\
				\hline
			\end{tabular}
		\end{center}
		\vspace{-3mm}
	\end{table}
	
	\subsection{Comparison with other DFL algorithms}
	To further evaluate the performance of PaME, we benchmark it against four state-of-the-art baseline methods: D-PSGD \cite{lian2017can} and DFedSAM \cite{shi2023improving},  BEER \cite{zhao2022beer}, and ANQ-NIDS \cite{michelusi2022finite}. The latter two algorithms incorporate compression techniques to enhance communication efficiency. To ensure fair comparison, the communication period and participation rate are standardized across all methods.
	
 {\it 6) Comparison of convergence speed:} Fig.~\ref{ex2-acc-vs-cr} presents the convergence curves for Example~\ref{eg5.2}, with the number of nodes varying across ${m \in \{32,64,128\}}$ and the model dimension fixed at ${n = 1000}$. The results indicate that increasing $m$ accelerates convergence. PaME consistently outperforms all competing algorithms in terms of convergence speed, requiring the fewest communication rounds to achieve the highest test accuracy. DFedSAM generally secures the second-best performance, whereas BEER and ANQ-NIDS exhibit slower convergence, likely due to information loss induced by compression.
	
	 {\it 7) Comparison of communication efficiency:}
	Fig.~\ref{ex2-cr-vs-m} and Fig.~\ref{ex2-dtv-vs-m} assess the communication efficiency for Example~\ref{eg5.2} when varying number of nodes ${m \in \{32,64,128\}}$ and model dimension ${n \in \{1000,5000,10000\}}$. Fig.~\ref{ex2-cr-vs-m} reports the number of communication rounds (CR) each algorithm needs to reach convergence, and PaME consistently requires the fewest CR across all configurations. Fig.~\ref{ex2-dtv-vs-m} further compares the total data transmission volume required by each algorithm to reach convergence. Across all examined settings, PaME attains the lowest communication cost, typically reducing the transmitted volume by at least $50\%$ compared with the other methods in most configurations. Together with the convergence results, these findings indicate that PaME achieves a more favorable trade-off between accuracy and communication efficiency.
	
		\begin{figure*}[t]
		\centering
		\includegraphics[width=.325\textwidth]{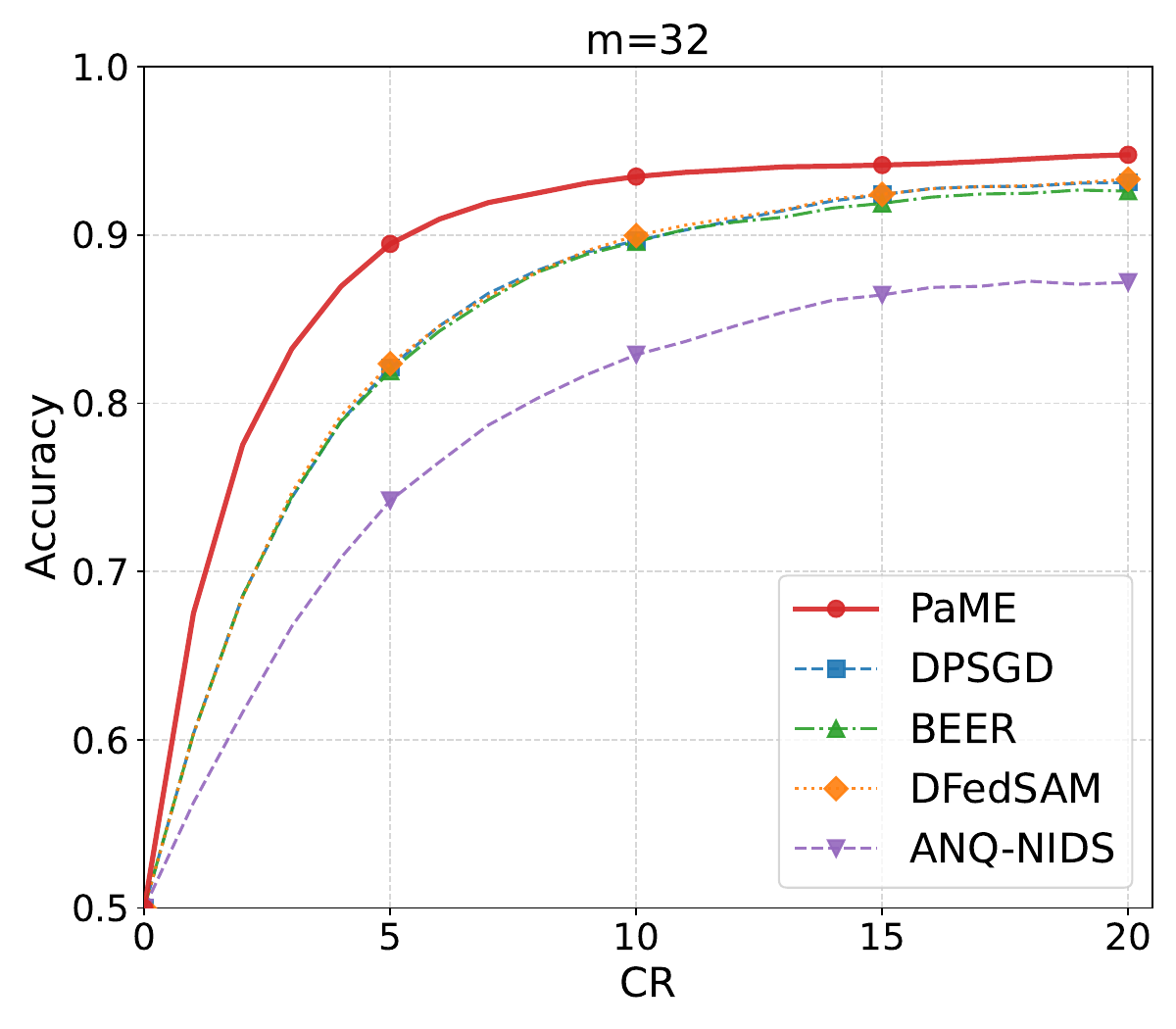}~
		\includegraphics[width=.325\textwidth]{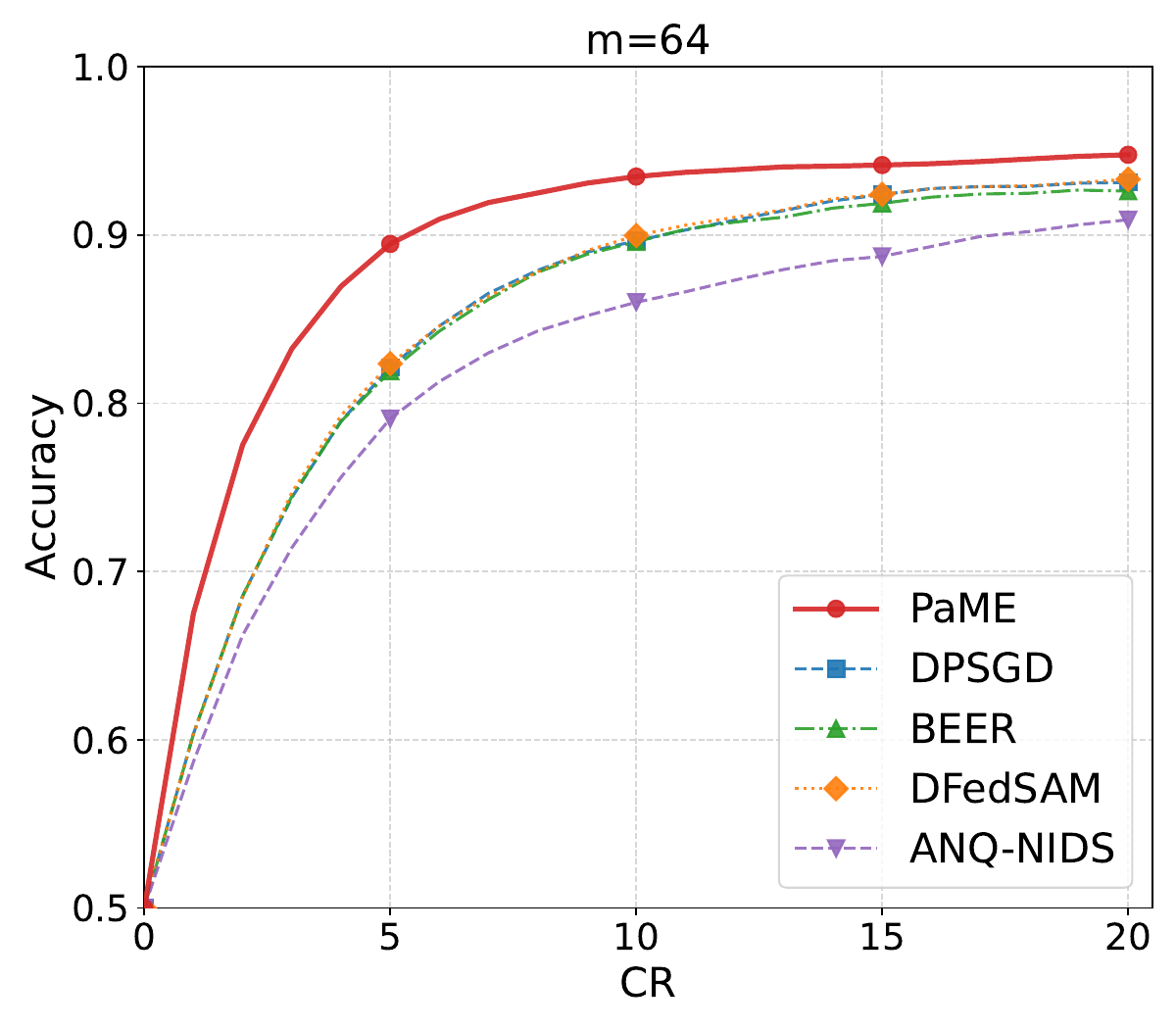}~
		\includegraphics[width=.325\textwidth]{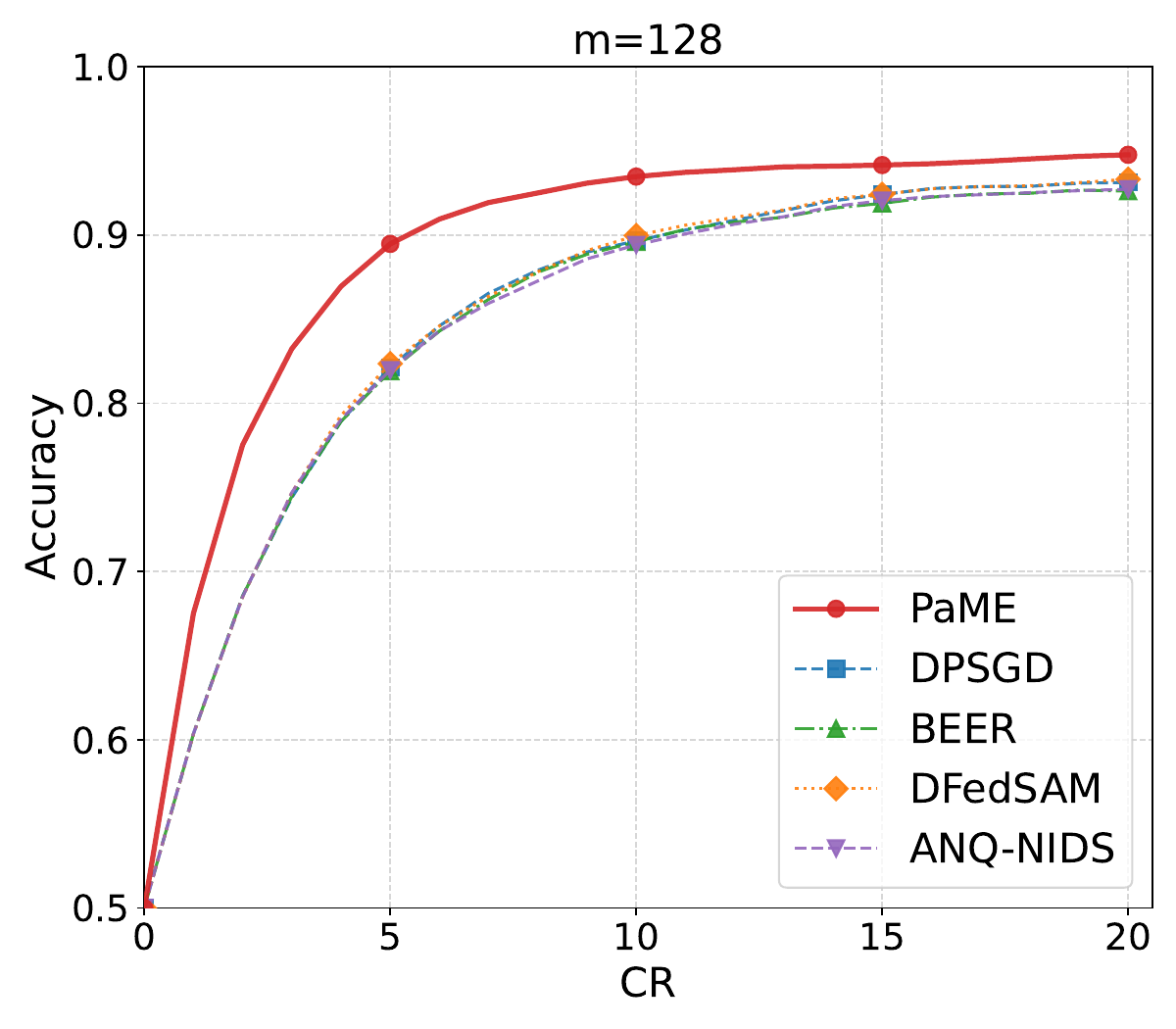}
		\caption{Accuracy v.s. CR for different algorithms solving Example \ref{eg5.2}.}\vspace{-1mm}
		\label{ex2-acc-vs-cr}
	\end{figure*} 
	
	\begin{figure*}[t]
		\centering
		\includegraphics[width=.325\textwidth]{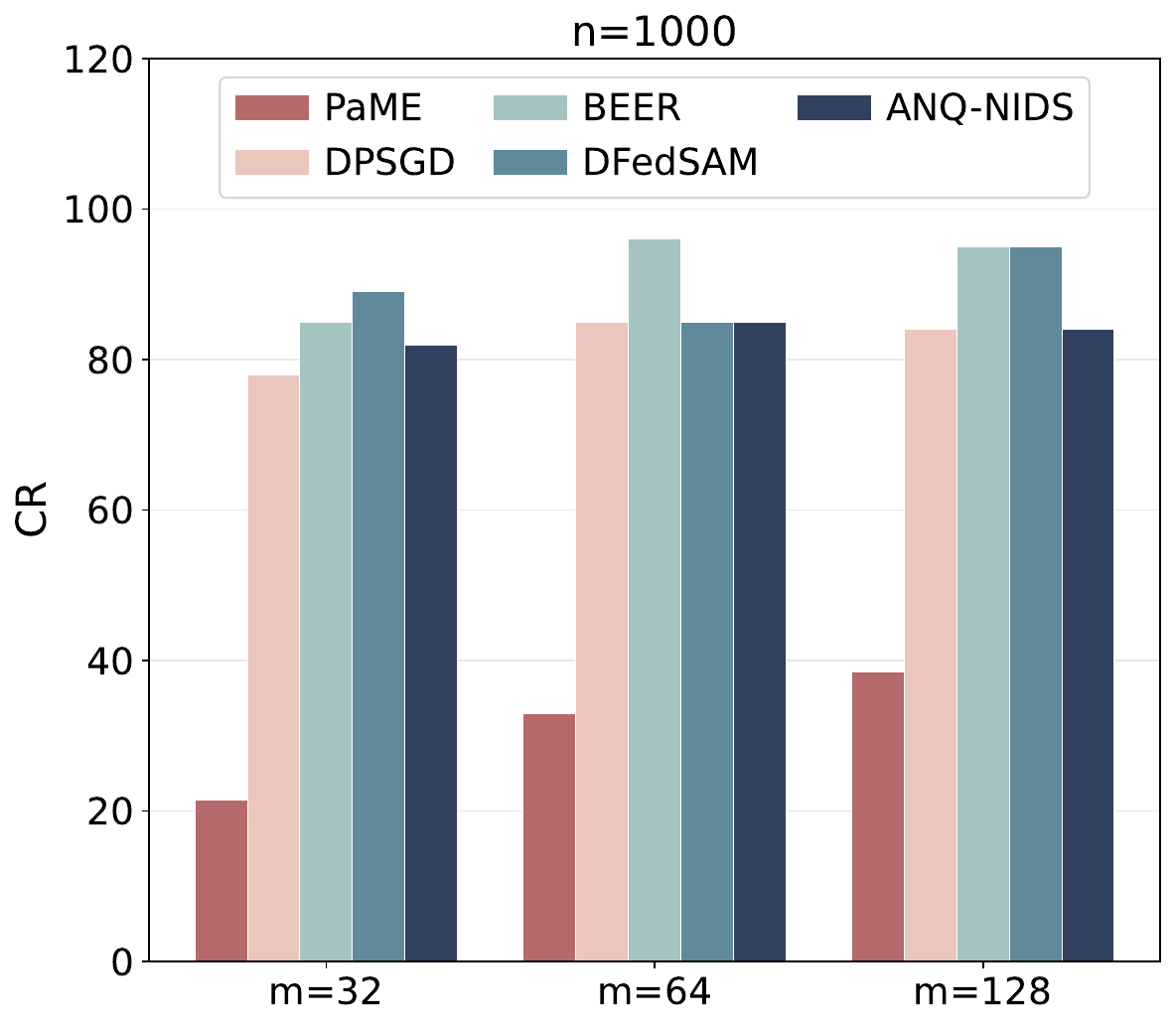}~
		\includegraphics[width=.325\textwidth]{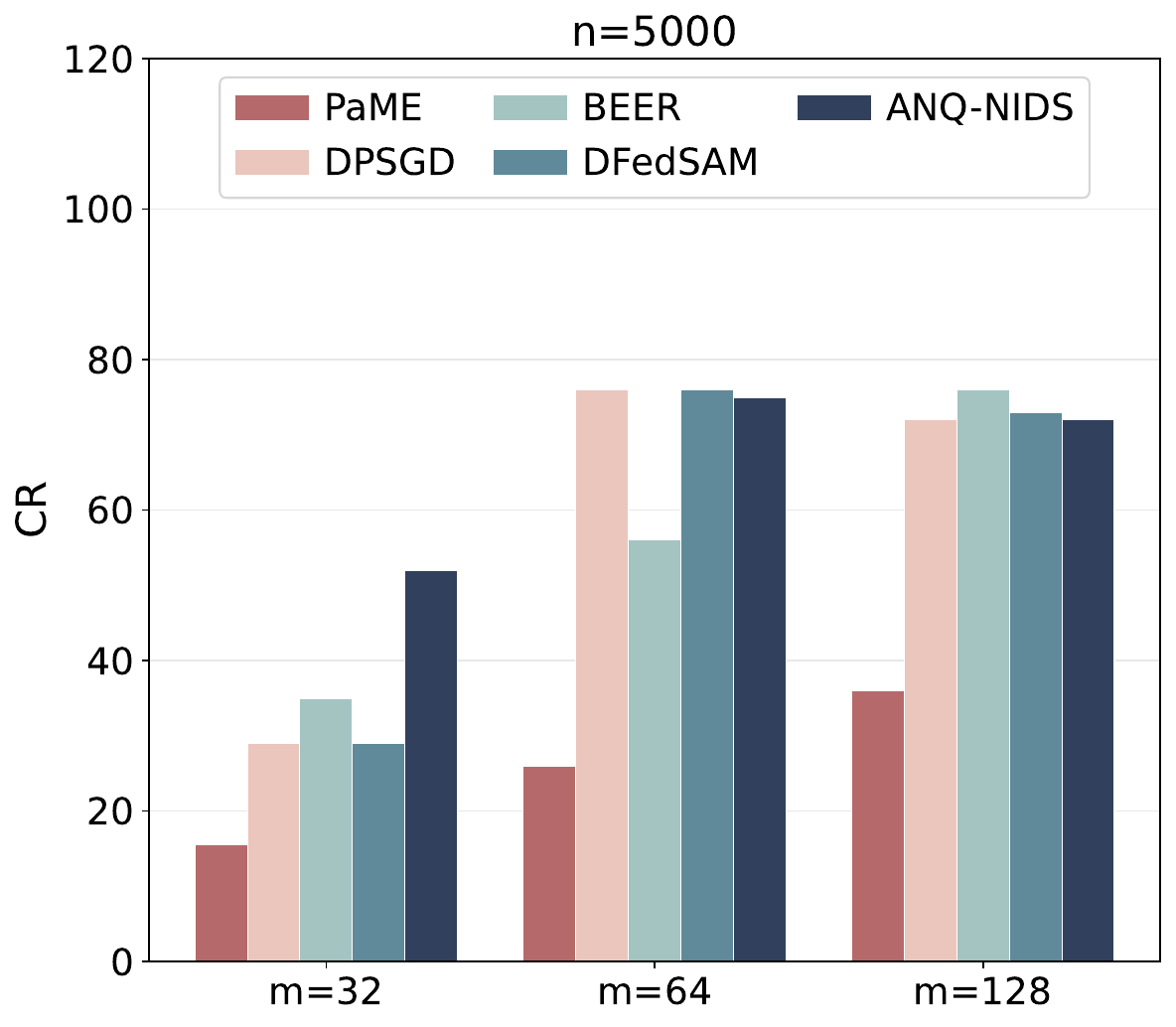}~
		\includegraphics[width=.325\textwidth]{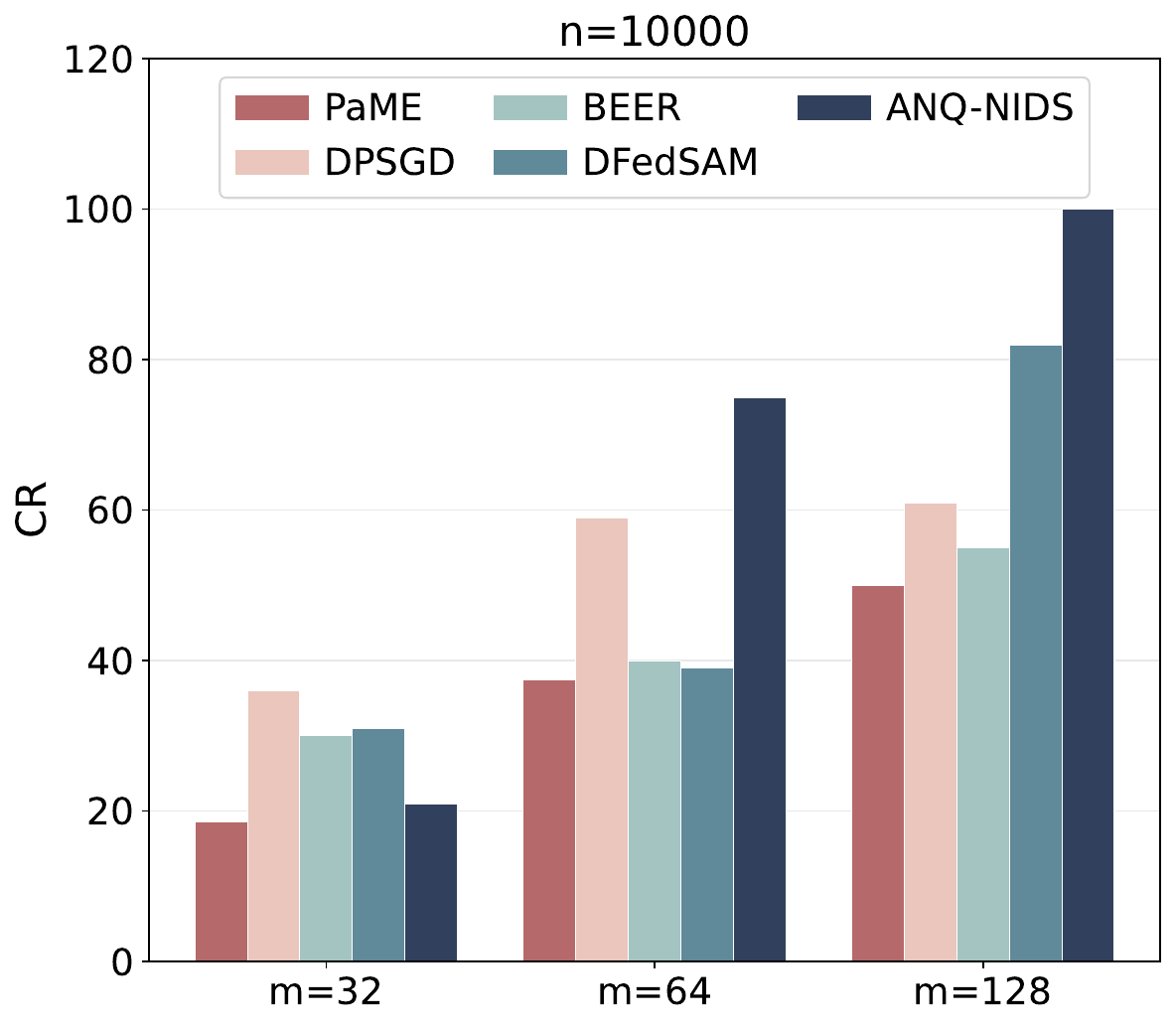}
		\caption{CR v.s. m for different algorithms solving Example \ref{eg5.2}.}\vspace{-2mm}
		\label{ex2-cr-vs-m} 
	\end{figure*} 
	
		\begin{figure*}[t]
		\centering
		\includegraphics[width=.325\textwidth]{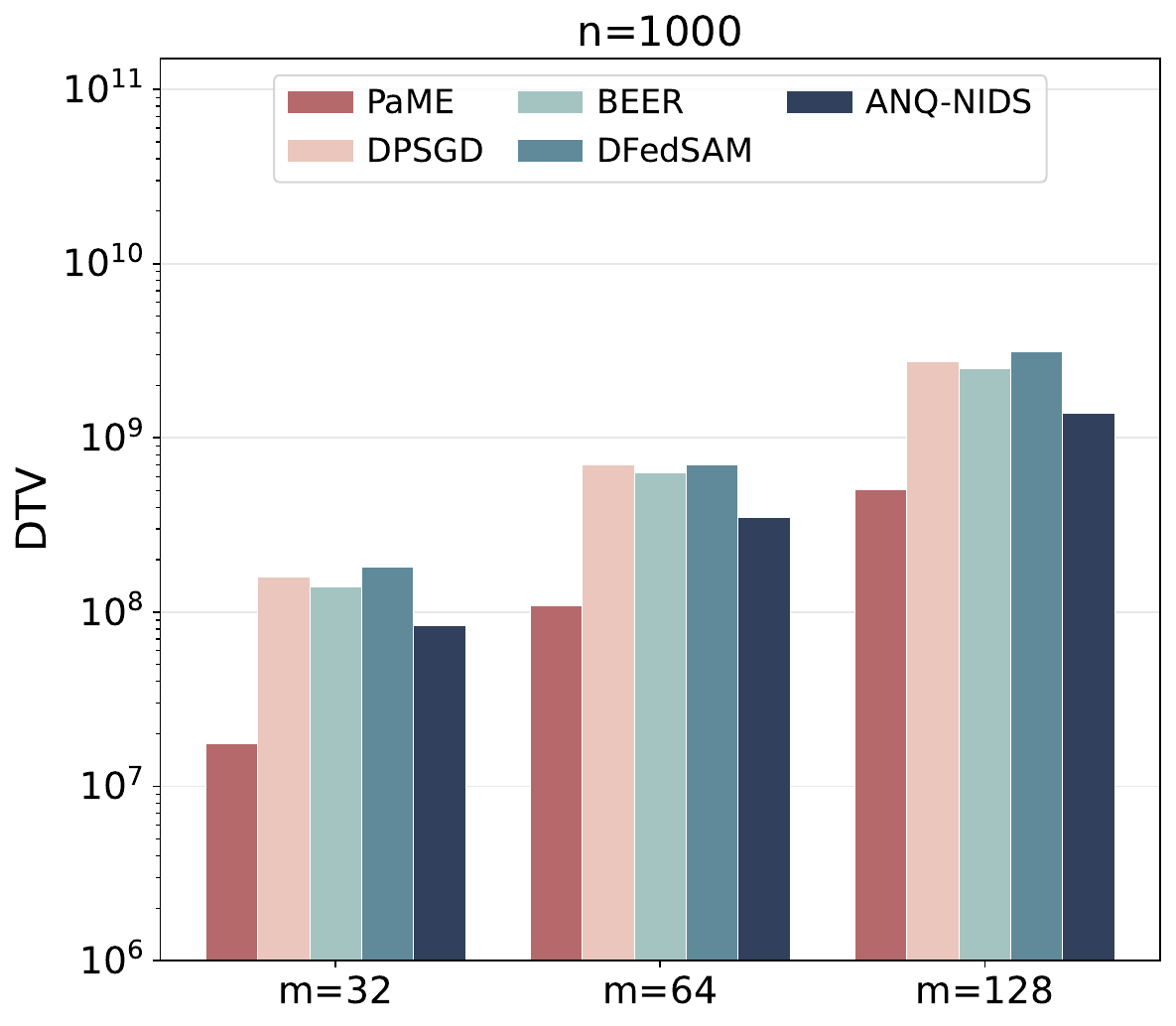}~
		\includegraphics[width=.325\textwidth]{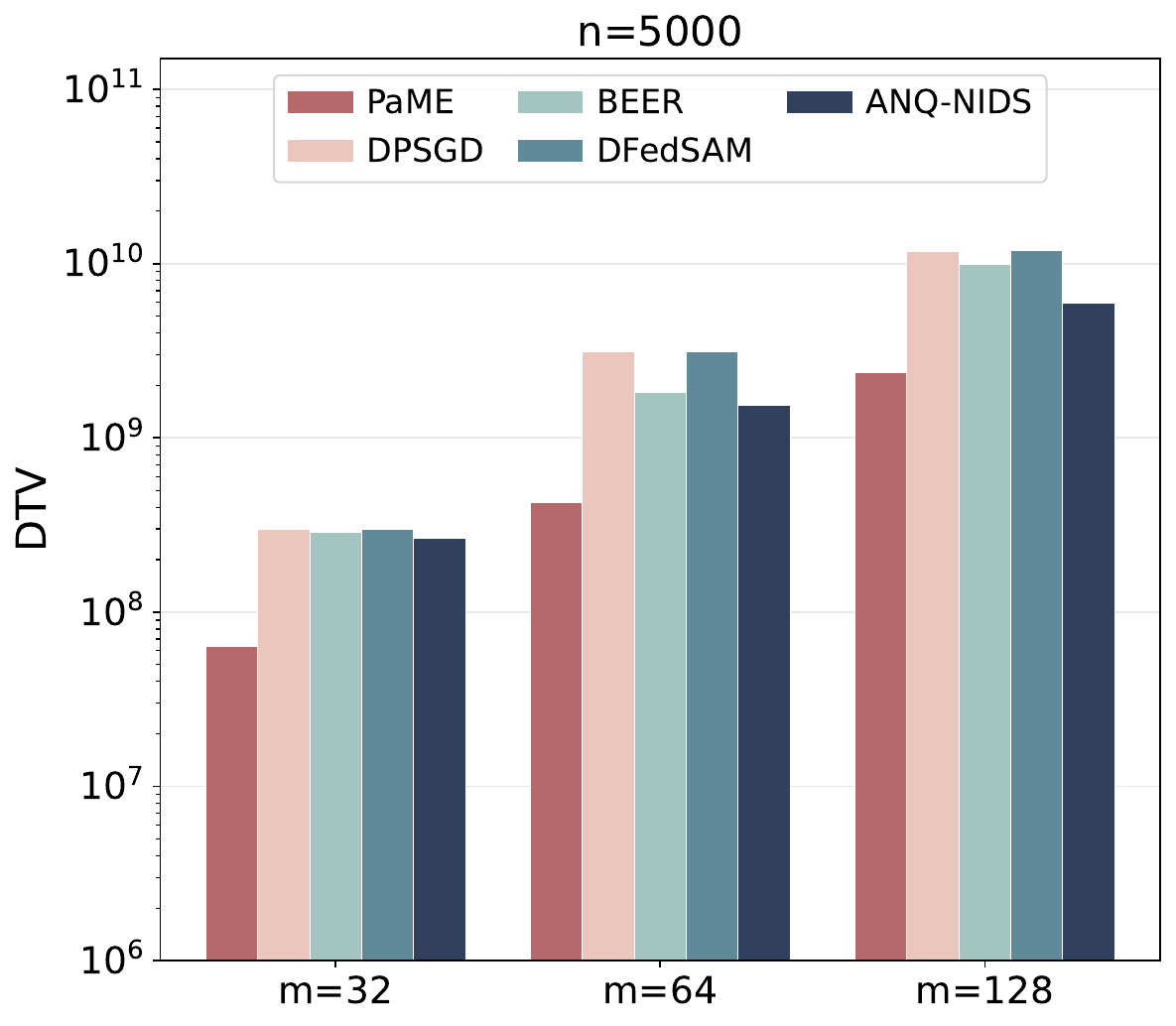}~
		\includegraphics[width=.325\textwidth]{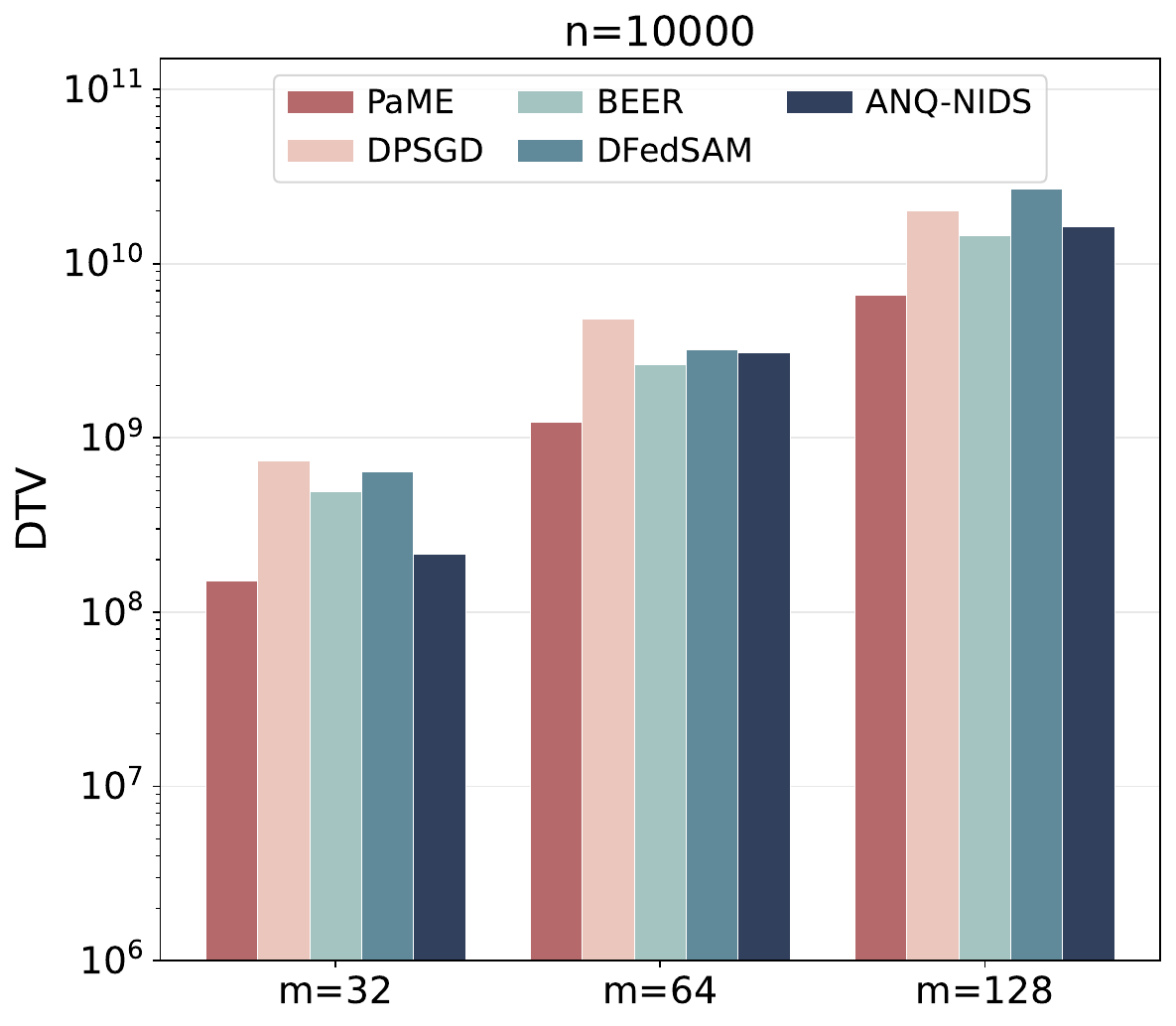}
		\caption{Data transmission volume (DTV) v.s. m for different algorithms solving Example \ref{eg5.2}.}
		\label{ex2-dtv-vs-m} \vspace{-2mm}
	\end{figure*} 
	
	 {\it 8) Comparison of data distribution:}
	Fig. \ref{ex3-acc} depicts the convergence trajectories for Example \ref{eg5.3} under varying degrees of data heterogeneity, simulated by assigning each node samples from ${C \in \{1, 7, 10\}}$ distinct categories, where a lower value of $C$ corresponds to a higher level of heterogeneity (non-IID). As expected, increased heterogeneity leads to performance degradation across all evaluated methods. Nevertheless, PaME consistently maintains the fastest convergence speed across all scenarios. Notably, even under the most extreme setting (i.e., ${C = 1}$), PaME outperforms competing algorithms and achieves superior final accuracy.
	
Fig. \ref{ex4-acc} extends this analysis to Example \ref{eg5.4}  on the CIFAR-10 dataset by comparing IID settings with Dirichlet partitions parameterized by $\beta \in \{0.3, 0.6\}$, where a lower $\beta$ signifies higher data heterogeneity. The results indicate that PaME consistently outperforms other decentralized baselines. Although higher levels of non-IID data inherently complicate model consensus, PaME exhibits robust convergence stability and achieves the highest accuracy among the compared methods.
	
	Fig.~\ref{ex5-acc} further evaluates the five algorithms on the more challenging Tiny-ImageNet setting under IID and Dirichlet non-IID partitions. The results show that PaME consistently achieves the highest accuracy across all three data settings. Although the non-IID partitions make the learning task more difficult and reduce the accuracy of all methods, PaME maintains stable convergence and a
	clear advantage over the compared decentralized baselines. This further confirms the robustness of PaME on more complex image classification tasks.

	\begin{figure*}[t]
		\centering
		\includegraphics[width=.325\textwidth]{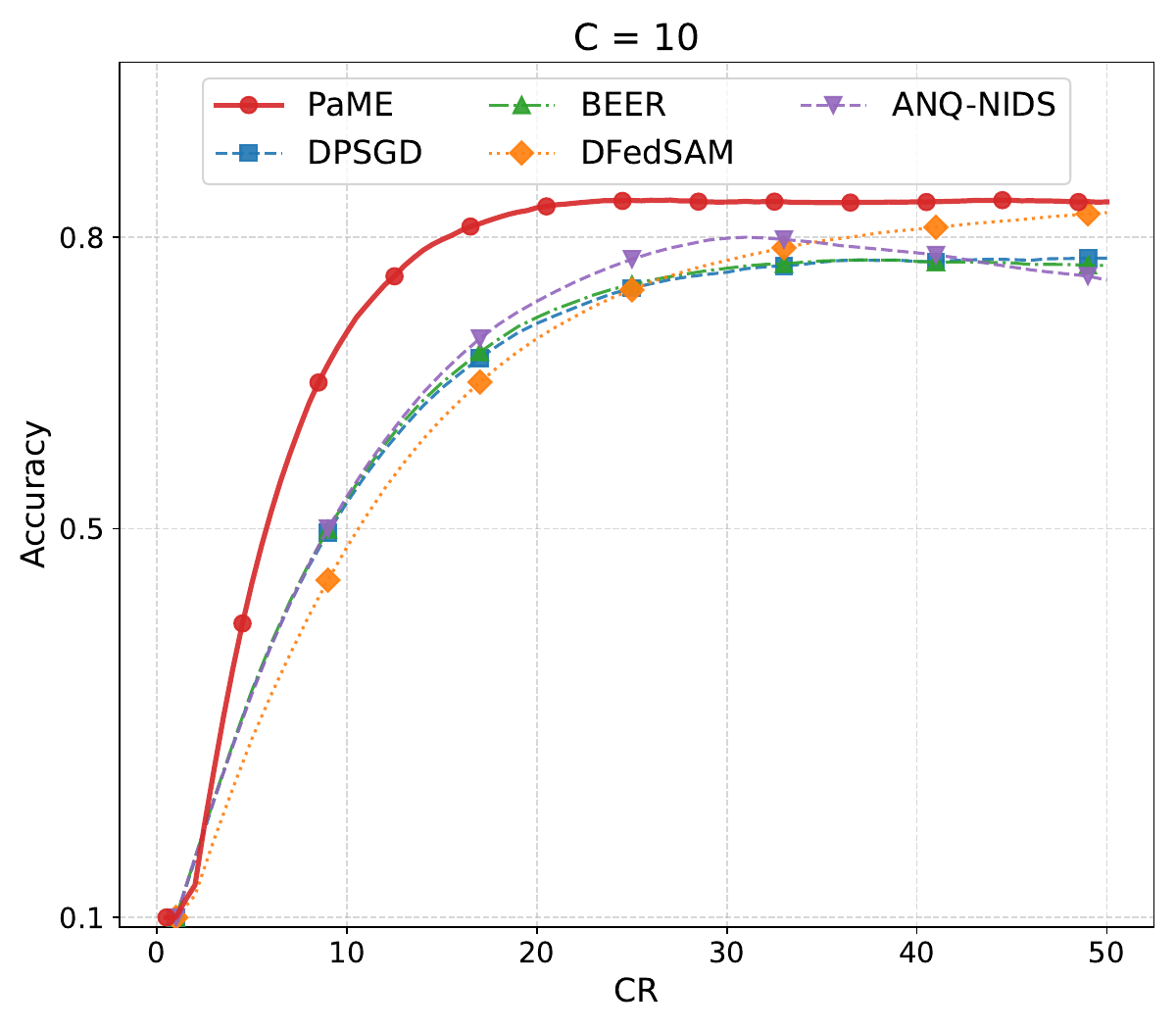}
		\includegraphics[width=.325\textwidth]{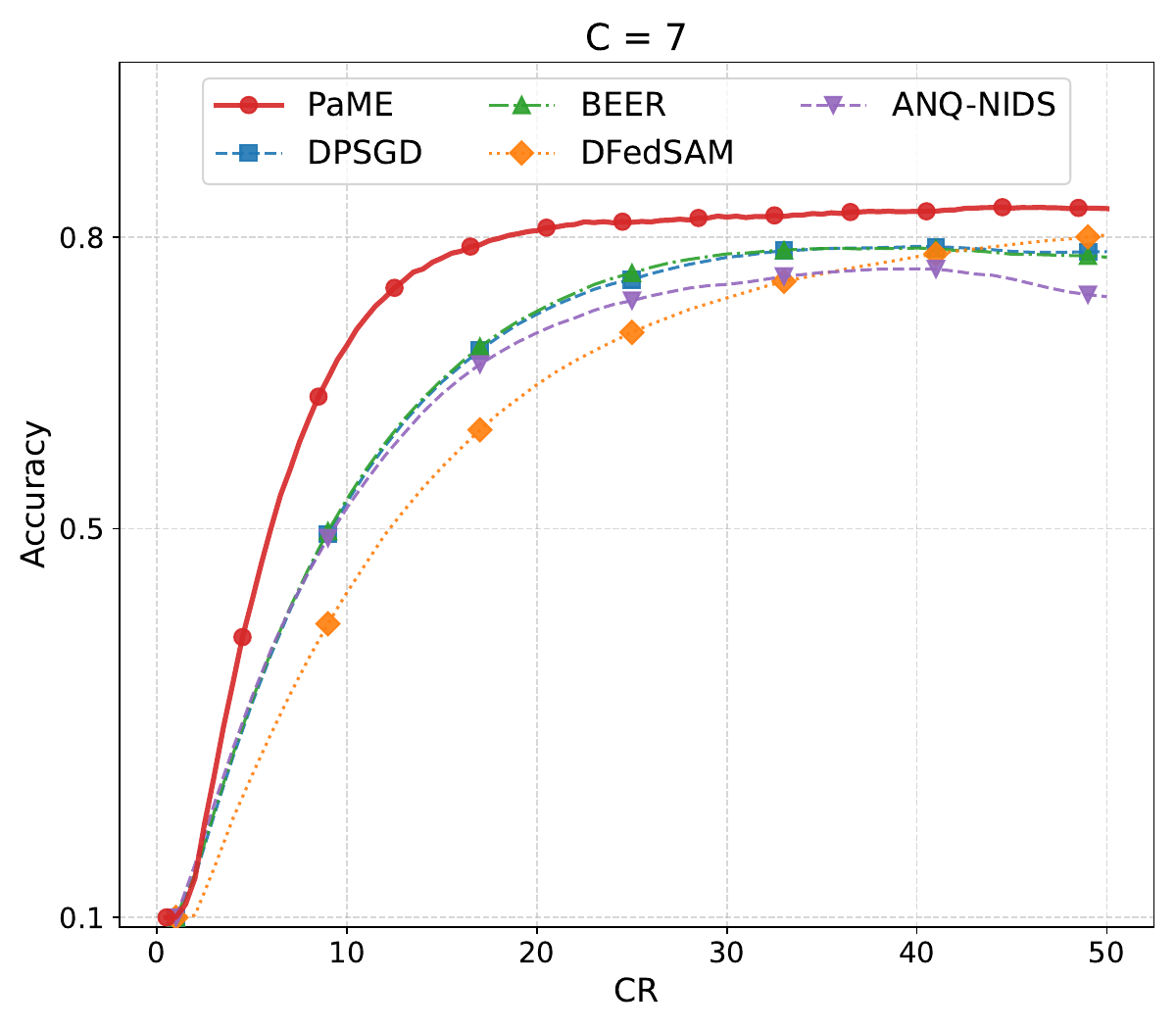}
		\includegraphics[width=.325\textwidth]{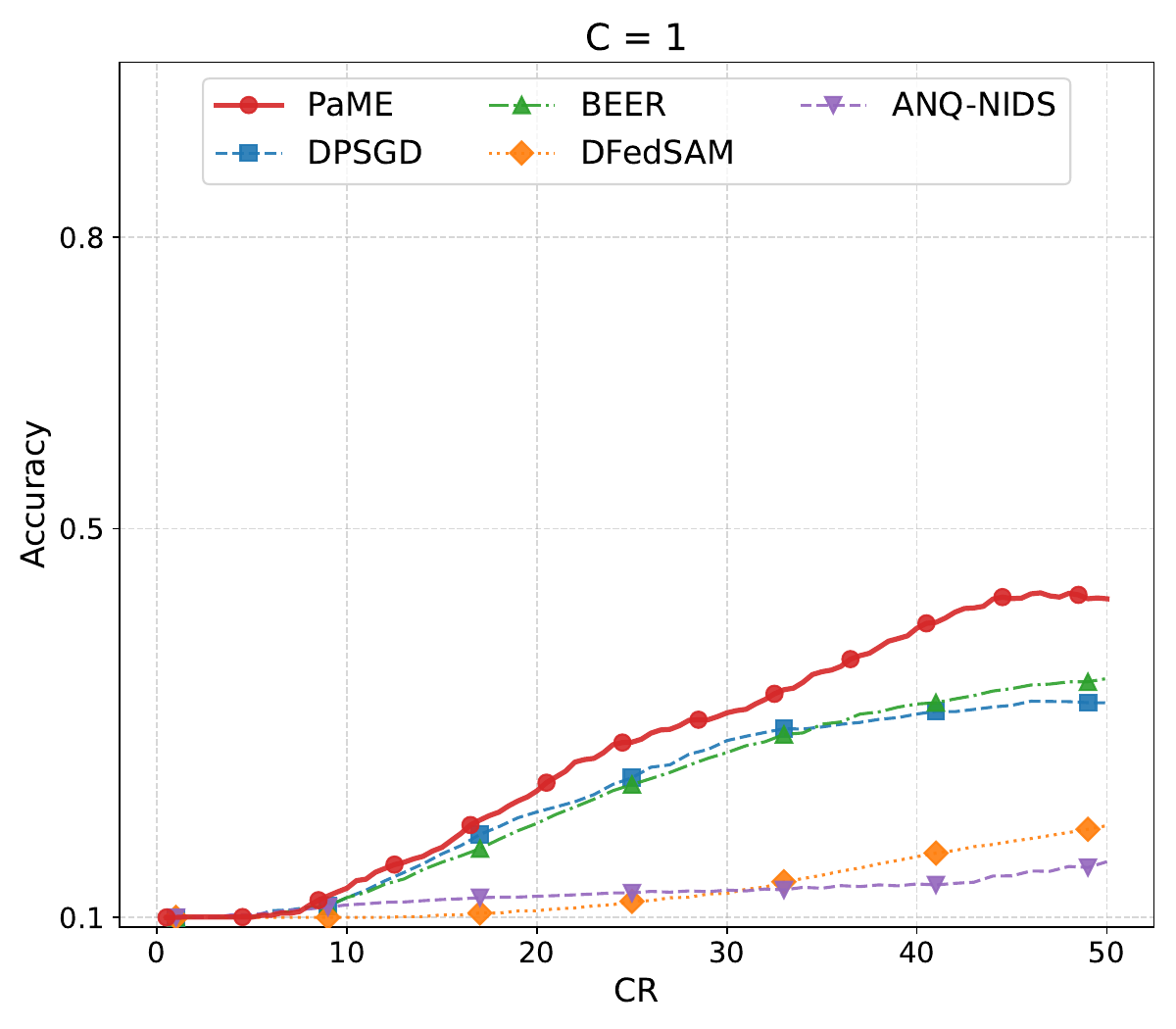}
		\caption{Accuracy v.s. CR for different algorithms solving Example \ref{eg5.3} with the Fashion-MNIST dataset.}\vspace{-2mm}
		\label{ex3-acc}
	\end{figure*} 
	
	\begin{figure*}[t]
		\centering
		\includegraphics[width=.325\textwidth]{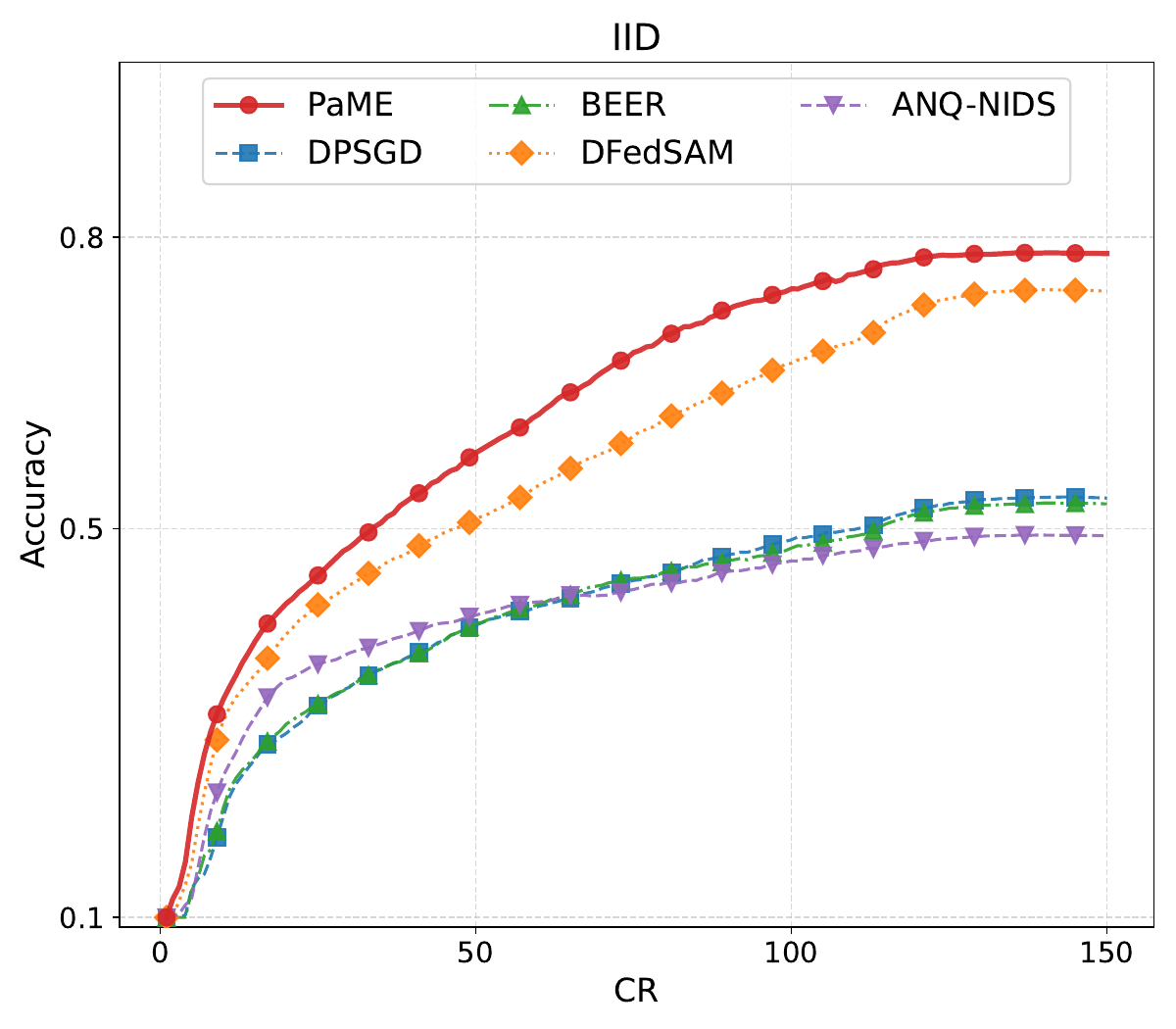}
		\includegraphics[width=.325\textwidth]{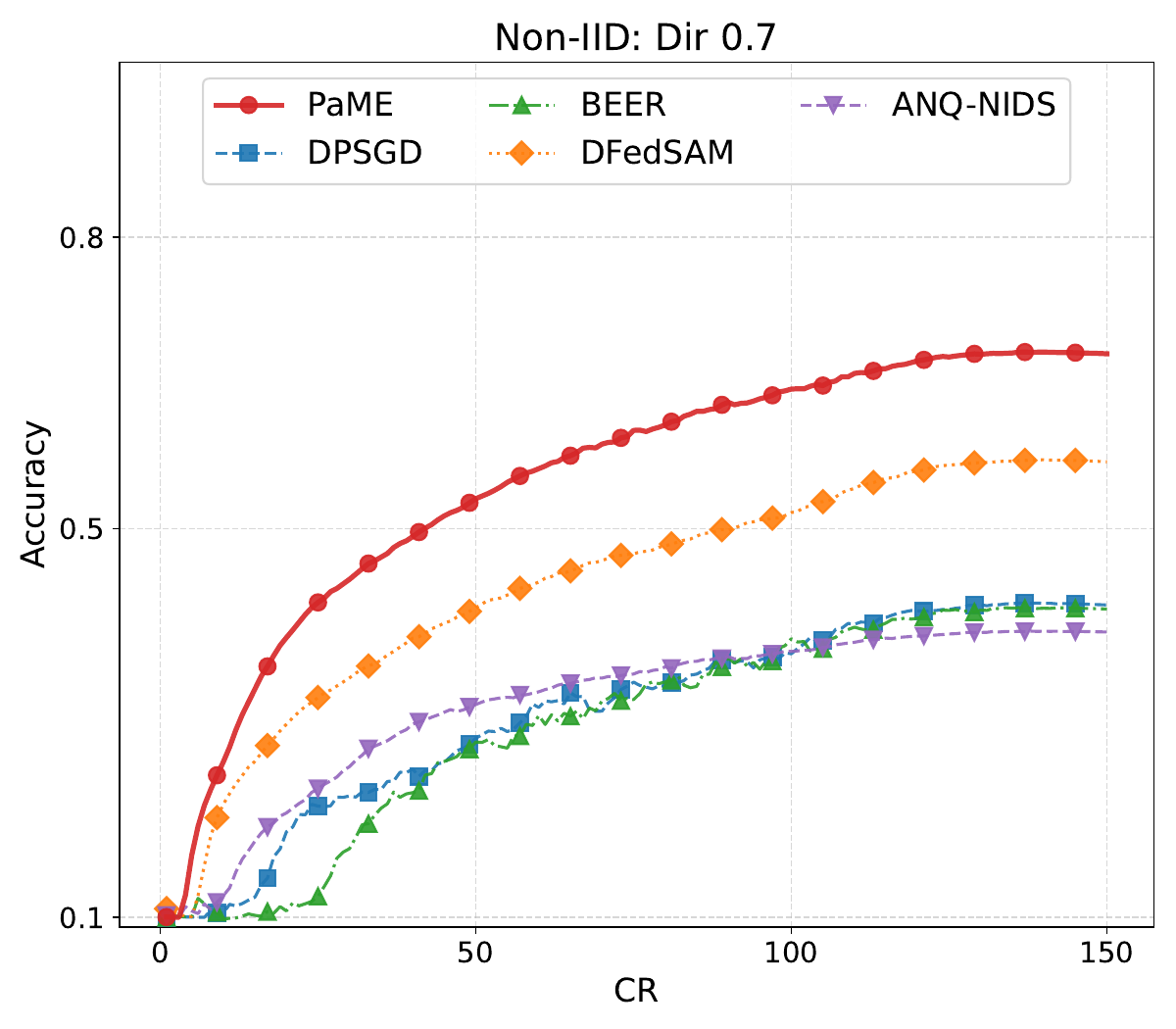}
		\includegraphics[width=.325\textwidth]{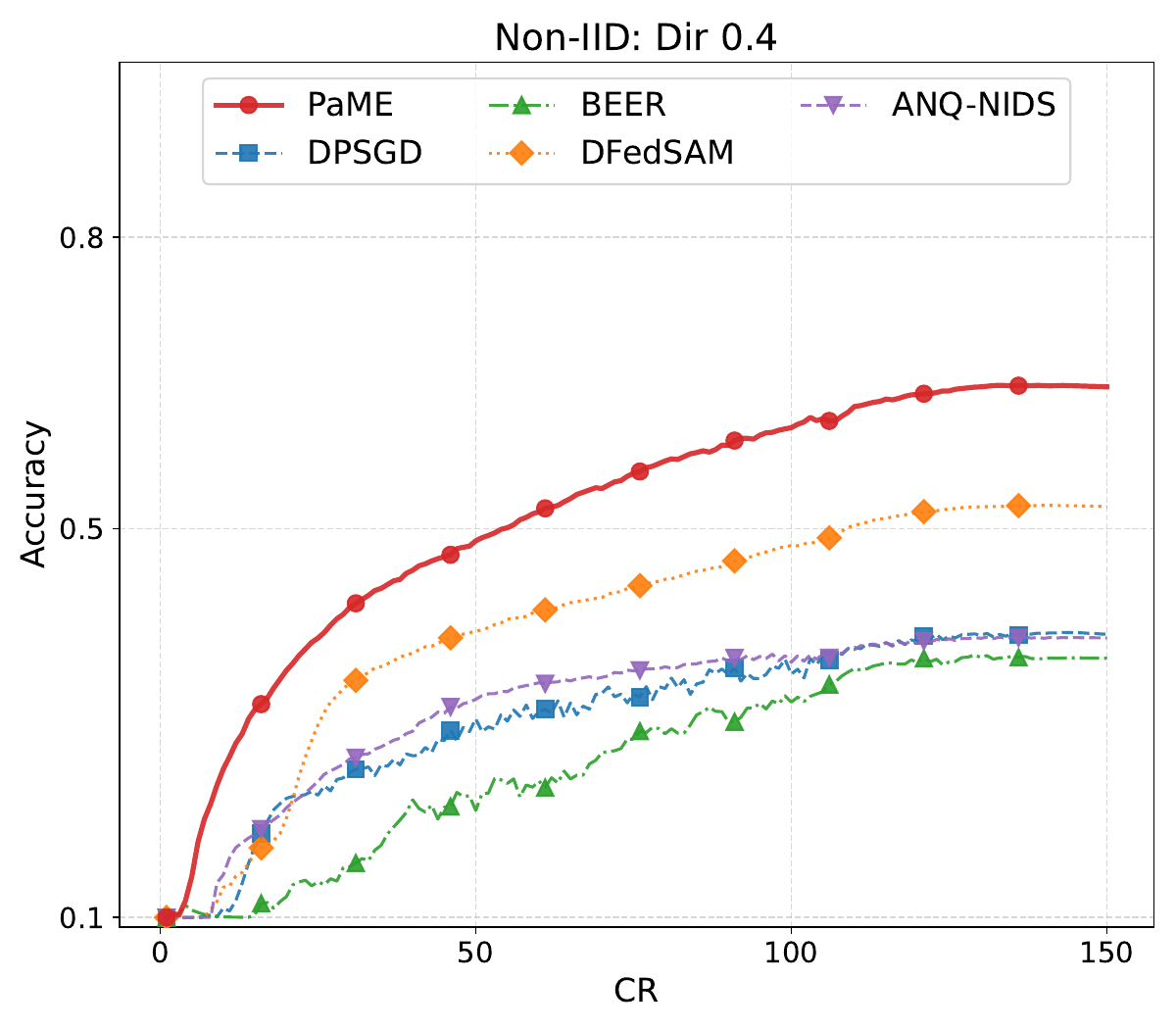}
		\caption{Accuracy v.s. CR for different algorithms solving Example \ref{eg5.4} with the CIFAR-10 dataset.}\vspace{-2mm}
		\label{ex4-acc}
	\end{figure*}

	{{\begin{figure*}[t]
		\centering
		\includegraphics[width=.325\textwidth]{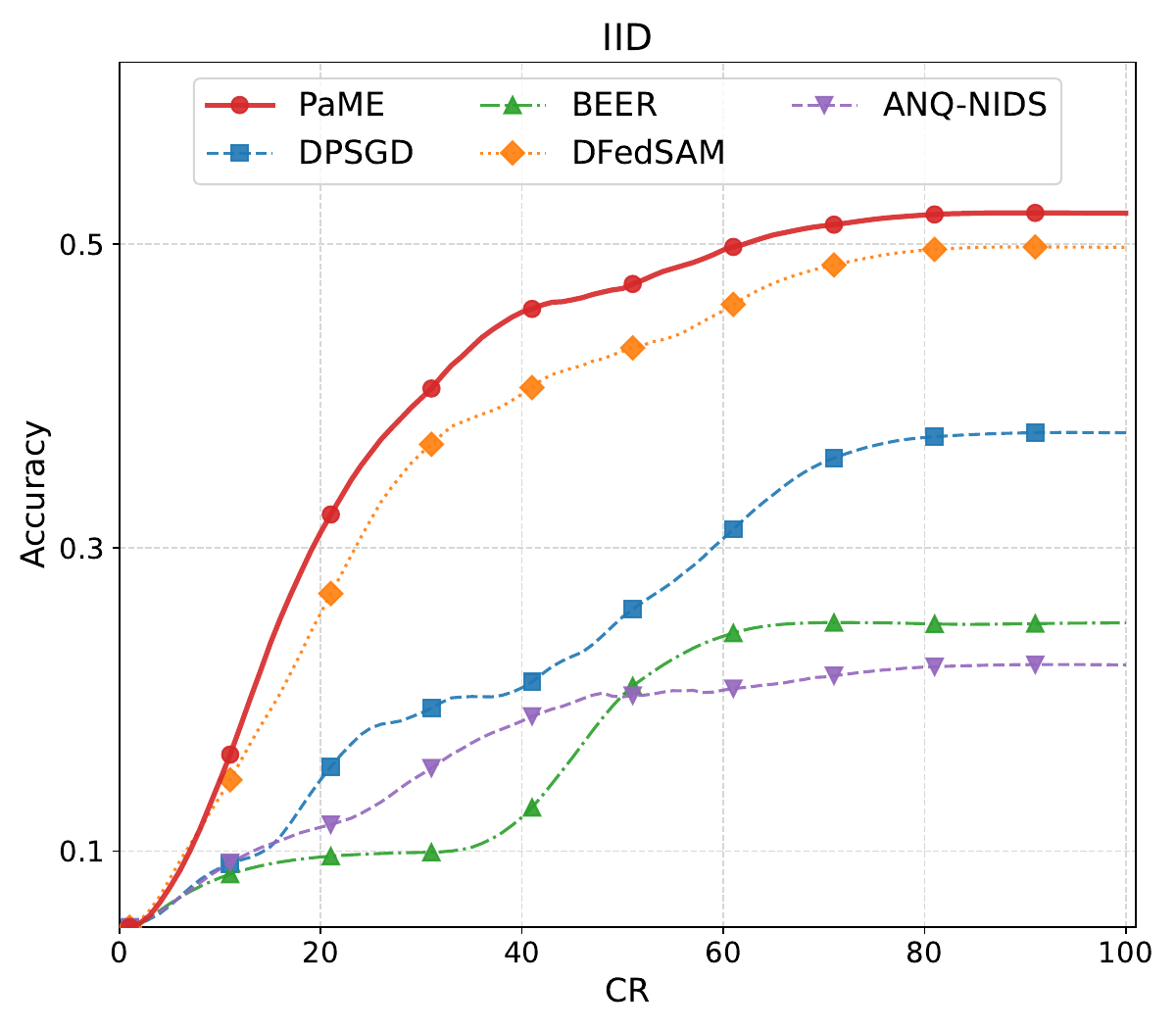}
		\includegraphics[width=.325\textwidth]{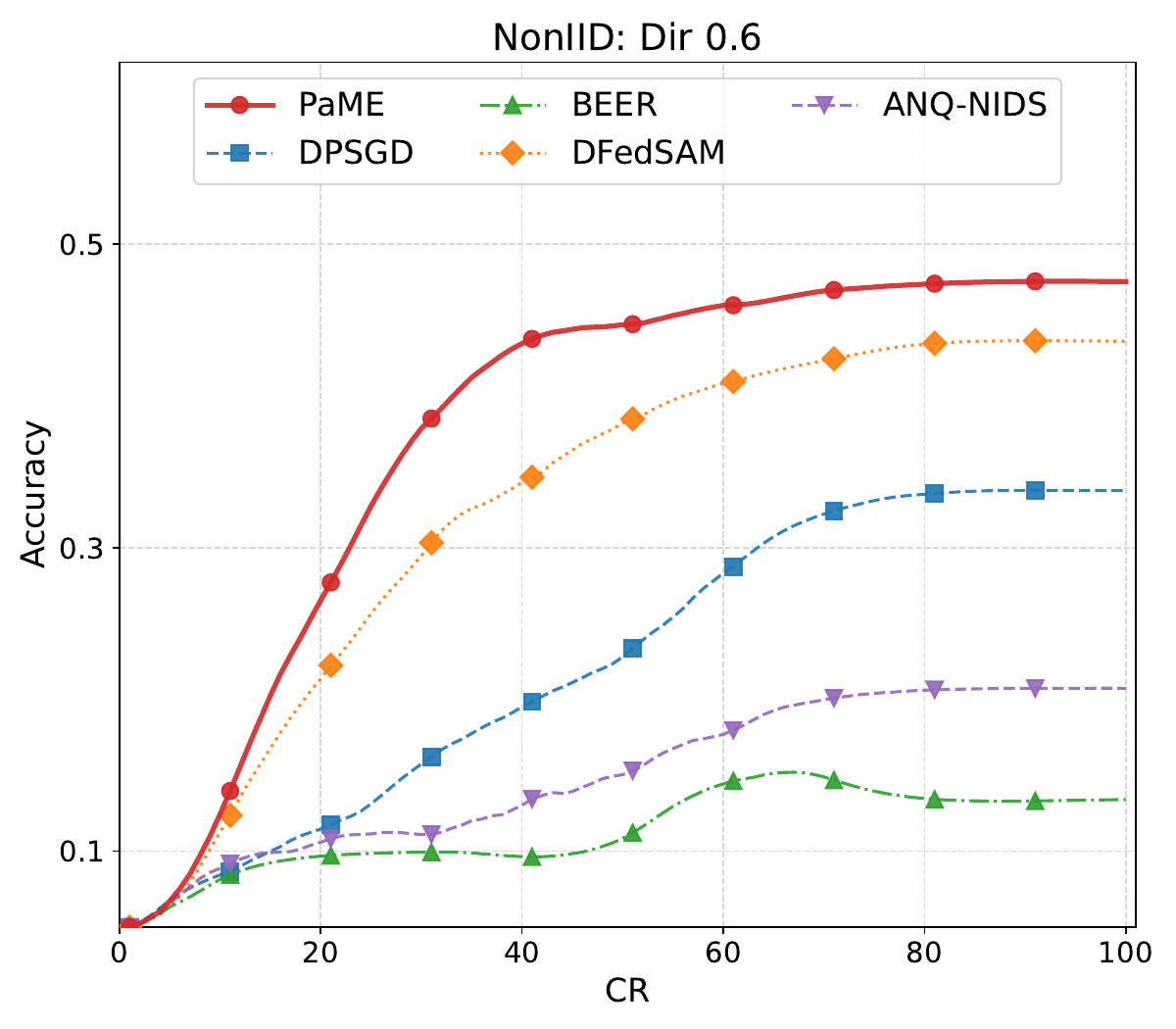}
		\includegraphics[width=.325\textwidth]{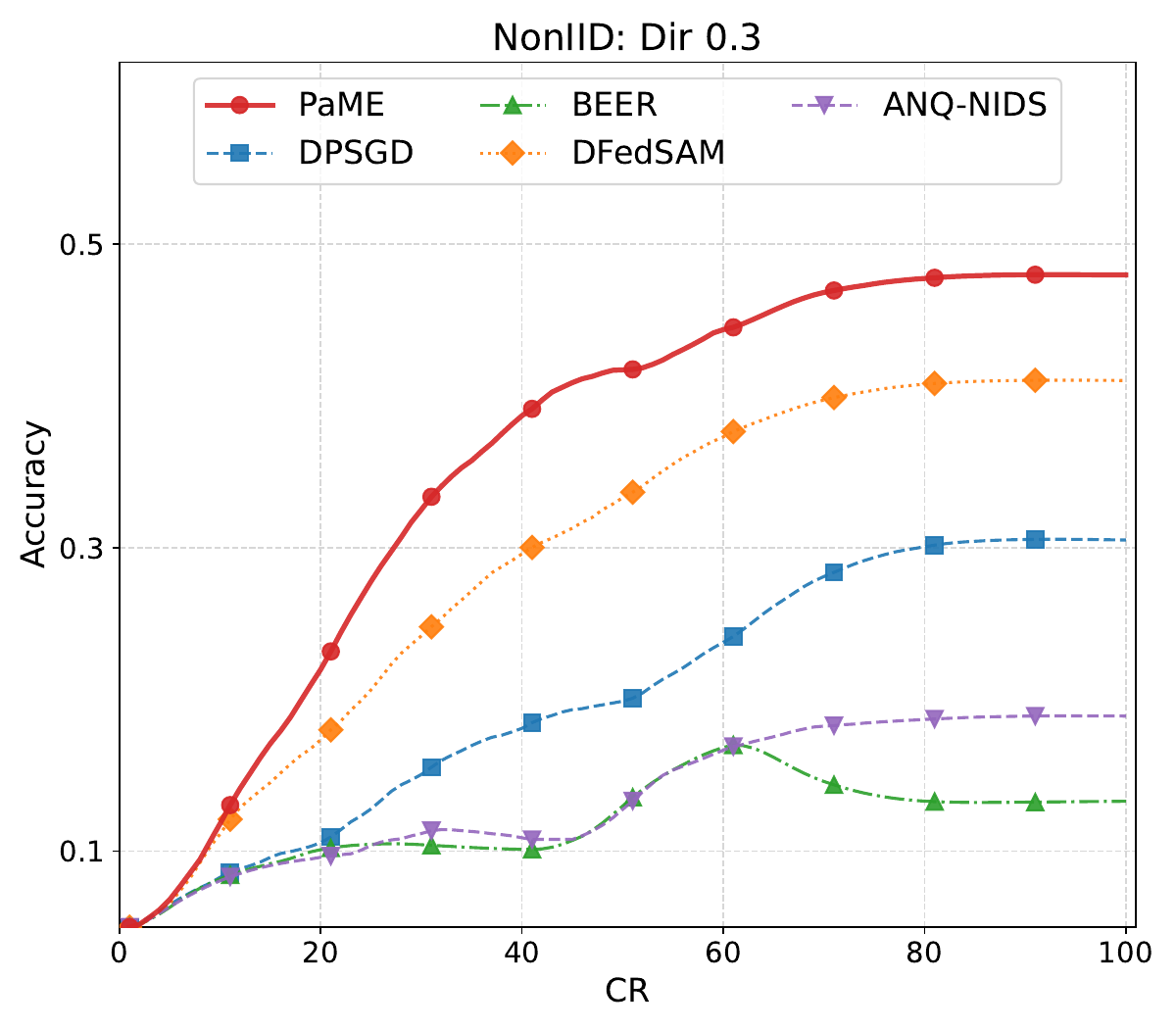}
		\caption{Accuracy v.s. CR for different algorithms solving Example \ref{eg5.4} with the Tiny-ImageNet dataset.}\vspace{-2mm}
		\label{ex5-acc}
	\end{figure*} }}
	\section{Conclusion}
	\label{Conclusion}
	This paper introduces PaME, a DFL algorithm to improve the trade-off among communication efficiency, privacy preservation, and model utility.   Its effectiveness stems from a novel partial message exchange mechanism, which is well suited to a variety of real-world scenarios, including unreliable wireless communications and edge computing environments. Moreover, rigorous theoretical guarantees are established under mild assumptions, which relax strict conditions significantly and thus enhance the robustness of DFL, particularly in the presence of heterogeneous data, highlighting  strong potential of PaME for practical applications.
	
	\bibliographystyle{IEEEtran}
	\bibliography{IEEEabrv,PaME-ref} 


 \vspace{-20pt}
 
 \begin{IEEEbiography}[{\includegraphics[width=1in,height=1.25in,clip,keepaspectratio]{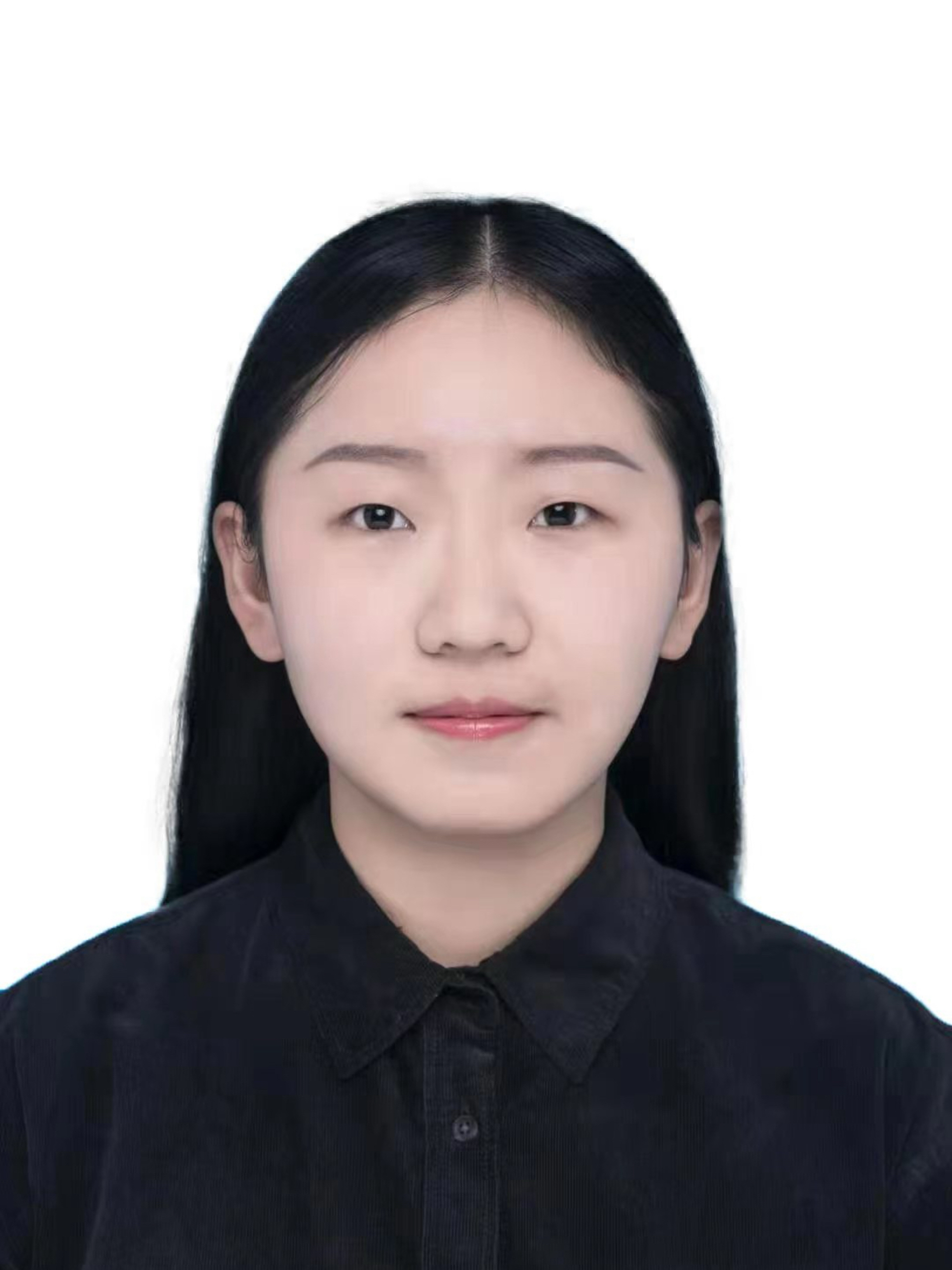}}]{Shan Sha}
received the B.S. degree from the School of Mathematics and Statistics, Beijing Jiaotong University, Beijing, China, in 2019, where she is currently working toward the PhD degree. From 2023 to 2024, she was a visiting Ph.D. student with the Intelligent Transmission and Processing Lab in Imperial College London. Her research interests include optimization theory and algorithms for federated learning.
\end{IEEEbiography}
 \vspace{-15pt}

\begin{IEEEbiography}  [{\includegraphics[width=1in,height=1.25in,clip,keepaspectratio]{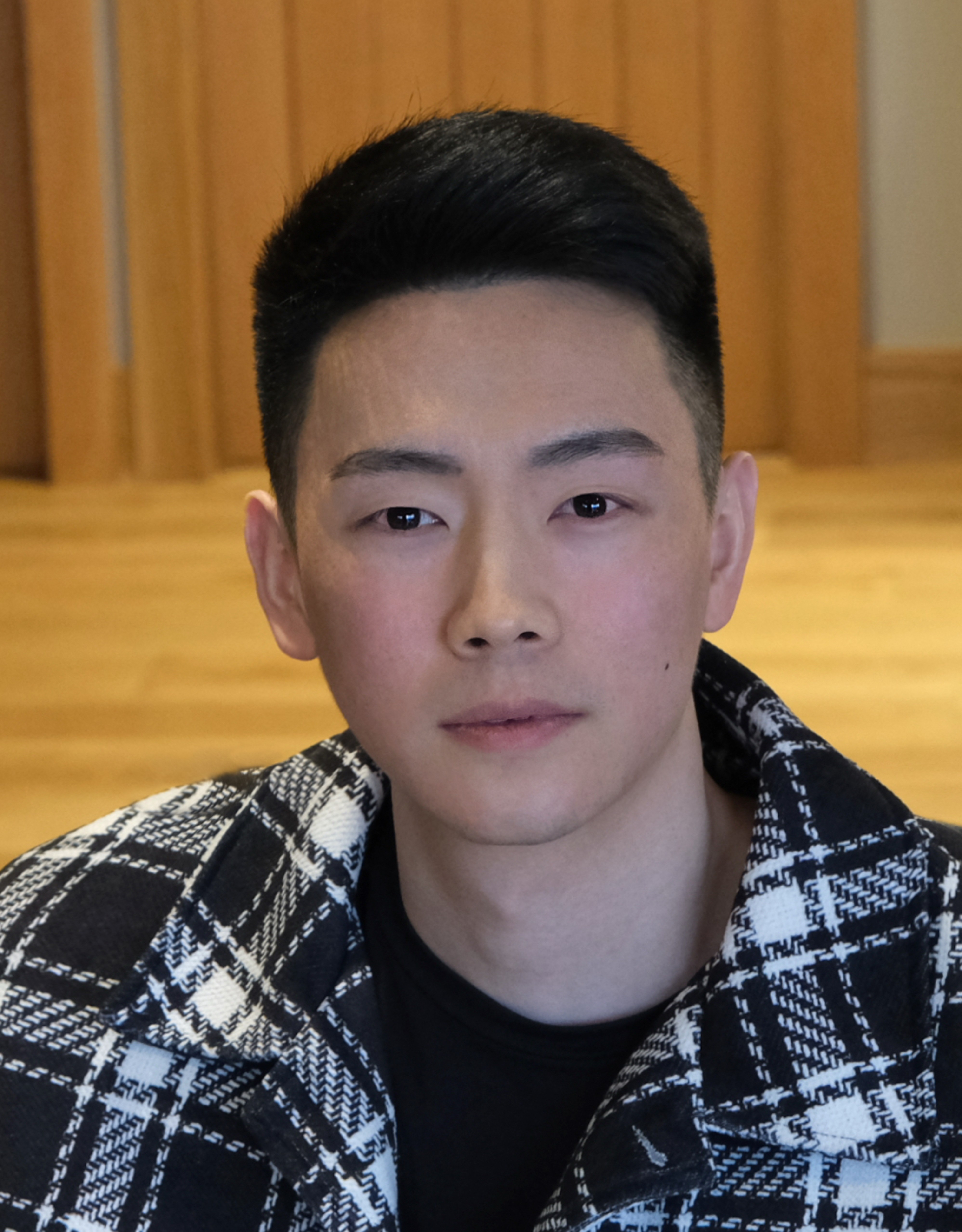}}]{Shenglong Zhou} received the Ph.D. degree from the University of Southampton, Southampton, U.K., in 2018, where he was a Research Fellow and a Teaching Fellow. From 2021 to 2023, he was a Research Fellow with Imperial College London, London, U.K. He is currently a Professor with Beijing Jiaotong University, Beijing, China.  His research interests include the theory and methods for optimization in the areas of sparse, low-rank matrix, 0/1 loss, and machine learning-related optimization.
\end{IEEEbiography}

 \vspace{-15pt}
\begin{IEEEbiography}[{\includegraphics[width=1in,height=1.25in,clip,keepaspectratio]{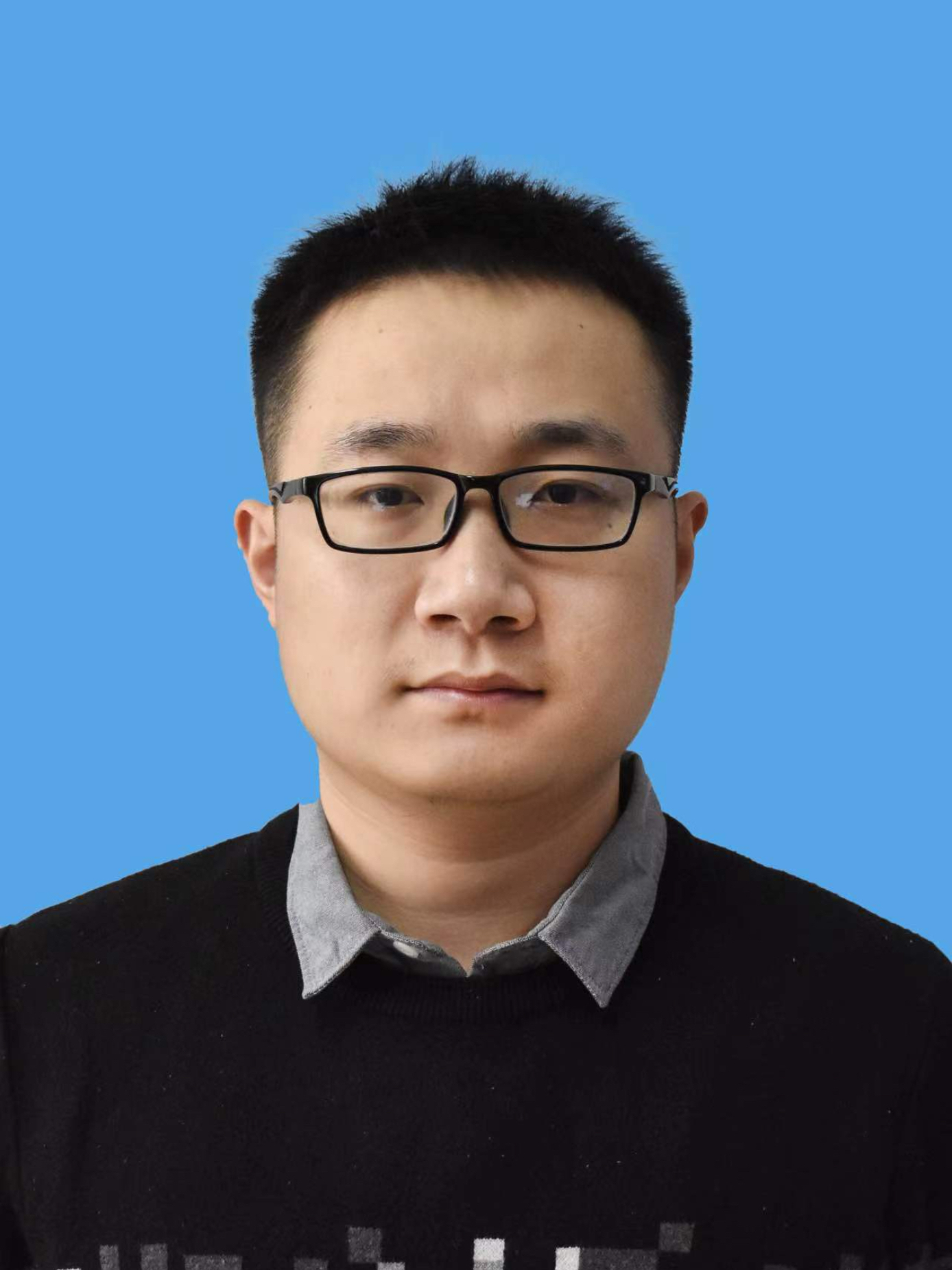}}]{Xin Wang}
 is a Lecturer at Beijing Jiaotong University, Beijing, China. He received his Ph.D. degree from the School of Mathematics and Statistics, Beijing Jiaotong University. His research interests include high-dimensional statistical analysis, optimization theory and algorithms.
\end{IEEEbiography}
 \vspace{-15pt}

\begin{IEEEbiography}[{\includegraphics[width=1in,height=1.25in,clip,keepaspectratio]{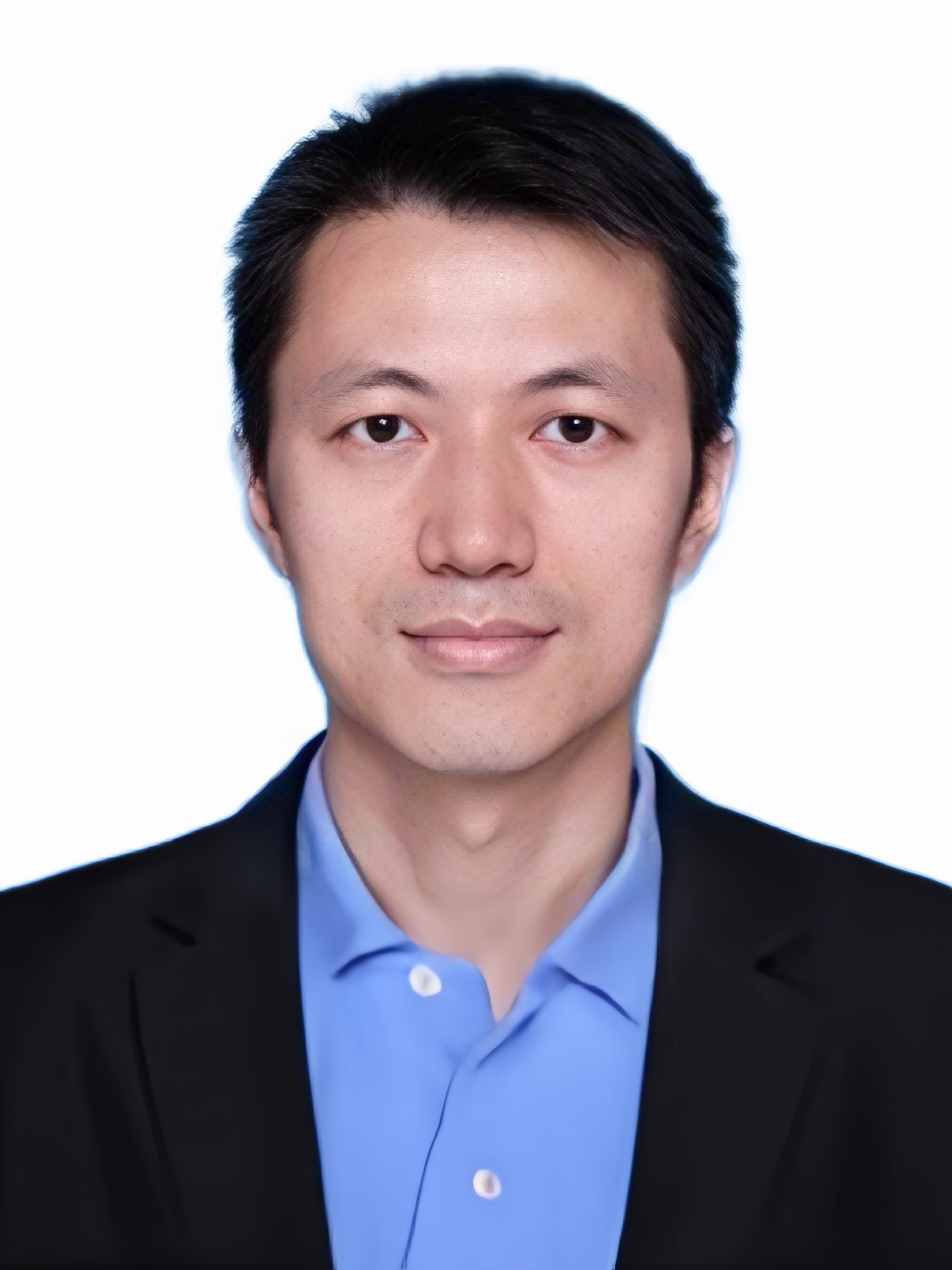}}]{Lingchen Kong}
received the Ph.D. degree from the School of Science, Beijing Jiaotong University, Beijing, China, in 2007. He is currently a professor of School of Mathematics and Statistics at Beijing Jiaotong University, Beijing, China. He was a Visiting Scholar with the Department of Statistics, University of Minnesota, Twin Cities, USA, from 2014 to 2015. His research interests include the large scale optimization, sparse optimization, linear regression, matrix regression, clustering, high-dimensional statistical analysis, etc.
\end{IEEEbiography}

\begin{IEEEbiography}  [{\includegraphics[width=1in,height=1.25in,clip,keepaspectratio]{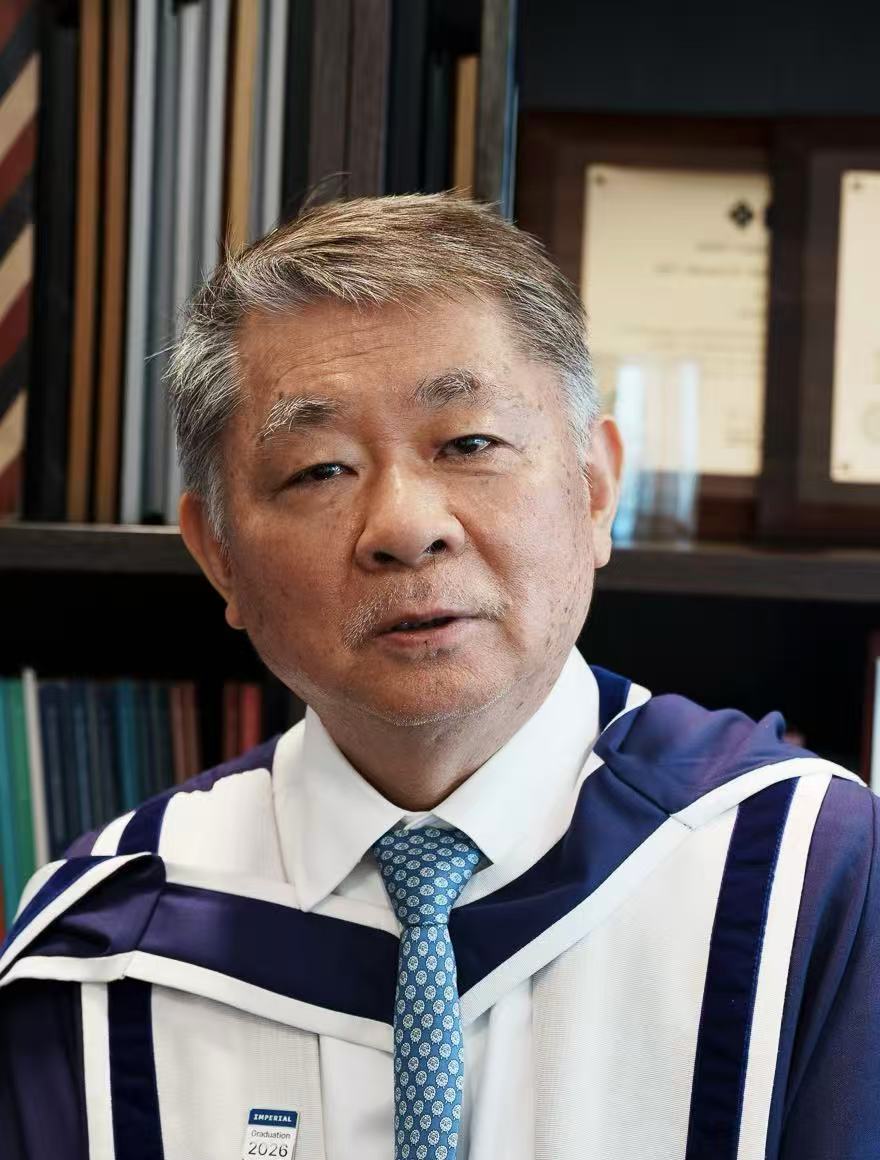}}]{Geoffrey Ye Li} is currently a Chair Professor at Imperial College London, UK. Before joining Imperial in 2020, he was a Professor at Georgia Institute of Technology for 20 years and a Principal Technical Staff Member with AT\&T Labs – Research (previous Bell Labs) for five years. He made fundamental contributions to orthogonal frequency division multiplexing (OFDM) for wireless communications, established a framework on resource cooperation in wireless networks, and introduced deep learning to communications. In these areas, he has published over 700 journal and conference papers in addition to over 40 granted patents. His publications have been cited over 80,000 times with an H-index over 130. He has been listed as a Highly Cited Researcher by Clarivate/Web of Science almost every year.

Dr. Geoffrey Ye Li was elected to Fellow of the Royal Academic of Engineering (FREng), IEEE Fellow, and IET Fellow for his contributions to signal processing for wireless communications. He received 2024 IEEE Eric E. Sumner Award, 2019 IEEE ComSoc Edwin Howard Armstrong Achievement Award, and several other awards from IEEE Signal Processing, Vehicular Technology, and Communications Societies.
\end{IEEEbiography}

\newpage

\setcounter{equation}{0}
\setcounter{theorem}{0}
\setcounter{section}{0}
\setcounter{algorithm}{0}
	\onecolumn
	\begin{center}
		
	\LARGE{Supplemental Material for\\
	``Decentralized Federated Learning by Partial Message Exchange"}	
\end{center}

		\section{\bf Proofs of Theorems in Section II}
		
	\subsection{\bf Proof of Theorem \ref{theorem1-unbiase}}
\begin{theorem}
Let $\overline{\w}$ be the average of $q$ vectors $\w_1$, $\w_2$, $\cdots$, ${\w_q \in \mathbb{R}^n}$. For each ${i\in[q]}$, independently construct a sparse vector ${\v_i\in\mathbb{R}^n}$  by uniformly selecting $s$ coordinates of $\w_i$ without replacement from $[n]$. Define indicator variables by
\begin{equation*}
\delta_{i\ell}=\begin{cases}
1,&\text{if $w_{i\ell}$ is selected},\\
0,& \text{otherwise},
\end{cases}\qquad\delta_\ell:=\sum_{i=1}^q\delta_{i\ell}.
\end{equation*}
and two averages $\overline{\v}$ and $\widetilde{\v}$ by
\begin{equation*}
\overline{v}_\ell=\begin{cases}
\dfrac{1}{\delta_\ell}\displaystyle{\sum_{i=1}^q} v_{i\ell},&\text{if}~\delta_\ell>0,\\
0,& \text{otherwise},
\end{cases}\qquad\widetilde{v}_\ell=\frac{1}{q}\sum_{i=1}^q v_{i\ell}
\end{equation*}
for any $\ell\in[n]$. 
Then
\begin{equation}\label{Ev-Ew}
\mathbb{E}\left(\overline{v}_\ell~|~\delta_\ell>0\right) = \overline{w}_\ell,\qquad
\mathbb{E}\left(\widetilde{v}_\ell~|~\delta_\ell>0\right)=\frac{\frac{s}{n}}{1-(1-\frac{s}{n})^q} \overline{w}_\ell.
\end{equation}
\end{theorem}
\begin{proof}
The definition of $\delta_{i\ell}$ and construction of $\v_i$ implies that when $\delta_\ell>0$ it follows
\begin{equation*}
\overline{v}_\ell
= \dfrac{1}{\delta_\ell}\displaystyle{\sum_{i=1}^q} v_{i\ell} = \frac{1}{\delta_\ell}\sum_{i=1}^q \delta_{i\ell} w_{i\ell}.
\end{equation*}
Given ${\delta_\ell=k>0}$, exactly $k$ indices among $[q]$ satisfy ${\delta_{i\ell}=1}$. Since the selections for different vectors are independent and symmetric, conditioning on $\delta_\ell=k$ implies that these $k$ indices form a uniformly random subset of $[q]$ of size $k$. Therefore,
\begin{equation*}
\mathbb{E}\!\left(\overline{v}_\ell \mid \delta_\ell=k\right)
= \frac{1}{k}\, \mathbb{E}\sum_{i:\,\delta_{i\ell}=1} w_{i\ell}
= \frac{1}{k} \frac{k}{q}\sum_{i=1}^q w_{i\ell}
= \overline{w}_\ell.
\end{equation*}
Since the above expectation does not depend on $k$, taking expectation over $\delta_\ell>0$ yields the first equation in \eqref{Ev-Ew}.  

On the other hand, by the definition of $\delta_{i\ell}$ and the construction of $\v_i$, we have
\[
\tilde{v}_\ell=\dfrac{1}{q}\dsum_{i=1}^q v_{i\ell}
=\dfrac{1}{q}\dsum_{i=1}^q \delta_{i\ell}w_{i\ell}.
\]
For any $i\in[q]$, since $s$ coordinates are selected uniformly without replacement from $[n]$ for each vector, we have
\[
\mathbb{P}(\delta_{i\ell}=1)=\dfrac{s}{n}.
\]
Moreover, since
\[
\delta_\ell=\dsum_{i=1}^q \delta_{i\ell},
\]
the event ${\delta_\ell>0}$ means that there exists at least one $i\in[q]$ such that $\delta_{i\ell}=1$. Therefore,
\[
\mathbb{P}(\delta_\ell>0)=1-\left(1-\dfrac{s}{n}\right)^q.
\]
Noting that the event ${\delta_{i\ell}=1}$ implies ${\delta_\ell>0}$, we obtain
\[
\mathbb{P}(\delta_{i\ell}=1\mid \delta_\ell>0)
=
\dfrac{\mathbb{P}(\delta_{i\ell}=1)}{\mathbb{P}(\delta_\ell>0)}
=
\dfrac{\frac{s}{n}}{1-\left(1-\frac{s}{n}\right)^q}=:c.
\]
This leads to
\[
\E(\delta_{i\ell}w_{i\ell}\mid \delta_\ell>0)
=
w_{i\ell}
\mathbb{P}(\delta_{i\ell}=1\mid \delta_\ell>0)
= c w_{i\ell},
\]
which further results in
\[
\E(\tilde v_\ell\mid \delta_t>0)
=
\dfrac{1}{q}\dsum_{i=1}^q \E(\delta_{i\ell}w_{i\ell}\mid \delta_\ell>0)
=
\dfrac{1}{q}\dsum_{i=1}^q
c w_{i\ell}
=
c\bar w_\ell.
\]
Hence, the second equality in \eqref{Ev-Ew} holds. This completes the proof.
\end{proof}

\subsection{\bf Proof of Theorem \ref{the:pame-reconstruction-risk}}

We first specify the threat model by identifying the information observable to a passive honest-but-curious adversary and isolating the data-dependent component
	in the communicated PaME messages. The adversary follows the PaME protocol but eavesdrops on the messages transmitted by node \(i\). In particular, at iteration \(k\), the adversary can observe the transmitted coordinate set
	 and the corresponding partial parameter vector. The adversary also knows public algorithmic information, such as the sampling rule, the coordinate-selection operator, and the protocol parameters, but does not have access to the private mini-batch \(\mathcal B_i^k\), the untransmitted coordinates, or the full stochastic gradient. This represents a relatively strong passive attack model, since the adversary is granted all transmitted coordinates and all public protocol information; hence the resulting reconstruction-risk bound is conservative for weaker eavesdropping adversaries.
	
    For node \(i\) at iteration \(k\), define the data-dependent local update increment by
	\begin{equation}\label{def-data-update-pame}
		\u_i^k(\mathcal B_i^k)
		:=
		\w_i^{k+1}-\overline{\v}_i^k
		=
		-\frac{1}{\sigma_i^k m_i^k}
		\nabla f_i(\overline{\v}_i^k;\mathcal B_i^k),
	\end{equation}
	where \(m_i^k=|N_i^k|\). Let \(\Omega_i^k\subseteq[n]\) denote the coordinate set of node \(i\) that is
	observable to the adversary, with \(|\Omega_i^k|=s_i^k\), and let
	\(\P_{\Omega_i^k}\in\mathbb R^{n\times n}\) be the corresponding
	coordinate-selection matrix, whose diagonal contains \(s\) ones on the
	selected coordinates and zeros elsewhere. If the
	adversary observes partial parameters \(\P_{\Omega_i^k}\w_i^{k+1}\) and knows
	\(\P_{\Omega_i^k}\overline{\v}_i^k\), then the data-dependent observable part is
	\(\P_{\Omega_i^k}\u_i^k\). If \(\overline{\v}_i^k\) is not known, the additional
	uncertainty can only make the reconstruction problem harder and is absorbed
	into the observation error below.
	
	For a finite observation window \(\{k,\ldots,k+T-1\}\), define
	\begin{equation}\label{def-window-data-update}
		\mathcal S_{i,T}
		:=
		(\mathcal B_i^k,\ldots,\mathcal B_i^{k+T-1}),
		\qquad
		\mathcal F_{i,T}(\mathcal S_{i,T})
		:=
		\big(\u_i^k(\mathcal B_i^k),\ldots,
		\u_i^{k+T-1}(\mathcal B_i^{k+T-1})\big).
	\end{equation}
	Thus, reconstructing the local batch window can be viewed as an inverse problem:
	the adversary attempts to infer the unknown data window \(\mathcal S_{i,T}\)
	from the partially observed data-dependent updates. Let
	\begin{equation}\label{def-observation-operator}
		\D_i^t
		:=
		\P_{\Omega_i^t}
	\end{equation}
	for direct partial observation, and define
	\begin{equation}\label{def-window-observation-operator}
		\D_{i,T}:=\operatorname{blkdiag}(\D_i^k,\ldots,\D_i^{k+T-1}).
	\end{equation}
	Then the adversary's observation over the finite window can be written as
	\begin{equation}\label{def-adversarial-observation}
		y_{i,T}
		=
		\D_{i,T}\mathcal F_{i,T}(\mathcal S_{i,T})
		+
		\boldsymbol{\xi}_{i,T},
	\end{equation}
	where \(\boldsymbol{\xi}_{i,T}\) denotes the aggregate observation uncertainty
	in the adversarial observation model. It collects the discrepancy between the
	idealized data-dependent update observation and the actual PaME observation,
	including stochastic mini-batch effects, local-update approximation error,
	neighbor-mixing deviation, nonlinear linearization residuals, and possible
	communication perturbations. We use \(\sigma_{\rm obs}\) to denote its effective
	scale. In this formulation,
	\(\D_{i,T}\mathcal F_{i,T}\) is the forward observation map available to the
	adversary, while reconstructing \(\mathcal S_{i,T}\) from \(y_{i,T}\) is the
	corresponding inverse problem.
	
	Consider a true local batch window \(\mathcal S_{i,T}^0\). After vectorizing the
	samples in \(\mathcal S_{i,T}\), suppose that \(\mathcal F_{i,T}\) is
	differentiable in a neighborhood of \(\mathcal S_{i,T}^0\) and admits the local
	expansion
	\begin{equation}\label{eq-local-linearization-pame}
		\mathcal F_{i,T}(\mathcal S_{i,T}^0+\mathbf h)
		=
		\mathcal F_{i,T}(\mathcal S_{i,T}^0)
		+
		J_{i,T}\mathbf h
		+
		\mathbf r_{i,T}(\mathbf h),
	\end{equation}
	where
	\begin{equation}\label{def-data-sensitive-jacobian}
		J_{i,T}
		:=
		\nabla_{\mathcal S}\mathcal F_{i,T}(\mathcal S)
		\big|_{\mathcal S=\mathcal S_{i,T}^0}
	\end{equation}
	is the data-sensitive Jacobian.

\begin{theorem}[Finite-window reconstruction risk under partial observation]
	\label{the:pame-reconstruction-risk}
	Suppose that \(\mathcal F_{i,T}\) is locally differentiable around the true
	batch window \(\mathcal S_{i,T}^0\). Let
	\[
	J_{i,T}
	:=
	\nabla_{\mathcal S}\mathcal F_{i,T}(\mathcal S)
	\big|_{\mathcal S=\mathcal S_{i,T}^0}
	\]
	be the data-sensitive Jacobian. Define the effective observable information
	ratio as
	\begin{equation}
		\label{eq:pame-effective-ratio}
		\rho_{i,T}
		:=
		\frac{\|\D_{i,T}J_{i,T}\|_F^2}{\|J_{i,T}\|_F^2}.
	\end{equation}
	Let \(d_{\mathcal B}\) be the dimension of the vectorized local batch window,
	\(M_{i,T}\) be the complete-observation information size, and
	\(\sigma_{\rm obs}\) be the scale of the observation uncertainty. 
	Following the inverse-problem formulation of data reconstruction attacks and
	the reconstruction-risk scaling established for noisy-gradient reconstruction
	in prior work~\cite{liu2025data}, we adopt the following complete-observation
	scaling as the baseline for the corresponding local inverse problem:
	\begin{equation}
	\label{eq-complete-observation-risk}
	R_L^{\rm full}
	\gtrsim
	\sigma_{\rm obs}
	\sqrt{\frac{d_{\mathcal B}}{M_{i,T}}},
	\qquad
	R_U^{\rm full}
	\lesssim
	b
	\sqrt{\frac{d_{\mathcal B}}{M_{i,T}}}.
		\end{equation}
	Then
	\begin{equation}
		\label{range-rho-iT}
		0\le \rho_{i,T}\le 1.
	\end{equation}
	Moreover, the reconstruction risk under the PaME
	partial-observation model satisfies the lower bound
	\begin{equation}
		\label{eq:pame-partial-risk-lower}
		R_L^{\rm PaME}
		\gtrsim
		\sigma_{\rm obs}
		\sqrt{
			\frac{d_{\mathcal B}}{\rho_{i,T}M_{i,T}}
		}.
	\end{equation}
	In addition, a masked reconstruction attack using only the visible
	coordinates has reconstruction error bounded by
	\begin{equation}
		\label{eq:pame-partial-risk-upper}
		R_U^{\rm PaME}
		\lesssim
		b
		\sqrt{
			\frac{d_{\mathcal B}}{\rho_{i,T}M_{i,T}}
		}
		+
		\sigma_{\rm obs},
	\end{equation}
	where \(b\) is the mini-batch size.
\end{theorem}

\begin{proof} Based on the definition of \(\P_{\Omega_i^k}\in\mathbb R^{n\times n}\) and $ \D_i^t:=\P_{\Omega_i^t}$, it follows from \eqref{def-window-observation-operator} that  \(\D_{i,T}\) is a coordinate-selection operator and thus it is non-expansive, i.e., \(\|\D_{i,T}\|\le 1\),
	therefore
	\begin{equation}\label{eq-rho-upper-bound}
		\rho_{i,T}
		=
		\frac{\|\D_{i,T}J_{i,T}\|_F^2}{\|J_{i,T}\|_F^2}
		\le
		\frac{\|\D_{i,T}\|^2\|J_{i,T}\|_F^2}{\|J_{i,T}\|_F^2}
		\le
		1.
	\end{equation}
	This quantity measures the fraction of data-sensitive Jacobian energy that
	remains visible after partial message exchange. Thus, PaME reduces the
	first-order observable information available to the adversary from
	\(J_{i,T}\) to \(\D_{i,T}J_{i,T}\).

	Based on the above threat characterization, we next quantify how partial
	message exchange reduces the effective information available to the adversary.
	The complete-observation case corresponds to \(\D_{i,T}=I\), in which the
	adversary observes the full data-dependent update map
	\(\mathcal F_{i,T}(\mathcal S_{i,T})\). Following the inverse-problem
	formulation of data reconstruction attacks and the reconstruction-error
	scalings established in~\cite{liu2025data}, we use the complete-observation
	risk scaling stated in \eqref{eq-complete-observation-risk} as the baseline
	for the corresponding local inverse problem. Here \(R_L^{\rm full}\) denotes
	the information-theoretic reconstruction-risk lower bound, while
	\(R_U^{\rm full}\) denotes the reconstruction error achieved by an explicit
	attack under complete observation.
	
	By \eqref{def-data-update-pame}, conditioned on the current public and
	historical information, the mini-batch \(\mathcal B_i^k\) affects the
	communicated model through the local update increment
	\(\u_i^k(\mathcal B_i^k)\). Hence, after subtracting the data-independent part
	\(\overline{\v}_i^k\) whenever it is available, the adversary only observes a
	masked version of the data-dependent update. Stacking these observations over
	the window \(\{k,\ldots,k+T-1\}\) gives the partial-observation model
	\eqref{def-adversarial-observation}.
	
	Around the true batch window \(\mathcal S_{i,T}^0\), the local expansion
	\eqref{eq-local-linearization-pame} yields
	\begin{equation*}
		\D_{i,T}\mathcal F_{i,T}(\mathcal S_{i,T}^0+\mathbf h)
		=
		\D_{i,T}\mathcal F_{i,T}(\mathcal S_{i,T}^0)
		+
		\D_{i,T}J_{i,T}\mathbf h
		+
		\D_{i,T}\mathbf r_{i,T}(\mathbf h).
	\end{equation*}

	Thus, in the local linearized inverse problem, PaME replaces the complete
	data-sensitive Jacobian \(J_{i,T}\) by the observable Jacobian
	\(\D_{i,T}J_{i,T}\). By the definition of the effective observable information
	ratio,
	\[
	\rho_{i,T}
	=
	\frac{\|\D_{i,T}J_{i,T}\|_F^2}{\|J_{i,T}\|_F^2},
	\]
	only a \(\rho_{i,T}\)-fraction of the data-sensitive Jacobian energy remains
	visible to the adversary. Therefore, under the same local inverse-problem
	scaling as in the complete-observation baseline, the effective information
	size is reduced from \(M_{i,T}\) to \(\rho_{i,T}M_{i,T}\). Substituting this
	reduced information size into the complete-observation lower-bound scaling
	gives
	\[
	R_L^{\rm PaME}
	\gtrsim
	\sigma_{\rm obs}
	\sqrt{\frac{d_{\mathcal B}}{\rho_{i,T}M_{i,T}}}.
	\]

	We next establish the upper bound. Consider the masked reconstruction attack
	based only on the visible coordinates:
	\begin{equation}
		\label{eq:privacy-proof-masked-attack}
		\widehat{\mathcal S}(y)
		\in
		\operatorname*{arg\,min}_{\mathcal S}
		\left\|
		\D_{i,T}\mathcal F_{i,T}(\mathcal S)
		-
		y
		\right\|^2
		+
		\mathcal R(\mathcal S),
	\end{equation}
	where \(\mathcal R(\mathcal S)\) denotes a regularization term encoding the
	prior information used by the reconstruction algorithm. Here we explicitly
	write \(\widehat{\mathcal S}(y)\) to emphasize that the reconstruction output
	depends on the observation \(y\).
	
	We first consider the ideal masked linearized inverse problem. Around the true
	batch window \(\mathcal S_{i,T}^0\), the local linearized masked observation
	map is
	\begin{equation}
		\label{eq:ideal-masked-linearized-map}
		\bar y_{i,T}(\mathbf h)
		:=
		\D_{i,T}\mathcal F_{i,T}(\mathcal S_{i,T}^0)
		+
		\D_{i,T}J_{i,T}\mathbf h .
	\end{equation}
	This is the same local inverse problem as in the complete-observation case,
	except that the complete data-sensitive Jacobian \(J_{i,T}\) is replaced by the
	observable Jacobian \(\D_{i,T}J_{i,T}\). Since
	\[
	\rho_{i,T}
	=
	\frac{\|\D_{i,T}J_{i,T}\|_F^2}{\|J_{i,T}\|_F^2},
	\]
	only a \(\rho_{i,T}\)-fraction of the data-sensitive Jacobian energy remains
	observable. Therefore, under the complete-observation reconstruction scaling
	with the effective information size \(M_{i,T}\) replaced by
	\(\rho_{i,T}M_{i,T}\), the ideal masked reconstruction error satisfies
	\begin{equation}
		\label{eq:ideal-masked-upper-bound}
		R_U^{\rm id}
		\lesssim
		b
		\sqrt{
			\frac{d_{\mathcal B}}{\rho_{i,T}M_{i,T}}
		}.
	\end{equation}
	
	We now compare the actual PaME observation with the ideal masked linearized
	observation. The actual PaME observation may differ from
	\(\bar y_{i,T}(\mathbf h)\) due to local-update approximation, neighbor mixing,
	nonlinear linearization residuals, stochastic observation effects, and possible
	communication perturbations. Instead of bounding these components separately,
	we collect them into the aggregate observation uncertainty
	\(\boldsymbol{\xi}_{i,T}\) and write
	\begin{equation}
		\label{eq:actual-pame-observation-upper}
		y_{i,T}^{\rm PaME}
		=
		\bar y_{i,T}(\mathbf h)
		+
		\boldsymbol{\xi}_{i,T}.
	\end{equation}
	The effective magnitude of this aggregate uncertainty is characterized by
	\(\sigma_{\rm obs}\), in the sense that
	\begin{equation}
		\label{eq:xi-effective-scale}
		\|\boldsymbol{\xi}_{i,T}\|_{\rm eff}
		\lesssim
		\sigma_{\rm obs}.
	\end{equation}
	
	For the upper bound, we consider a locally stable masked reconstruction
	procedure. Specifically, its solution map \(y\mapsto \widehat{\mathcal S}(y)\)
	satisfies, in the considered local neighborhood,
	\begin{equation}
		\label{eq:masked-reconstruction-stability}
		\big\|
		\widehat{\mathcal S}(y)
		-
		\widehat{\mathcal S}(y')
		\big\|
		\le
		\kappa
		\|y-y'\|,
	\end{equation}
	where \(\kappa\) is a local stability constant. Applying
	\eqref{eq:masked-reconstruction-stability} with
	\(y=y_{i,T}^{\rm PaME}\) and \(y'=\bar y_{i,T}(\mathbf h)\), and using
	\eqref{eq:actual-pame-observation-upper}, gives
	\begin{equation}
		\label{eq:observation-perturbation-transfer}
		\big\|
		\widehat{\mathcal S}(y_{i,T}^{\rm PaME})
		-
		\widehat{\mathcal S}(\bar y_{i,T}(\mathbf h))
		\big\|
		\le
		\kappa
		\|\boldsymbol{\xi}_{i,T}\|.
	\end{equation}
	Taking the same risk measure as in the definition of \(R_U\), and using
	\eqref{eq:xi-effective-scale}, we obtain
	\begin{equation}
		\label{eq:perturbation-risk-bound}
		\big\|
		\widehat{\mathcal S}(y_{i,T}^{\rm PaME})
		-
		\widehat{\mathcal S}(\bar y_{i,T}(\mathbf h))
		\big\|_{\rm risk}
		\lesssim
		\kappa\sigma_{\rm obs}.
	\end{equation}
	
	Finally, by the triangle inequality,
	\begin{align}
		R_U^{\rm PaME}
		&\le
		R_U^{\rm id}
		+
		\big\|
		\widehat{\mathcal S}(y_{i,T}^{\rm PaME})
		-
		\widehat{\mathcal S}(\bar y_{i,T}(\mathbf h))
		\big\|_{\rm risk} \lesssim
		b
		\sqrt{
			\frac{d_{\mathcal B}}{\rho_{i,T}M_{i,T}}
		}
		+
		\kappa\sigma_{\rm obs}.
		\label{eq:pame-partial-risk-upper-with-kappa}
	\end{align}
	Since \(\kappa\) is a local stability constant independent of the main
	information parameters \(d_{\mathcal B}\), \(M_{i,T}\), \(\rho_{i,T}\), and
	\(\sigma_{\rm obs}\), it can be absorbed into the hidden constant in
	\(\lesssim\). Therefore,
	\begin{equation}
		\label{eq:pame-partial-risk-upper-proof}
		R_U^{\rm PaME}
		\lesssim
		b
		\sqrt{
			\frac{d_{\mathcal B}}{\rho_{i,T}M_{i,T}}
		}
		+
		\sigma_{\rm obs}.
	\end{equation}
	This proves the desired algorithmic upper bound for a masked reconstruction attack.

\end{proof}

		\section{\bf Algorithm, Assumption, and Setup}
		
			\begin{algorithm}[!t] 
 \caption{DFL by Partial Message Exchange (PaME)  \label{algorithm-PaME}}
\begin{algorithmic}[1]

	\State	Initialize {$\mathbf{w}^0_i=0$}, two integers {$ \kappa_i>0$} and {$s_i>0$}, {$\sigma_i^0>0$}, {$\gamma_i>1$} for each {$i\in[m]$}.
		
		\For{iteration $k=0,1,2,3,\cdots $}			
		\For{node $i=1,2,\cdots,m $}
		\If{$k \in\mathcal{K}_i :=\{0,\kappa_i,2\kappa_i,3\kappa_i,\cdots \}$} 
		
			\State Random neighbor selection: ${N}_i^{k}\subseteq {N}_i$. 
					
\State Neighbor number update:	 ~$m_i^{k}=\left|{N}_i^k\right|$.

				\State Partial message exchange: ~$\vik=\texttt{PME}(\w_i^k,~\{\w_j^k: j\in{N}_i^{k}\})$.

							
		
		\Else

		
		\State Local parameter tracking:  ~~$\vik= \w_i^k$.
		
		\State Neighbor number tracking:	 $m_i^{k}=m_i^{k-1}$.							
		
		
    \EndIf
    
    \State Random sub-batch data sampling: {$\mathcal{B}_i^{k} \subseteq \mathcal{D}_i$}.
		
		\State Local parameter updating: 
		\begin{eqnarray}\label{SRDFL-localstep}
		\mathbf{w}_{i}^{k+1}=\vik - \frac{\nabla f_i(\vik;\mathcal{B}_i^{k})}{\sigma_i^k m_i^k}.
		\end{eqnarray} 
		
		\State Hyper-parameter increasing: $\sigma_i^{k+1}=\gamma_i \sigma_i^{k}$.	
     \EndFor
 \EndFor
\end{algorithmic}
				
	\end{algorithm}
		
	\begin{algorithm}[!t] 
 \caption{$\vik=\texttt{PME}(\w_i^k,~\{\w_j^k: j\in{N}_i^{k}\})$  \label{algorithm-PME}}
\begin{algorithmic}[1]
 \For{every node $j\in {N}_i^{k}$} 
\State Random subset selection: ${T_j^k\subseteq[n]}$ with ${|T_j^k|= s_j}$.   

 \State Sparse vector $\v_j^k$ generation:
 \begin{equation}\label{sparse-generator}
v_{j\ell}^k=\begin{cases}
 {w}_{j\ell}^k, &\text{if}~\ell\in T_j^k,\\
  0, &\text{if}~\ell\notin T_j^k.\\
 \end{cases}
 \end{equation}
 \State Partial message $\v_j^k$ exchange to node $i$.
 \EndFor
 
 \State  Node $i$ averages received messages by
  \begin{equation}\label{pame}
\overline{v}_{i\ell}^k=\begin{cases}
 \dfrac{\sum_{j\in {N}_i^{k}}v_{j\ell}^k}{\lambda_{i,\ell}^k}, &\text{if}~\lambda_{i,\ell}^k\neq0,\\[1ex]
  w_{i\ell}^k, &\text{if}~\lambda_{i,\ell}^k=0,\\
 \end{cases}
 \end{equation}
 for each $\ell\in[n]$, where 
   \begin{equation}\label{def-lambda-itk}
\lambda_{i,\ell}^k =\left|\left\{j\in {N}_i^{k}:~ v_{j\ell}^k \neq 0 \right\}\right|.
 \end{equation}
\end{algorithmic}
\end{algorithm}
	\begin{assumption} 
		For each ${i \in [m]}$,  $\nabla f_i$ is Lipschitz continuous with ${\alpha_i > 0}$ on $\mathbb{N}(2\delta)$ for a given $\delta\in(0,\infty)$, where 
		$$\mathbb{N}(2\delta):= \{\w \in \mathbb{R}^n : \|\w\|_\infty \le 2\delta\}$$
		and $\|\w\|_\infty$ denotes the infinity norm of $\w$.
		\label{assump-gradientlip}
	\end{assumption}
		\begin{assumption}  
		\label{assump-double-stochasitic}
		Suppose  network matrix $\B$ is a doubly stochastic matrix satisfying
	\begin{equation}\label{max-lambda}
	\zeta:=\max\{|\lambda_2(\B)|, |\lambda_m(\B)|\}<1.
\end{equation}
where $\lambda_i(\B)$	 is the $i$th largest eigenvalue of $\B$ and $\B$ is defined by 
	\begin{equation}
		B_{ji}:=
		\left\{
		\begin{array}{ll}
			\dfrac{1}{m_i}, & \text{if } j\in N_i,\\[1ex]
			0, & \text{otherwise},
		\end{array}
		\right.
	\end{equation}
	\end{assumption}
Note that doubly stochastic matrix satisfies 
\begin{equation}\label{JB-BJ-J}
\B  \mathbf{1}^\top =\mathbf{1}^\top \B = \mathbf{1}^\top,\qquad \mathbf{J}\B=\B \mathbf{J}=\mathbf{J},\qquad\text{with}~ \mathbf{J}:= \dfrac{\mathbf{1}\mathbf{1}^\top}{m} ,
\end{equation}
 where $\mathbf{1}$ be a vector with all entries being $1$. It is easy to see that $\B$ satisfies 
	\begin{equation}\label{bd-I-B}
	\|\B-\mathbf{J}\|_2^2 \le  \zeta^2,\qquad	\left\|\I-\B\right\|_{2}^2 
		\leq (\nl\I-\mathbf{J}\nr + \nl \B-\mathbf{J}\nr )^2\leq  (1+ \zeta  )^2. 
	\end{equation}
	
\begin{setup}\label{setup}
Parameters in Algorithm~\ref{algorithm-PaME} are chosen as follows.
	\begin{itemize}[leftmargin=14pt]
		\item[1)] Set ${\W^0 = \mathbf{0}}$ and ${\kappa_i = k_0}$ for all ${i\in [m]}$. This is adopted for analytical convenience without loss of generality. In fact, one can always let $k_0$ be the least common multiple of $\{\kappa_1,\kappa_2,\cdots,\kappa_m\}$, then the subsequent analysis remains similar to the case of ${\kappa_i = k_0}$.  
		\item[2)] Set ${m_i^k = t_i=\lfloor \nu_i |N_i|\rfloor }$ for all ${k \ge 0}$ and ${i \in [m]}$, where ${\nu_i \in (0, 1]}$ is the participation rate and $\lfloor a\rfloor$ is the floor of $a$. That is, at every iteration, node $i$ selects the same number of neighbor nodes to join in the training.  

		\item[3)] Set ${\sigma_i^0 = \sigma^0}$ and ${\gamma_i=\gamma}$, ${s_i=s}$ for all ${i\in [m]}$. Again, this is adopted for analytical convenience without loss of generality. In fact, for different ${\sigma_i^0}$ and ${\gamma_i}$,  we can conduct similar analysis by considering $$\sigma^0=\min\{\sigma_1^0,\cdots,\sigma_m^0\},\qquad \gamma=\min\{\gamma_1,\cdots,\gamma_m\},\qquad  s=\min\{s_1,\cdots,s_m\}.$$ 	Moreover, in Algorithm \ref{algorithm-PME}, each node ${j\in N_i^k}$ independently constructs a set ${T_j^k}$ by uniformly selecting $s$ entries of $[n]$ without replacement. 	
		\item[4)] In the sequel, given $\zeta\in(0,1)$ defined in (\ref{max-lambda}) and integers $k_0$ and $(t_1,\cdots,t_m)$,  initialize  
		\begin{equation}\label{initial-p-gamma-nu}
	p:=\frac{s}{n}\in(0,1),\qquad	\gamma\in\left(1, \zeta^{-2/k_0}\right),\qquad \nu_i \in (0, 1],~~\forall~i\in[m],
		\end{equation}	
		to satisfy
	\begin{equation}\label{cond-p-zeta-gamma}
		 (1-p)^{t_i} \left(1+ \zeta \right)^2+2p \sum_{j\in N_i}\nu_j
		<
		\left(\frac{1}{\gamma^{k_0/2} }- \zeta \right)^2,\qquad \forall~i\in[m].
	\end{equation}
	Moreover, define a useful constant by
				\begin{equation}\label{def-eps}
				\begin{aligned}			
					\varepsilon_i(t)&:=\sup \Big\{\left\|\nabla f_i\left(\w ; \mathcal{B}_i\right)-\nabla f_i\left(\w ; \mathcal{B}_i^{\prime}\right)\right\|_{\infty}:~ \mathcal{B}_i, \mathcal{B}_i^{\prime} \subseteq \mathcal{D}_i, \w \in \mathbb{N}(t)\Big\}, \quad \forall i \in[m].\\	
				\end{aligned}\end{equation}
			It is easy to see that such a constant $\varepsilon_i(t)$ is well defined for any given $t\in(0,\infty)$, namely,  
						\begin{equation*}\label{def-eps-bd}
					0<\varepsilon_i(t)<\infty \quad \text{for any given $t\in(0,\infty)$}.
				\end{equation*}
		Based on this constant, we always choose $ \sigma^0$ such that
		\begin{equation}\label{def-sigma}
					\sigma^0\geq\sigma := \max\left\{4\alpha_{\max},   \dfrac{\varepsilon(2\delta)\gamma}{(\gamma-1)\delta t_{\min} } \right\},
				\end{equation}	
		where   \begin{equation}\label{def-eps-t-r}
				\begin{aligned} 
				\varepsilon(2\delta) &:=\max\Big\{\varepsilon_1(2\delta),\varepsilon_2(2\delta),\cdots,\varepsilon_m(2\delta)\Big\},\\
				 	\alpha_{\max}&:=\max\Big\{\alpha_1,\alpha_2,\cdots,\alpha_m\Big\},\\
				 	 t_{\min}&:=\min\Big\{t_1,t_2,\cdots,t_m\Big\}.\end{aligned} 
				\end{equation}
	\end{itemize}  
	\end{setup}	
	
\noindent According to Setup \ref{setup},  ${\sigma_i^0 = \sigma^0}$ and ${\gamma_i=\gamma}$ for all ${i\in [m]}$ , one has 
						\begin{equation*} 
					 \sigma^{k+1}_i=\gamma_i\sigma_i^{k}= \gamma \sigma^{k}_i=\gamma^{k+1}\sigma_i^0=\gamma^{k+1}\sigma^0.
				\end{equation*}
	As a consequence, 
				\begin{equation}
					\label{SRDFL-CRstep}
		\sigma^{k}:=	\sigma^{k}_1=\cdots=\sigma_m^{k} =\gamma\sigma^{k}=\gamma^{k}\sigma^0,
				\end{equation}
				which results in
				\begin{equation}\label{stepbound}
					\sigma^{k+1} >  \sigma^0 \ge \sigma = \max\left\{4\alpha_{\max}, \dfrac{\varepsilon(2\delta)\gamma}{(\gamma-1)\delta t_{\min} }  \right\}.
				\end{equation}	
 Under Assumption \ref{assump-gradientlip} and \eqref{stepbound},    for $\forall i \in [m]$ and $\forall k \ge 0$,
			\begin{equation}\label{lip-inequal}
				\begin{aligned}
					f_i(\w) - f_i(\v) &\le \left\langle\nabla f_i(\u), \w-\v\right\rangle + \frac{\alpha_i}{2}\left\|\w-\v\right\|^2 \\
					&\overset{\eqref{stepbound}}{\le} \left\langle\nabla f_i(\u), \w-\v\right\rangle + \frac{\sigma^k}{8}\left\|\w-\v\right\|^2,
				\end{aligned}
			\end{equation}			
			where $\u = \w$ or $\u = \v$. 	 Finally, recall that	$f(\mathbf{w}) := \sum_{i =1}^m f_i(\mathbf{w})$ and each ${f_i: \mathbb{R}^n \to \mathbb{R}}$ is  bounded from below. 	Denote 
	\begin{equation}\label{def-f*}
	f^*:=\sum_{i=1}^m f_i^{\min},\qquad\text{with}~~ f_i^{\min}= {\min}_{\w} f_i(\mathbf{w}).
\end{equation}			
It is easy to see that $f^*\leq {\min}_{\w} \sum_{i=1}^m f_i(\mathbf{w})$.

	\section{\bf Proofs of Theorems in Section III}
		
	\subsection{\bf Some Facts}			
	\noindent	First, let $\mathbb{I}(\cdot)$ be an indicator function defined by 
	\begin{equation}\label{def-indicator}
	\mathbb{I}(j \in \N)=
	\begin{cases}
	1,& \text{if}~j \in \N,\\
	0,& \text{if}~j \notin \N.
	\end{cases}
\end{equation}		
 By Setup \ref{setup}, there is $\kappa_i = k_0, \forall i \in [m]$. We use this $k_0$ to define 
			\begin{equation}\label{def-K-0}
			 \mathcal{K}_0:= \{0, k_0, 2k_0, \cdots\}.
			\end{equation}
			 According to \eqref{SRDFL-localstep} in Algorithm \ref{algorithm-PaME}, for \(\forall i \in [m]\), 
			\begin{equation}\label{local-updates}
				\w_i^{k+1} 
				= \vik - \h_i^k,\qquad \vik = \begin{cases}
				\texttt{PME}(\w_i^k,~\{\w_j^k: j\in{N}_i^{k}\}), & \text{when}~k\in\mathcal{K}_0,\\[1ex]
				 \w_i^k , & \text{when}~k\notin\mathcal{K}_0,\\
				\end{cases}
			\end{equation}
			where
			\begin{equation}\label{def-hi}
				\h_i^{k}
				:=
				\frac{\g_i^{k}}{\sigma^{k}t_i}, \qquad \g_i^k = \nabla f(\vik;\mathcal{B}_i^k).  
			\end{equation}
			For compactness, define a diagonal stepsize matrix and a stacked gradient matrix by
			\begin{equation}
			\label{def-D-G}
			\D^{k}:=
			\left[\begin{array}{cccc}
			\dfrac{1}{\sigma^{k}t_1}&0&\cdots&0\\
			0&\dfrac{1}{\sigma^{k}t_2}&\cdots&0\\
			\vdots&\vdots&\ddots&\vdots\\
			0&0&\cdots&\dfrac{1}{\sigma^{k}t_m}\\
			\end{array}\right],
			\qquad
			\G^{k}:=\bigl(\g_1^{k}~,\g_2^{k}~,\dots,~\g_m^{k}\bigr), 
			\end{equation}
			so that
			\begin{equation}\label{H-DG-sum}
				\H^{k}:=(\h_1^{k},\dots,\h_m^{k})=\G^{k}\D^{k}.
			\end{equation}
Letting $m_i=|\N_i|$ and $t_i=|\N_i^k|$, we denote three average points by  
\begin{equation}
		  \label{def-barw-tildew}
		\boldsymbol{\varpi}^k:=\frac{1}{m}\sum_{i=1}^m \w_i^k,\qquad 
		\widehat{\w}_i^k := \frac{1}{m_i}\sum_{j \in \N_i} \w_j^k,
		\qquad
		\widetilde{\w}_i^k
		:=  \frac{1}{t_i}\sum_{j \in \N_i^k} \w_j^k,
		\end{equation}
		which enable us to denote several matrices used in the sequel,
\begin{equation}\label{def-all-matrices}
\begin{aligned}
\boldsymbol{\Pi}^k&:=(\boldsymbol{\varpi}^k,~\boldsymbol{\varpi}^k,~\cdots,~\boldsymbol{\varpi}^k)=\W^k \mathbf{J},\\[1ex]
\W^k&:=(\w_1^k,~\w_2^k,~\cdots,~\w_m^k),\\[1ex]
\V^k&:= (\overline{\v}_1^k,~~\overline{\v}_2^k,~~\cdots,~\overline{\v}_m^k),
\\[1ex]
\W^{k}\B&~=( \widehat{\w}_1^k,~ \widehat{\w}_2^k,~\cdots,~ \widehat{\w}_m^k),\\[1ex]
\P^{k}&:=\V^{k}-\W^{k}\B.\end{aligned}\end{equation} 
This first condition in \eqref{def-all-matrices}  
 leads to
\begin{equation}\label{WB-I=0}
\boldsymbol{\Pi}^k(\B-\I)=\W^k\mathbf{J}(\B-\I)\overset{\eqref{JB-BJ-J}}{=}0,\qquad\boldsymbol{\Pi}^k(\B-\mathbf{J})=\W^k\mathbf{J}(\B-\mathbf{J})\overset{\eqref{JB-BJ-J}}{=}0.
\end{equation}
Now,		equation \eqref{local-updates} implies the matrix recursion
			\begin{equation}\label{one-step-comm}
				\W^{k+1}=
				\V^{k}-\H^{k} = \W^{k}\B-\H^{k} + \P^{k}.
			\end{equation}	
Finally, for any $\w_i, i \in [m]$, and $p>0$ we have
				\begin{equation}\label{triangle-ineq}
					\begin{aligned}
						&2\langle \w_1,\w_2 \rangle \le p\|\w_1\|^2 + \frac{1}{p}\|\w_2\|^2, \quad
						&&\|\w_1+\w_2\|^2 \le (1+p)\|\w_1\|^2 + \left(1+\frac{1}{p}\right)\|\w_2\|^2, \\
						&\Bigg\|\sum_{i=1}^{m} \w_i\Bigg\|\le \sum_{i=1}^{m}\|\w_i\|, \quad
						&&\Bigg\|\sum_{i=1}^{m} \w_i\Bigg\|^2 \le m\sum_{i=1}^{m}\|\w_i\|^2.
					\end{aligned}
				\end{equation}

\subsection{\bf Key Lemmas}	
\begin{lemma}			
		  Given an integer $\ell\ge 1$ and $x_1,\dots,x_{\ell}\in\mathbb{R}$ be arbitrary scalars.
			Let $S$ be a uniformly random subset of $\{1, \dots, \ell\}$  without replacement and with size $|S| = \lambda$. Conditioned on the event $\lambda=r$ (where $1 \le r \le \ell$), we have
			\begin{equation}\label{eq:subset-mean-second-moment}
				\mathbb{E}\Bigg[\left(\frac{1}{r}\sum_{j\in S} x_j\right)^2 \Bigm| \lambda=r\Bigg]
				\;\le\;
				\frac{1}{\ell}\sum_{j=1}^{\ell} x_j^2.
			\end{equation}	
			Moreover, letting $\bar x:=\frac{1}{\ell}\sum_{j=1}^{\ell}x_j$ and $\hat x:=\frac{1}{r}\sum_{j\in S}x_j$, the variance formula for simple random sampling without replacement (SRSWOR) gives
			\begin{equation}\label{srswor}
				Var\!\left(\hat x \Bigm| \lambda=r\right)
				= 	\E\left[\left(\hat x-\bar x\right)^2 \Bigm| \lambda=r\right]
				= \frac{\ell-r}{r\,\ell\,(\ell-1)}\sum_{j=1}^{\ell}(x_j-\bar x)^2 .
			\end{equation}			
\end{lemma}				
			\begin{proof}
				By Jensen's inequality,
				\[
				\Bigg(\frac{1}{r}\sum_{j\in S} x_j\Bigg)^2
				\le
				\frac{1}{r}\sum_{j\in S} x_j^2.
				\]
				Taking conditional expectation over the random subset $S$ yields: 
				\[
				\mathbb{E}\Bigg[\Bigg(\frac{1}{r}\sum_{j\in S} x_j\Bigg)^2 \Bigm| \lambda=r\Bigg]
				\le
				\frac{1}{r}\sum_{j=1}^{\ell}\mathbb{P}(j\in S\mid \lambda=r)x_j^2
				= \frac{1}{r} \sum_{j=1}^{\ell}\frac{r}{\ell} x_j^2 
				=
				\frac{1}{\ell}\sum_{j=1}^{\ell}x_j^2.
				\]
This shows \eqref{eq:subset-mean-second-moment}. 				
				Since $S$ is a simple random sample of size $r$ drawn uniformly without replacement from $\{1,\dots,\ell\}$.
				Let $$I_j:=
				\begin{cases}
				1,& \text{if $j$ is selected},\\
				0,& \text{otherwise}.
				\end{cases} $$ 
				Then $\hat x=\frac{1}{r}\sum_{j=1}^{\ell} I_j x_j$
				and $\E[I_j\mid \lambda=r]=\frac{r}{\ell}$, which indicates that for $j\neq k$,
				\[
				\E[I_j I_k\mid \lambda=r]=\frac{\binom{\ell-2}{r-2}}{\binom{\ell}{r}}
				=\frac{r(r-1)}{\ell(\ell-1)}.
				\]
				Expanding the second moment gives
				\begin{align*}			
				\E[\hat x^2\mid \lambda=r]
				&=\frac{1}{r^2}\Bigg(
				\sum_{j=1}^{\ell}\E[I_j\mid \lambda=r]x_j^2
				+\sum_{j\neq k}\E[I_j I_k\mid \lambda=r]x_j x_k
				\Bigg)\\
				&=\frac{1}{r^2}\Bigg(
				\frac{r}{\ell}\sum_{j=1}^{\ell}x_j^2
				+\frac{r(r-1)}{\ell(\ell-1)}\sum_{j\neq k}x_j x_k
				\Bigg)\\
				&= \frac{1}{r^2}\Bigg(
				\frac{r}{\ell}\sum_{j=1}^{\ell}x_j^2
				+\frac{r(r-1)}{\ell(\ell-1)}\left(\ell^2\bar x^2-\sum_{j=1}^{\ell}x_j^2\right)\Bigg) \\	
				&= \frac{1}{r^2}\Bigg(
				\frac{r(\ell-r)}{\ell(\ell-1)}\sum_{j=1}^{\ell}x_j^2
				+\frac{r\ell(r-1)}{\ell-1}\bar x^2 
				\Bigg)\\
					&=\bar x^2
				+
				\frac{\ell-r}{r\,\ell\,(\ell-1)}\Bigg(\sum_{j=1}^{\ell}x_j^2-\ell \bar x^2 \Bigg) \\
				&=
				\bar x^2
				+\frac{\ell-r}{r\,\ell\,(\ell-1)}\sum_{j=1}^{\ell}(x_j-\bar x)^2,			
				\end{align*}
where the third equation is from $$\sum_{j\neq k}x_j x_k=\Bigg(\sum_{j=1}^{\ell}x_j\Bigg)^2-\sum_{j=1}^{\ell}x_j^2
				=\ell^2\bar x^2-\sum_{j=1}^{\ell}x_j^2.$$  The above condition and $\E[\hat x\mid \lambda=r]=\bar x$ from \eqref{Ev-Ew} imply $Var(\hat x\mid \lambda=r)=\E[(\hat x-\bar x)^2\mid \lambda=r]$
				and yield  \eqref{srswor}.
			\end{proof}

	\begin{lemma}\label{consensus-error-p1}
		Let $\{(\W^{k}, {\mathbf{V}}^{k})\}$ be the sequence generated by Algorithm~\ref{algorithm-PaME} with Setup \ref{setup}. Then  for any $k\geq 1$, 
		\begin{equation}
			\label{lemma1}
			\begin{aligned}
				\sum_{i=1}^{m} 2  t_i \E\nl \w_i^k - \overline{\v}_i^k \nr^2 &\le 
				8m \E\nl \W^k - \boldsymbol{\Pi}^k\nr_F^2,\\
				  \E\nl \W^k - \V^k\nr_F^2 &\le 
				\dfrac{4m}{t_{\min}} \E\nl \W^k - \boldsymbol{\Pi}^k\nr_F^2.
			\end{aligned}
		\end{equation}
	\end{lemma}
	\begin{proof}		
	We proof the result by considering two cases.

\noindent \underline{Case 1:  $k \notin \mathcal{K}_0$.} Under such a case,  Algorithm~\ref{algorithm-PaME} states that $\overline{\v}_i^k=\w_i^k$ for any $i\in [m]$. Therefore, \eqref{lemma1} holds clearly.
		
\noindent \underline{Case 2:  $k \in \mathcal{K}_0$.}
		Recall \eqref{pame} that
		\begin{equation}\label{pame-1}
			\overline{v}_{i\ell}^k = \begin{cases}
				\dfrac{\sum_{j \in N_i^k} v_{j\ell}^k}{\lambda_{i,\ell}^k}, & \text{if } \lambda_{i,\ell}^k \neq 0, \\
				w_{i\ell}^k, & \text{if } \lambda_{i,\ell}^k = 0.
			\end{cases}
		\end{equation}
		Let  $\boldsymbol{\varpi}^k$ be defined in \eqref{def-barw-tildew}, and  $\varpi_\ell^k$ be its $\ell$th entry. By \eqref{def-lambda-itk}, it follows
		\begin{equation}\label{S-lambda-w}
	N_{i,\ell}^k := \{j \in N_i^k: v_{j\ell}^k \neq 0\},\qquad	|N_{i,\ell}^k| = \lambda_{i,\ell}^k,\qquad \varpi_\ell^k = \frac{1}{m}\sum_{i=1}^mw_{i\ell}^k,\qquad \forall~\ell\in[n].\end{equation}
Moreover, if $\lambda_{i,\ell}^k=r>0$ then
\begin{equation}\label{v-v-w-r}
\overline{v}_{i\ell}^k = 
				\dfrac{\sum_{j \in N_i^k} v_{j\ell}^k}{\lambda_{i,\ell}^k} = \dfrac{\sum_{j \in N_i^k} w_{j\ell}^k}{r}=\dfrac{\sum_{j \in N_{i,\ell}^k} w_{j\ell}^k}{r}. \end{equation}
Using the above facts, for any $\ell\in[n]$ and any $1\le r\le t_i=|N_i^k|$, we have		
		\begin{equation}
			\label{local-pame-coor}
			\begin{array}{rcl}			
				\E\left(w_{i\ell}^k - \overline{v}_{i\ell}^k\right)^2
				&=& \dsum_{r=1}^{t_i}\mathbb{P}(\lambda_{i,\ell}^k=r)\mathbb{E}\left[(w_{i\ell}^k-\overline{v}_{i\ell}^k)^2\Big|\lambda_{i,\ell}^k=r\right]+\mathbb{P}(\lambda_{i,\ell}^k=0)\mathbb{E}\left[(w_{i\ell}^k-\overline{v}_{i\ell}^k)^2\Big|\lambda_{i,\ell}^k=0\right]\\[1.5ex]
				&\overset{(\ref{pame-1},\ref{v-v-w-r})}{=}& \dsum_{r=1}^{t_i}\mathbb{P}(\lambda_{i,\ell}^k=r)\mathbb{E}\Bigg[\Bigg(w_{i\ell}^k -   \dfrac{1}{r}\dsum_{j \in N_{i,\ell}^k} w_{j\ell}^k  \Bigg)^2 \Big|\lambda_{i,\ell}^k=r\Bigg]+\mathbb{P}(\lambda_{i,\ell}^k=0)\mathbb{E}\left[(w_{i\ell}^k-w_{i\ell}^k)^2\Big|\lambda_{i,\ell}^k=0\right]\\
				&=&  \dsum_{r=1}^{t_i}\mathbb{P}(\lambda_{i,\ell}^k=r)\mathbb{E}\Bigg[\Bigg(w_{i\ell}^k - {\varpi}_\ell^k - \dfrac{1}{r}\dsum_{j \in N_{i,\ell}^k}(w_{j\ell}^k - {\varpi}_\ell^k)\Bigg)^2 \Big|\lambda_{i,\ell}^k=r\Bigg]\\
				&\overset{\eqref{triangle-ineq}}{\le}&\dsum_{r=1}^{t_i}\mathbb{P}(\lambda_{i,\ell}^k=r)\Bigg(2(w_{i\ell}^k-{\varpi}_\ell^k)^2 + 2\mathbb{E}\Bigg[\Bigg(\dfrac{1}{r}\dsum_{j \in N_{i,\ell}^k}(w_{j\ell}^k - {\varpi}_\ell^k)\Bigg)^2 \Big|\lambda_{i,\ell}^k=r\Bigg]\Bigg)
				\\[4ex]
				&\overset{\eqref{eq:subset-mean-second-moment}}{\le}& 2\left(w_{i\ell}^k - {\varpi}_\ell^k\right)^2\dsum_{r=1}^{t_i}\mathbb{P}(\lambda_{i,\ell}^k=r)
				+ \dfrac{2}{t_i}\dsum_{j \in \N_i^k}(w_{j\ell}^k - {\varpi}_\ell^k)^2\dsum_{r=1}^{t_i}\mathbb{P}(\lambda_{i,\ell}^k=r) \\[3.5ex]
				& \le& 2\left(w_{i\ell}^k - {\varpi}_\ell^k\right)^2 + \dfrac{2}{t_i}\dsum_{j \in \N_i^k}(w_{j\ell}^k - {\varpi}_\ell^k)^2.
			\end{array}
		\end{equation}		
Additionally, we have
		\begin{equation}
			\label{directed-graph}
			\begin{aligned}
				\sum_{i=1}^{m} \sum_{j \in \N_i^k}\nl \w_j^k-\boldsymbol{\varpi}^k\nr^2 
				&= \sum_{i=1}^{m}\sum_{j=1}^{m}  \mathbb{I}(j \in \N_i^k)\nl \w_j^k - \boldsymbol{\varpi}^k\nr^2 \\
				&=\sum_{j=1}^{m}\nl \w_j^k - \boldsymbol{\varpi}^k\nr^2 \sum_{i=1}^{m} \mathbb{I}(j \in \N_i^k) \\
				&\le m \sum_{i=1}^{m}\nl \w_i^k - \boldsymbol{\varpi}^k \nr^2.
			\end{aligned}
		\end{equation}	
		where $\mathbb{I}(j \in \N)$ is an indicator function defined by \eqref{def-indicator}.
Combine the above facts yields
		\begin{equation}
			\label{local-pame-final}
			\begin{array}{rcl}		
				\dsum_{i=1}^{m}2 t_i\E\nl \w_i^k-\overline{\v}_i^k\nr^2 
				&=& 2\dsum_{i=1}^{m}  t_i\dsum_{\ell=1}^{n}  \E\left[(w_{i\ell}^k - v_{i\ell}^k)^2\right] \\
				&\overset{\eqref{local-pame-coor}}{\le}& 2\dsum_{i=1}^{m}  t_i\dsum_{\ell=1}^{n} \Bigg(2\left(w_{i\ell}^k - {\varpi}_\ell^k\right)^2 + \dfrac{2}{t_i}\dsum_{j \in \N_i^k}(w_{j\ell}^k - {\varpi}_\ell^k)^2\Bigg) \\
				&=& 
				4\dsum_{i=1}^{m} t_i \nl \w_i^k-\boldsymbol{\varpi}^k\nr^2
				+4\dsum_{i=1}^{m} \dsum_{j \in \N_i^k}\nl \w_j^k-\boldsymbol{\varpi}^k\nr^2\\
				&\overset{\eqref{directed-graph}}{\le}&8m\nl \W^k - \boldsymbol{\Pi}^k\nr_F^2.
			\end{array}
		\end{equation}
where the last inequality uses  $t_i =|N_i^k|\leq|N_i|=m_i\le m$. The second condition in \eqref{lemma1} follows $t_i\geq t_{\min}$
\end{proof}

 \begin{lemma}\label{bound-of-wik-vik}
			Let $\{(\W^k, \V^k)\}$ be the sequence generated by Algorithm~\ref{algorithm-PaME} with Setup \ref{setup}.	 Then under Assumptions \ref{assump-gradientlip} and \ref{assump-double-stochasitic},  
			$$\w_i^{k}\in \mathbb{N}(2\delta),\qquad \overline{\v}_i^{k} \in \mathbb{N}(2\delta),\qquad 
			\forall~ i\in[m],~k\geq0. $$
		\end{lemma} 
		\begin{proof} We first claim that for any $k\geq 0$, 
		\begin{equation}\label{inclusion}
		 \| \w_i^{k}\|_{\infty} \le \phi,~~\forall~i\in[m] ~~\Longrightarrow~~\|\overline{\v}_i^{k}\|_{\infty} \le \phi,~~\forall~i\in[m].
		\end{equation}
In fact, when $k \notin \mathcal{K}_0$, $$\| \overline{\v}_i^{k}\|_{\infty} = \| \w_i^{k}\|_{\infty} \le \phi.$$ When $k \in \mathcal{K}_0$, we consider two cases:
		\begin{itemize}[leftmargin=11pt]
		\item If $\lambda_{i,\ell}^k=0$, we have 
		$$\left|\overline{v}_{i\ell}^{k}\right|\overset{\eqref{pame}}{=} \left| w_{i\ell}^{k}\right|\leq \phi.$$
		\item If $\lambda_{i,\ell}^k\neq0$, 
		\begin{equation*}
\left|\overline{v}_{i\ell}^{k}\right|\overset{\eqref{pame}}{\le}
 \dfrac{\sum_{j\in {N}_i^{k}}\left|v_{j\ell}^k\right|}{\lambda_{i,\ell}^k}\overset{\eqref{v-v-w-r}}{=}\dfrac{\sum_{j \in N_{i,\ell}^k} \left|w_{j\ell}^{k}\right|}{|N_{i,\ell}^k|}\leq \dfrac{\sum_{j \in N_{i,\ell}^k} \phi}{|N_{i,\ell}^k|} = \phi.
		\end{equation*}
		\end{itemize}
Both cases yield $\|\overline{\v}_i^{k}\|_{\infty} \le \phi,$ showing \eqref{inclusion}.  Recall \eqref{stepbound} that 
$$\sigma^0\geq \sigma \geq \dfrac{\varepsilon(2\delta)\gamma}{(\gamma-1)\delta t_{\min} }.$$
This condition indicates 
\begin{equation}\label{sum-eps-sig-t}
\sum_{j=0}^k \dfrac{\varepsilon_i(2\delta)}{\sigma^{j}t_i}\le \sum_{j=0}^\infty \dfrac{\varepsilon_i (2\delta)}{\sigma^{j}t_i} = \dfrac{\varepsilon_i(2\delta)}{\sigma^{0}t_i} \sum_{j=0}^\infty \dfrac{1}{\gamma^{j}} \le  \dfrac{\varepsilon_i(2\delta)}{\sigma^{0}t_i}  \dfrac{\gamma}{\gamma-1} \overset{\eqref{def-eps-t-r}}{\leq}  \dfrac{\varepsilon(2\delta)}{\sigma t_{\min} }  \dfrac{\gamma}{\gamma-1} \le \delta. 
\end{equation}
\begin{itemize}[leftmargin=11pt]
\item When $j=0$, the algorithm has initializations $\w_i^{0}=0$ for all $i\in[m]$, namely, for any $i\in[m]$,
\begin{equation}\label{w0-v0-bd}
\|\w_i^{0}\|_{\infty} \le \delta \leq 2\delta,\qquad \|\overline{\v}_i^{0}\|_{\infty} \overset{\eqref{inclusion}}{\leq} \delta \leq 2\delta. 
\end{equation}
 \item When $j=1$, according to \eqref{local-updates}, for any $i\in[m]$,
\begin{equation}\label{w1-v1-bd}
\begin{array}{rcl} 
				\|\w_i^{1}\|_{\infty} 
				&\overset{\eqref{local-updates}}{\leq}&  
				 \|\overline{\v}_i^{0}\|_{\infty} +\| \h_i^0\|_{\infty}  \\[2ex]
				&\overset{\eqref{w0-v0-bd}}{\leq}& \delta + \| \h_i^0\|_{\infty} \\[2ex]
				 &\overset{\eqref{def-hi}}{=}&  \delta + \dfrac{1}{\sigma^{0}t_i} \|\g_i^{0}\|_{\infty}\\[3ex]
				&\overset{(\ref{def-eps-t-r},\ref{w0-v0-bd})}{\leq}& \delta  + \dfrac{\varepsilon_i(2\delta)}{\sigma^{0}t_i} \overset{\eqref{sum-eps-sig-t}}{\leq} 2\delta,\\[3ex]
				\|\overline{\v}_i^{1}\|_{\infty} &\overset{\eqref{inclusion}}{\leq}& \delta  + \dfrac{\varepsilon_i(2\delta)}{\sigma^{0}t_i} \leq 2\delta.
		\end{array}	\end{equation}
  \item When $j=2$, using the similar reasoning to show \eqref{w1-v1-bd} enables the following conditions,
  \begin{equation}\label{w2-v2-bd}
\begin{array}{rcl} 
				\|\w_i^{2}\|_{\infty} 
				  & \leq& \delta  + \dfrac{\varepsilon_i(2\delta)}{\sigma^{0}t_i}+  \dfrac{\varepsilon_i(2\delta)}{\sigma^{1}t_i} \overset{\eqref{sum-eps-sig-t}}{\leq} 2\delta,\\[3ex]
				\|\overline{\v}_i^{2}\|_{\infty} &\overset{\eqref{inclusion}}{\leq}& \delta  + \dfrac{\varepsilon_i(2\delta)}{\sigma^{0}t_i}+  \dfrac{\varepsilon_i(2\delta)}{\sigma^{1}t_i} \leq 2\delta.
		\end{array}	\end{equation}
		\item When $j=k$, using the similar reasoning to show \eqref{w2-v2-bd} enables the following conditions,
  \begin{equation*}\label{wj-vj-bd}
\begin{array}{rcl} 
				\|\w_i^{k}\|_{\infty} 
				  & \leq& \delta  + \dsum_{j=0}^{k-1} \dfrac{\varepsilon_i(2\delta)}{\sigma^{j}t_i} \overset{\eqref{sum-eps-sig-t}}{\leq} 2\delta,\\[3ex]
				\|\overline{\v}_i^{k}\|_{\infty} &\overset{\eqref{inclusion}}{\leq}& \delta  + \dsum_{j=0}^{k-1} \dfrac{\varepsilon_i(2\delta)}{\sigma^{j}t_i} \leq 2\delta.
		\end{array}	\end{equation*}
\end{itemize}
This shows the desired result.
		\end{proof}

\begin{lemma}\label{C_1-guarantee}
Let $\{(\W^k, \V^k)\}$ be the sequence generated by Algorithm~\ref{algorithm-PaME} with Setup \ref{setup}. For every $i\in[m]$, let 
	\begin{equation}\label{def-nu-xi-C1}
		\begin{aligned}	
		\varrho_i&:=(1-p)^{t_i},\\[1ex]
			c_i&:=\dsum_{r=1}^{t_i} \binom{t_i}{r} p^r (1-p)^{t_i-r}\dfrac{t_i-r}{r(t_i-1)},\\[1ex]
			\xi_i &:= \frac{(1-\varrho_i)(m_i-t_i)}{m_i t_i (m_i-1)},\\[1ex]
			C_1 &:=
				\max_i\Bigg\{   \varrho_i \big(1+ {\zeta}\big)^2
				+ \dsum_{j \in \N_i}\Big(\dfrac{c_j}{m_j}
				+  \xi_j\Big)\Bigg\}.
		\end{aligned}
	\end{equation} 
Here, we denote $0/0=0$ and thus $c_i=0$ when $t_i=1$. Then 
	\begin{equation}\label{cond-contraction}
		 {\zeta}+\sqrt{C_1}<\frac{1}{\gamma^{k_0/2} }.
	\end{equation}
\end{lemma}

\begin{proof}
	For $t_i=1$, $c_i=0\le 1-(1-p)=1-\varrho_i$. For $t_i\ge2$,  
	\[
	c_i
	 \le
	\sum_{r=1}^{t_i}\binom{t_i}{r}p^r(1-p)^{t_i-r}
	=
	1-(1-p)^{t_i}
	=1-\varrho_i.
	\]
Using this fact and the definition of $\xi_j$, we obtain
	\[
	\frac{c_j}{m_j}+\xi_j
	\le  \frac{(1-\varrho_j)(t_j m_j+m_j-2t_j)}{ t_j m_j (m_j-1)} .
	\]
	By Bernoulli's inequality, $1-\varrho_j=1-(1-p)^{t_j}\le t_j p$. Using this condition and $t_j= \lfloor\nu_j m_j\rfloor$ yields
\begin{equation*}
\begin{array}{rcl}
	\dsum_{j\in  N_i}\left(\frac{\nu_j}{m_j}+\xi_j\right)
	&\le&
	p\dsum_{j\in  N_i} \dfrac{   t_j m_j+m_j-2t_j }{  m_j (m_j-1)}\\[3ex]
	&\le&
	p\dsum_{j\in  N_i} \dfrac{   \lfloor\nu_j m_j\rfloor (m_j-2)+m_j }{  m_j (m_j-1)}\\[3ex]
	&\le&
	p\dsum_{j\in  N_i} \dfrac{    \nu_j  (m_j-2)+ 1}{  m_j-1}\\[3ex]
	&\le& 2p\dsum_{j\in  N_i} \nu_j,
\end{array}\end{equation*}
which results in
\begin{equation*} 
	 \varrho_i \big(1+ {\zeta}\big)^2
				+ \dsum_{j \in \N_i}\Big(\dfrac{c_j}{m_j}
				+  \xi_j\Big) \leq  \varrho_i \Big(1+ {\zeta}\Big)^2 +2p \sum_{j\in N_i}\nu_j  
		\overset{\eqref{cond-p-zeta-gamma}}{<}
		\left(\frac{1}{\gamma^{k_0/2} }- {\zeta}\right)^2,
	\end{equation*}
for any $i\in[m]$,	 and hence $\sqrt{C_1}< {\gamma^{-k_0/2} }- {\zeta}, $ showing the desired result.
\end{proof}

	\begin{lemma}\label{projecterror}
Let $\{(\W^k, \V^k)\}$ be the sequence generated by Algorithm~\ref{algorithm-PaME} with Setup \ref{setup} and $C_1$ be defined by (\ref{def-nu-xi-C1}).	 Then under Assumptions \ref{assump-gradientlip} and \ref{assump-double-stochasitic}, it holds
	\begin{equation}\label{eq:P-alpha-bound}
		\E\|\P^{k}\|_F^2
		\le C_1\E\|\W^k-\boldsymbol{\Pi}^k\|_F^2,\qquad \forall~k \in \mathcal{K}_0.
	\end{equation}
	\end{lemma}
	
	\begin{proof}
		Let $\widehat{\w}_i^k$ and $\widetilde{\w}_i^k$ be defined by  \eqref{def-barw-tildew} and  $\widehat{w}_{i\ell}^k$ and $\widetilde{w}_{i\ell}^k$ be their $\ell$th entry. Following the notation in Lemma~\ref{consensus-error-p1}, let $N_{i,\ell}^k\subseteq \N_i^k$ be defined in \eqref{S-lambda-w} and $\lambda_{i,\ell}^k=|N_{i,\ell}^k|$. To bound $$\P^{k} = \V^{k} - \W^{k}\B=\left(\overline{\v}_1^k-\widehat{\w}_1^k,\overline{\v}_2^k-\widehat{\w}_2^k,\cdots,\overline{\v}_m^k-\widehat{\w}_m^k\right),$$ 
		as shown in \eqref{def-all-matrices}, it  suffices to bound $\|\overline{\v}_i^k-\widehat{\w}_i^k\|$ or $\overline{v}_{i\ell}^k - \widehat{w}_{i\ell}^k$. Therefore,  condition on the random neighbor set $\N_i^k$, for any $1\le r\le t_i=|\N_i^k|$, we have 
		\begin{equation}\label{coo-exp-pre}
			\begin{array}{rcl}
				\mathbb{E}\Bigg[\Big(\overline{v}_{i\ell}^k - \widehat{w}_{i\ell}^k\Big)^2 \Bigm| \lambda_{i,\ell}^k = r\Bigg]
				&\overset{\eqref{v-v-w-r}}{=}& \mathbb{E}\Bigg[\Bigg(\dfrac{1}{r}\dsum_{j \in N_{i,\ell}^k} w_{j\ell}^k
				- \widetilde{w}_{i\ell}^k + \widetilde{w}_{i\ell}^k - \widehat{w}_{i\ell}^k\Bigg)^2
				\Bigm| \lambda_{i,\ell}^k = r\Bigg] \\ 
				&=& \mathbb{E}\Bigg[\Bigg(\dfrac{1}{r}\dsum_{j \in N_{i,\ell}^k}w_{j\ell}^k
				- \widetilde{w}_{i\ell}^k \Bigg)^2 \Bigm| \lambda_{i,\ell}^k = r\Bigg]
				+ \mathbb{E}\left[\left( \widetilde{w}_{i\ell}^k - \widehat{w}_{i\ell}^k\right)^2
				\Bigm| \lambda_{i,\ell}^k = r\right] \\[5ex]
				&\overset{\eqref{srswor}}{=}& \dfrac{t_i-r}{r t_i (t_i-1)}\dsum_{j\in\N_i^k}\left(w_{j\ell}^k-\widetilde{w}_{i\ell}^k\right)^2
				+ \left( \widetilde{w}_{i\ell}^k - \widehat{w}_{i\ell}^k\right)^2,
			\end{array}
		\end{equation}	
		where the second equation used the followsing fact,					\begin{equation*} 
			\begin{array}{rcl}
				&& \mathbb{E}\Bigg[2\Bigg(\dfrac{1}{r}\dsum_{j \in N_{i,\ell}^k}w_{j\ell}^k
				- \widetilde{w}_{i\ell}^k \Bigg) \left( \widetilde{w}_{i\ell}^k - \widehat{w}_{i\ell}^k\right) \Bigm| \lambda_{i,\ell}^k = r\Bigg]	\\		
				&=&\mathbb{E}\Bigg[2\Bigg(\dfrac{1}{r}\dsum_{j \in N_{i,\ell}^k}w_{j\ell}^k
				- \widetilde{w}_{i\ell}^k \Bigg)  \Bigm| \lambda_{i,\ell}^k = r\Bigg]\left( \widetilde{w}_{i\ell}^k - \widehat{w}_{i\ell}^k\right)\\	
				&=&2\Bigg(\mathbb{E}\left[\dfrac{1}{r}\dsum_{j \in N_{i,\ell}^k}w_{j\ell}^k
				\Bigm| \lambda_{i,\ell}^k = r\right] - \widetilde{w}_{i\ell}^k\Bigg)  \left( \widetilde{w}_{i\ell}^k - \widehat{w}_{i\ell}^k\right)\\	
				&=&2\Bigg( \dfrac{1}{t_i}\dsum_{j \in N_{i}^k}w_{j\ell}^k
				- \widetilde{w}_{i\ell}^k\Bigg)  \left( \widetilde{w}_{i\ell}^k - \widehat{w}_{i\ell}^k\right)=0.			
			\end{array}
		\end{equation*}	
		We point out that if $t_i=1$, then 
		\begin{equation}
			\label{w-w=0-ti=1}
			\dsum_{j\in\N_i^k}\left(w_{j\ell}^k-\widetilde{w}_{i\ell}^k\right)^2=\left(w_{j\ell}^k- {w}_{j\ell}^k\right)^2=0,
		\end{equation}
		which means \eqref{coo-exp-pre} still holds. In this scenario, we let $0/0=0$. 
		According to Setup \ref{setup},   each node ${j\in N_i^k}$ independently constructs a set ${T_j^k\subseteq[n]}$ by uniformly selecting $s$ entries of $[n]$ without replacement. This indicates that each entry of $\w_j^k$ will be selected in a probability $p=s/n$. Consequently,
		\begin{equation}\label{prob-lambda=r}
			\mathbb{P}\left(\lambda_{i,\ell}^k = 0\right) = (1-p)^{t_i}=\varrho_i,\qquad  \mathbb{P}\left(\lambda_{i,\ell}^k = r\right) = \binom{t_i}{r} p^r (1-p)^{t_i-r},\qquad \sum_{r=1}^{t_i}\mathbb{P}\left(\lambda_{i,\ell}^k = r\right) = 1-\varrho_i.
		\end{equation}
where $\varrho_i$ is defined by (\ref{def-nu-xi-C1}).		Using these fact enable us to derive that
		\begin{equation}\label{coo-exp}
			\begin{array}{rcl}		
				&&\mathbb{E}\big[(\overline{v}_{i\ell}^k-\widehat{w}_{i\ell}^k)^2 \big| \lambda_{i,\ell}^k = r\big]\\[2ex]
				&=& \mathbb{P}\left(\lambda_{i,\ell}^k = 0\right)\mathbb{E}\left[(\overline{v}_{i\ell}^k - \widehat{w}_{i\ell}^k)^2\Bigm| \lambda_{i,\ell}^k = 0\right] 
				+ \dsum_{r=1}^{t_i}\mathbb{P}\left(\lambda_{i,\ell}^k = r\right)
				\mathbb{E}\left[(\overline{v}_{i\ell}^k-\widehat{w}_{i\ell}^k)^2 \Bigm| \lambda_{i,\ell}^k = r\right]\\
				&\overset{\eqref{pame-1}}{=}&\mathbb{P}\left(\lambda_{i,\ell}^k = 0\right)\mathbb{E}\left[(w_{i\ell}^k - \widehat{w}_{i\ell}^k)^2\Bigm| \lambda_{i,\ell}^k =0\right] 
				+ \dsum_{r=1}^{t_i}\mathbb{P}\left(\lambda_{i,\ell}^k = r\right)
				\mathbb{E}\left[(\overline{v}_{i\ell}^k-\widehat{w}_{i\ell}^k)^2 \Bigm| \lambda_{i,\ell}^k = r\right]\\
				&\overset{\eqref{coo-exp-pre}}{=}& \mathbb{P}\left(\lambda_{i,\ell}^k = 0\right) (w_{i\ell}^k - \widehat{w}_{i\ell}^k)^2
				+ \dsum_{r=1}^{t_i}\mathbb{P}\left(\lambda_{i,\ell}^k = r\right)\Bigg(
				\dfrac{t_i-r}{r(t_i-1)}\frac{1}{t_i}\sum_{j \in \N_i^k}(w_{j\ell}^k - \widetilde{w}_{i\ell}^k)^2 +  (\widetilde{w}_{i\ell}^k - \widehat{w}_{i\ell}^k)^2 \Bigg)\\			
				&\overset{\eqref{prob-lambda=r}}{=}& 
\varrho_i (w_{i\ell}^k - \widehat{w}_{i\ell}^k)^2
				+ \dfrac{c_i}{t_i}\dsum_{j \in \N_i^k}(w_{j\ell}^k - \widetilde{w}_{i\ell}^k)^2
				+ \big(1-\varrho_i\big)(\widetilde{w}_{i\ell}^k - \widehat{w}_{i\ell}^k)^2,
			\end{array}
		\end{equation}
		where $c_i$ is defined by \eqref{def-nu-xi-C1}. Here we let $c_i=0$ due to \eqref{w-w=0-ti=1}.  Now we bound the three terms in \eqref{coo-exp} separately.
		\begin{itemize}[leftmargin=12pt]
			\item  {For the first term,} by $\W^{k}\B=(\widehat{\w}_1^k,\cdots,\widehat{\w}_m^k)$, we have 
			\begin{equation}\label{transerr-1}
				\begin{aligned}
					\sum_{i=1}^m \|\w_i^k -  {\widehat{\w}}_i^k\|^2
					= \|\W^k(\I-\B)\|_F^2 \overset{\eqref{WB-I=0}}{=}  \|\W^k(\I - \B) - \boldsymbol{\Pi}^k(\I - \B)\|_F^2 
					\overset{\eqref{bd-I-B}}{\le}
					 (1+ {\zeta} )^2 \|\W^k-\boldsymbol{\Pi}^k\|_F^2.
				\end{aligned}
			\end{equation}		
			\item  {For the second term,} taking expectation with respect to $\N_i^k$ gives
			\begin{equation}\label{transerr-2}
				\begin{aligned}
					\E\Bigg[\sum_{i=1}^m \frac{c_i}{t_i}\sum_{j \in \N_i^k}\big\|\w_j^k -  {\widetilde{\w}}_i^k\big\|^2\Bigg]
					&\le \E\Bigg[\sum_{i=1}^m \frac{c_i}{t_i}\sum_{j \in \N_i^k}\|\w_j^k -  \boldsymbol{\varpi}^k\|^2\Bigg] \\
					&= \sum_{i=1}^m\frac{c_i}{m_i}\sum_{j \in \N_i}\|\w_j^k- \boldsymbol{\varpi}^k\|^2 \\
					&= \sum_{i=1}^m\sum_{j \in \N_i}\frac{\nu_j}{m_j}\|\w_i^k- \boldsymbol{\varpi}^k\|^2,
				\end{aligned}
			\end{equation}
			where the first inequality is from   \eqref{def-barw-tildew}, that is,  $\boldsymbol{\varpi}^k$ and $ {\widetilde{\w}}_i^k$ are the mean of ${\{\w_j^k:j\in[m]\}}$ and ${\{\w_j^k:j\in N_i^k\}}$  and the last equality is derived by exchanging the summation indices $i$ and $j$ using the undirected-graph property.
			
			\item  {For the third term,} we apply the variance formula for simple random sampling with respect to $\N_i^k$,
			\begin{equation}\label{transerr-3}
				\begin{array}{rcl}	
					\dsum_{i=1}^{m}\big(1-\varrho_i\big)\E  \| {\widetilde{\w}}_i^k -  {\widehat{\w}}_i^k\|^2 
					&\overset{\eqref{srswor}}{=}& \dsum_{i=1}^{m} \frac{1-\varrho_i}{m_i}\frac{m_i-t_i}{t_i(m_i-1)} \sum_{j \in \N_i}\|\w_j^k -  {\widehat{\w}}_i^k\|^2 \\
					&\overset{\eqref{def-nu-xi-C1}}{=}& \dsum_{i=1}^{m}\xi_j   \sum_{j \in \N_i}\|\w_j^k -  {\widehat{\w}}_i^k\|^2 \\
					&\le&   \dsum_{i=1}^{m} \sum_{j \in \N_i}\xi_j \|\w_i^k -  \boldsymbol{\varpi}^k\|^2,
				\end{array}
			\end{equation}
			where  the first inequality is  from \eqref{def-barw-tildew}, that is,  $\boldsymbol{\varpi}^k$ and $ {\widehat{\w}}_i^k$ are the mean of ${\{\w_j^k:j\in[m]\}}$ and ${\{\w_j^k:j\in N_i\}}$.  
		\end{itemize}
		Finally, combining \eqref{coo-exp} with \eqref{transerr-1}--\eqref{transerr-3} and applying the law of total expectation,
		we obtain
		\begin{equation}\label{bound-on-P-intermediate}
			\begin{aligned}
				\E\|\P^k\|_F^2
				= \sum_{i=1}^{m}\E\|\overline{\v}_i^k -  {\widehat{\w}}_i^k\|^2  
				\le C_1
				\|\W^k-\boldsymbol{\Pi}^k\|_F^2.
			\end{aligned}
		\end{equation}
		where $C_1$  is defined by \eqref{def-nu-xi-C1}. Therefore,  we arrive at \eqref{eq:P-alpha-bound}.
		Here, the expectation in \eqref{bound-on-P-intermediate} is taken over the randomness at round $k$
		(conditioned on $\W^k$), while \eqref{eq:P-alpha-bound} further averages over all randomness accumulated up to round $k$.
		\end{proof}

	\begin{lemma}\label{consensus-error-p2} 
		Let $\{(\W^k, \V^k)\}$ be the sequence generated by Algorithm~\ref{algorithm-PaME} with Setup \ref{setup}.	 Then under Assumptions \ref{assump-gradientlip} and \ref{assump-double-stochasitic}, there exist a constant $C_2>0$ and $\beta>\gamma$ such that  
		\begin{equation}
			\label{lemma3}
			\E\nl\W^k - \boldsymbol{\Pi}^k\nr_F^2 \le \frac{C_2}{\beta^k}.
		\end{equation}
	\end{lemma}		
	
	\begin{proof} 
	For compactness, we define the expected consensus error and the gradient heterogeneity term at step $k$ as
	\begin{equation}\label{def-calW-calG}
		\mathcal{W}^k := \E\nl\W^k - \boldsymbol{\Pi}^k\nr_F^2, \qquad \mathcal{G}^k := \E\nl\H^k(\I - \mathbf{J})\nr_F^2.
	\end{equation}
	By   Lemma \ref{bound-of-wik-vik}, we have  $\overline{\v}_i^{k}\in \mathbb N(2\delta)$ and thus 
\begin{equation}\label{bd-g-v-k}
\|\g_i^{k}\big\|^2 \overset{\eqref{def-hi}}{=}\|\nabla f(\vik;\mathcal{B}_i^k)\|^2 \leq n\|\nabla f(\vik;\mathcal{B}_i^k) \big\|^2_\infty \overset{(\ref{def-eps})}{\leq}  n\varepsilon_i^2(2\delta) \overset{(\ref{def-eps-t-r})}{\leq}  n\varepsilon^2(2\delta).
\end{equation}
As $t_{\min}=\min\{t_1,t_2,\cdots,t_m\}$, $\sigma^{k}=\sigma^{0}\gamma^{k}\geq \sigma \gamma^{k}$, and $\nl\I - \mathbf{J}\nr = 1 $,  
			we have 
			\begin{equation}\label{eq:Gk-bound}
			\begin{array}{rcl}			 
			\mathcal{G}^k &\overset{\eqref{def-calW-calG}}{\le}& \E\nl\H^k\nr_F^2	  \overset{\eqref{H-DG-sum}}{=} \E \nl  \G^{k}\D^{k}\nr_F^2
				 \le  
				\E \|\D^{k}\|^2\|\G^{k}\|_F^2 \\[3ex]
				&
				\overset{\eqref{def-D-G}}{=} &  \dfrac{ \E \|\G^{k}\|_F^2}{\sigma^{2k} t^2_{\min}}\leq \dfrac{ 1}{\sigma^{2k} t^2_{\min}} \dsum_{i=1}^m \E \big\|\g_i^{k}\big\|^2  \leq     
				  \dfrac{\Gamma }{\gamma^{2k}},\qquad \text{with}~ \Gamma:= \dfrac{ mn  \varepsilon^2(2\delta)}{ \sigma^2  t^2_{\min}}.
			\end{array}	
			\end{equation}
		Let $\iota > 0$ satisfy 
	\begin{equation}\label{iota}
	\iota \in \Bigg(0, \frac{1}{\gamma\left(  {\zeta} + \sqrt{C_1}\right)^{2/k_0}}-1\Bigg).
	\end{equation}
	The above range is well defined due to \eqref{cond-contraction}. Then define 
	\begin{equation}\label{def-rho-0-1}
	 \rho_1 := 1+\iota, \qquad \rho_2 := 1+\frac{1}{\iota},\qquad
			\rho := \rho_1\left( {\zeta} + \sqrt{C_1}\right)^2.
			 \end{equation}
It is easy to check that
	\begin{equation}\label{rho-0-1-gamma}
	 \rho\rho_1^{k_0-1} =( 1+\iota)^{k_0}\left( {\zeta} + \sqrt{C_1}\right)^2 \overset{\eqref{iota}}{<} \dfrac{1}{\gamma^{k_0}}.
			 \end{equation}
	\noindent \underline{Case 1: $k \notin \mathcal{K}_0$.} Algorithm~\ref{algorithm-PaME} states that $\overline{\v}_i^k=\w_i^k$. Recalling the update rule \eqref{one-step-comm} and, the consensus error evolves as
	\begin{equation}
		\W^{k+1} - \boldsymbol{\Pi}^{k+1} \overset{\eqref{def-all-matrices}}{=}
		\W^k - \H^k - (\W^k - \H^k)\mathbf{J}
		=(\W^k - \boldsymbol{\Pi}^k) - \H^k(\I - \mathbf{J}).
	\end{equation}
	Taking the Frobenius norm, taking expectation, and using Young's inequality with constant $\iota$, we obtain
	\begin{equation}\label{eq:W-bound-local}
		\mathcal{W}^{k+1} \overset{\eqref{triangle-ineq}}{\le} \rho_1 \mathcal{W}^k + \rho_2\mathcal{G}^k \overset{\eqref{eq:Gk-bound}}{\le} \rho_1 \mathcal{W}^k +  \frac{\rho_2\Gamma }{\gamma^{2k}}.
	\end{equation}	
	\noindent \underline{Case 2: $k \in \mathcal{K}_0$.} By the definition of $\P^k = \V^k - \W^k\B$ and $\mathbf{J}\B=\B\mathbf{J}=\mathbf{J}$ from \eqref{JB-BJ-J},
	\begin{equation}\label{eq:W-evolve-comm}
		\begin{array}{rcl}
			\W^{k+1} - \boldsymbol{\Pi}^{k+1} & \overset{\eqref{WB-I=0}}{=}&
			\W^{k+1}(\mathbf{I}-\mathbf{J}) \\[1ex]
			&\overset{\eqref{one-step-comm}}{=}&(\W^k\B-\H^k+\P^k)(\mathbf{I}-\mathbf{J}) 
			\\[1ex]
			& \overset{\eqref{WB-I=0}}{=}&\W^k(\B - \mathbf{J}) - \boldsymbol{\Pi}^k(\B - \mathbf{J})  + \P^k(\I - \mathbf{J}) - \H^k(\I - \mathbf{J})\\[1ex]
			&=&
			(\W^k - \boldsymbol{\Pi}^k)(\B - \mathbf{J}) + \P^k(\I - \mathbf{J}) - \H^k(\I - \mathbf{J}).
		\end{array}
	\end{equation}
	Applying Young's inequality twice with constants $\iota > 0$, we have
	\begin{equation}\label{eq:W-bound-comm}
		\begin{array}{rcl}
			\mathcal{W}^{k+1} 
			&\le& \rho_1 \E\nl(\W^k - \boldsymbol{\Pi}^k)(\B - \mathbf{J}) + \P^k(\I - \mathbf{J})\nr_F^2 +\rho_2\mathcal{G}^k \\[2ex]
		&\le& \rho_1\left[  \left(1+\dfrac{\sqrt{C_1}}{\zeta}\right)\E  \|(\W^k - \boldsymbol{\Pi}^k)(\B - \mathbf{J}) \|_F^2 + \left(1+\dfrac{\zeta}{\sqrt{C_1}}\right)  \E\nl\P^k(\I - \mathbf{J})\nr_F^2 \right] + \rho_2\mathcal{G}^k \\[2ex]
			& \overset{\eqref{bd-I-B}}{\leq}& \rho_1\left[  \Big(\zeta^2+\zeta\sqrt{C_1} \Big)  \mathcal{W}^k + \left(1+\dfrac{\zeta}{\sqrt{C_1}}\right)  \E\nl\P^k\nr_F^2 \right] + \rho_2\mathcal{G}^k \\[2ex]
			& \overset{\eqref{eq:P-alpha-bound}}{\leq}&\rho_1\Big(  \zeta^2+\zeta\sqrt{C_1}     +  C_1+\zeta\sqrt{C_1} \Big)\mathcal{W}^k + \rho_2\mathcal{G}^k\\[2ex]
			& \overset{(\ref{def-rho-0-1},\ref{eq:Gk-bound})}{\le}& \rho  \mathcal{W}^k +  \dfrac{\rho_2\Gamma }{\gamma^{2k}},
		\end{array}
	\end{equation} 
where the second inequality  also used  $\nl\I - \mathbf{J}\nr_2 = 1$.
		 At the communication step $k = a k_0$, according to \eqref{eq:W-bound-comm},
	\begin{equation}\label{eq:W-comm-step}
		\mathcal{W}^{ak_0+1} \le \rho \mathcal{W}^{ak_0} + \dfrac{\rho_2 \Gamma }{\gamma^{2ak_0}} 
	\end{equation}
	Denote 
		\begin{equation}\label{def-MD}
		M :=(\rho\rho_1^{k_0-1})^{{1}/{k_0}}\overset{\eqref{rho-0-1-gamma}}{<} \frac{1}{\gamma},\qquad P:=P(k_0),\qquad\text{with}~
		P(t):=\rho_2 \Gamma  \Bigg( \rho_1^{t-1}  +   \dsum_{j=1}^{t-1} \dfrac{ \rho_1^{t-1-j}}{\gamma^{2j} } \Bigg).
	\end{equation}
	One can verify that
	$$P \leq \dfrac{\rho_2 \rho_1^{k_0} \gamma^{2} \Gamma }{ \rho_1 \gamma^{2}-1}.$$
	For  any $2 \le t \le k_0$, applying \eqref{eq:W-bound-local} recursively yields,
	\begin{equation*} 
		\begin{array}{rcl}
			\mathcal{W}^{ak_0+t} 
			 &\le&  \rho_1^{t-1} \mathcal{W}^{ak_0+1} + \rho_2 \Gamma^2 \dsum_{j=1}^{t-1}   \dfrac{\rho_1^{t-1-j}}{\gamma^{2(ak_0+j)}} \\
			 &\overset{\eqref{eq:W-comm-step}}{\le}&  \rho \rho_1^{t-1}\mathcal{W}^{ak_0} +    \dfrac{P(t)}{\gamma^{2ak_0}}\\[2ex]
			 & \le& M^{k_0} \mathcal{W}^{ak_0} +    \dfrac{P}{\gamma^{2ak_0}}.
		\end{array}
	\end{equation*}	
	Since $\rho<	M^{k_0}$ and $\rho_2\Gamma\leq P$, by \eqref{eq:W-comm-step}, the above inequality also holds for $t=1$. Overall,
		\begin{equation}\label{eq:W-local-unroll-t}
		\begin{array}{rcl}
			\mathcal{W}^{ak_0+t} \leq  M^{k_0} \mathcal{W}^{ak_0} +    \dfrac{P}{\gamma^{2ak_0}},\qquad\forall t=1,2,\cdots,k_0.
		\end{array}
	\end{equation}
	In particular, when $t=k_0$,  
		\begin{equation}\label{eq:x-rec-final}
		\mathcal{W}^{(a+1)k_0}\le M^{k_0} \mathcal{W}^{ak_0} + \dfrac{P}{\gamma^{2ak_0}} ,\qquad~ \forall a=0,1,2\cdots.
	\end{equation}	
According to the initialization of Algorithm \ref{algorithm-PaME}, $\W^0=0$, we have $ \mathcal{W}^0=0$.	Iterating \eqref{eq:x-rec-final} yields, for any $a\ge 1$,
	\begin{equation*}\label{eq:x-unroll-final}
		\mathcal{W}^{ak_0}
		\le
		M^{ak_0} \mathcal{W}^0
		+
		P\dsum_{\tau=0}^{a-1} \dfrac{M^{(a-1-\tau)k_0}}{\gamma^{2\tau k_0}} =\dfrac{P(M^{ak_0}-\gamma^{-2ak_0})}{M-\gamma^{-2}}.
	\end{equation*}
 Note that  $M$ depends continuously on $\iota$. Thus, we can always select an $\iota$ within the range defined in \eqref{iota} to strictly ensure $M \neq \gamma^{-2}$ without violating \eqref{rho-0-1-gamma}. Therefore, the above condition leads to	
	\begin{equation}\label{eq:x-explicit}
		\begin{array}{rcl}
			\mathcal{W}^{ak_0} 
			\le  Q \varphi^{ak_0},\qquad\text{with}~\varphi := \max\left\{M, \gamma^{-2} \right\} \in (0, 
	1),~~Q:=\dfrac{P}{|M-\gamma^{-2}|}.
		\end{array}
	\end{equation}
 Now for any $t=1,2,\cdots,k_0-1$,
	\begin{equation*} 
		\begin{array}{rcl}
			\mathcal{W}^{ak_0+t} 
			&\overset{\eqref{eq:W-local-unroll-t}}{\le}&  M^{k_0} \mathcal{W}^{ak_0} +    \dfrac{P}{\gamma^{2ak_0}}\\[2ex]
			&\overset{\eqref{rho-0-1-gamma}}{\le}&   \dfrac{\mathcal{W}^{ak_0}}{\gamma}  +\dfrac{P}{\gamma^{2ak_0}}    \\[2ex]
			& \overset{\eqref{eq:x-explicit}}{\le}&      (Q +P ) \varphi^{ak_0}.
		\end{array}
	\end{equation*}	
	This together with \eqref{eq:x-explicit} indicates that for any $t=1,2,\cdots,k_0$,
		\begin{equation}\label{eq:W-local-unroll-t11}
		\begin{array}{rcl}
			\mathcal{W}^{ak_0+t}  \le        (Q +P ) \varphi^{ak_0}=    \dfrac{Q +P}{\varphi^{t} } \varphi^{ak_0+t}
			 \leq   \dfrac{Q +P}{\varphi^{k_0} } \varphi^{ak_0+t}.
		\end{array}
	\end{equation}	
Now let
	\begin{align*}
	\beta:=\dfrac{1}{\varphi}= \min \Bigg\{\dfrac{1}{M}, \gamma^{2} \Bigg\} = \min \Bigg\{ \dfrac{1}{(\rho\rho_1^{k_0-1})^{1/k_0}},{\gamma^{2}}\Bigg\}  \overset{\eqref{rho-0-1-gamma} }{>} \min \Big\{ \gamma , \gamma^{2} \Big\} =  \gamma, 	\end{align*}
which together with \eqref{eq:W-local-unroll-t11} leads to
		\begin{equation*} 
		\begin{array}{rcl}
			\mathcal{W}^{k} \leq  \dfrac{Q +P}{\varphi^{k_0} }  \dfrac{1}{ \beta^k},  
		\end{array}
	\end{equation*}
	showing the desired result.
	\end{proof}

\subsection{\bf Main Results}

\noindent Recall that 
		\begin{eqnarray}\label{def-tilde-H}
		\begin{aligned}
		H^{k}  :=H(\W^k,\V^k,\boldsymbol{\sigma}^k) := \E \sum_{i=1}^{m}\left(
		f_i(\w_i^k) + \frac{\sigma^k t_i}{2}\nl\w_i^k - \overline{\v}_i^k\nr^2\right),\qquad
		\widetilde{H}^{{k}} : = H^{k}+ \frac{C}{\gamma^k}  +  \dfrac{D}{\eta^k}.
		\end{aligned}
		\end{eqnarray}
	Here, the involved constants are defined as 
			\begin{eqnarray}\label{def-C3-C4}
		\begin{aligned}
	\eta:=\dfrac{\beta}{\gamma}>1,\qquad	C:= \dfrac{\gamma}{\gamma-1} \dfrac{ 4mn \varepsilon^2(2\delta)}{ \sigma t_{\min}},\qquad D=    \dfrac{\eta}{\eta-1} \dfrac{4(1+\gamma)m\sigma^0 C_2}{\beta}  , 
		\end{aligned}
		\end{eqnarray}
		where $\beta >\gamma$ and $C_2$ are defined in Lemma \ref{consensus-error-p2}.

	\begin{lemma}\label{bound-of-para}
			Let $\{(\W^k, \V^k)\}$ be the sequence generated by Algorithm~\ref{algorithm-PaME} with Setup \ref{setup}.	 Then under Assumptions \ref{assump-gradientlip} and \ref{assump-double-stochasitic},   
			\begin{equation}\label{lemma4-inequality}
				\begin{aligned}
					\widetilde{H}^{k}  \ge  H^{k}
					> -\infty .
				\end{aligned}
			\end{equation}
		\end{lemma} 
		
		\begin{proof}
From Lemma \ref{bound-of-wik-vik}, we have  $\w_i^{k}\in \mathbb N(2\delta)$, which contribute to
			\begin{equation*}\label{fi-lower-bound}
				\begin{array}{rcl}
					f_i(\overline{\v}_i^{k})
					&\overset{\eqref{lip-inequal}}{\le}&
					f_i(\w_i^{k})
					+ \big\langle\nabla f_i(\w_i^{k}),\,\overline{\v}_i^{k}-\w_i^{k}\big\rangle
					+ \dfrac{\sigma^{k}}{8}\big\|\w_i^{k}-\overline{\v}_i^{k}\big\|^2\\[2ex]
					&\overset{\eqref{triangle-ineq}}{\le}&
					f_i(\w_i^{k})
					+ \dfrac{2}{ \sigma^{k} t_{\min}} \big\|\nabla f_i(\w_i^{k})\big\|^2
					+ \dfrac{(t_{\min}+1)\sigma^{k} }{8}\big\|\w_i^{k}-\overline{\v}_i^{k}\big\|^2\\[2.5ex]
					& \le& 
					f_i(\w_i^{k})
					+  \dfrac{2n}{ \sigma^{k} t_{\min}} \big\|\nabla f_i(\w_i^{k})\big\|^2_{\infty}
					+ \dfrac{ \sigma^{k} t_{\min} }{4}\big\|\w_i^{k}-\overline{\v}_i^{k}\big\|^2\\[3ex]
					&\overset{\eqref{def-eps}}{\le}&
					f_i(\w_i^{k})
					+   \dfrac{2n\varepsilon^2(2\delta)}{\sigma^{k} t_{\min}} 
					+ \dfrac{ \sigma^{k} t_{\min} }{4}\big\|\w_i^{k}-\overline{\v}_i^{k}\big\|^2,
			\end{array}
			\end{equation*}
where the last two inequalities are from $t_{\min}\geq 1$ and $\varepsilon_i(2\delta)\leq  \varepsilon(2\delta)$. 
			Therefore, we obtain
			\begin{equation*}\label{H-alpha-lb}
				\begin{array}{rcl}
					\widetilde{H}^{k}~\ge~ H^{k}
					&\geq &
					\E\dsum_{i=1}^{m}
					\left(f_i(\w_i^{k})
					+ \frac{\sigma^{k} t_{\min} }{2}\big\| \w_i^{k} - \overline{\v}_i^{k} \big\|^2\right) \\[3ex]
					&\overset{\eqref{fi-lower-bound}}{\ge}&
					\E\dsum_{i=1}^{m}
					\left(f_i(\overline{\v}_i^{k})
					- \dfrac{2n\varepsilon^2(2\delta)}{\sigma^{k} t_{\min}} 
					+ \frac{\sigma^{k}t_{\min} }{4}\big\|\w_i^{k}-\overline{\v}_i^{k}\big\|^2\right)\\[3ex]
					&\overset{\eqref{def-f*}}{\ge}& f^*
			-\dfrac{2mn\varepsilon^2(2\delta)}{\sigma^{k} t_{\min}}  \\[3ex]
			&\overset{\eqref{stepbound}}{\ge}&  f^*
			- \dfrac{2mn\varepsilon^2(2\delta)}{\sigma t_{\min} }>-\infty. 
				\end{array}
			\end{equation*}					 
			This completes the proof.
		\end{proof}

	\subsection{\bf Proof of Theorem \ref{the:descent-property}}
	\addtocounter{theorem}{0}
	 \begin{theorem}\label{the:descent-property}
	Let $\{(\W^k, \V^k)\}$ be the sequence generated by Algorithm~\ref{algorithm-PaME} with Setup \ref{setup}.	 
	Then the following statements hold under Assumptions \ref{assump-gradientlip} and \ref{assump-double-stochasitic}.
	\begin{itemize}[leftmargin=14pt]
		\item[1)] 
		For any ${{k} \ge 0}$ and ${i\in[m]}$,  $\w_i^{k}\in \mathbb{N}(2\delta)$ and $\overline{\v}_i^{k} \in \mathbb{N}(2\delta)$.
		\item[2)] 
		For any ${k} \ge 0$,
		\begin{eqnarray} \label{descent-H-H}
		\widetilde{H}^{{k}} - \widetilde{H}^{{k}+1} \ge \sum_{i=1}^{m}\frac{\sigma_i^k t_i}{8}\Big(\|\Delta \w_i^{{k}+1} \|^2 
		+ \|\Delta \overline{\v}_i^{{k}+1} \|^2 \Big), 
		\end{eqnarray}
		where $\widetilde{H}^{{k}}$ is defined by (\ref{def-tilde-H}) and  the involved constants are defined by (\ref{def-C3-C4}).
		\end{itemize}	 
\end{theorem} 
\begin{proof} Since 1) has been shown in Lemma \ref{bound-of-wik-vik}, we only prove 2).  By $\eta= {\beta}/{\gamma}>1$ in \eqref{def-C3-C4} and $ \gamma\in(1,\zeta^{-1/k_0})$ in \eqref{initial-p-gamma-nu}, 
			\begin{equation}\label{gamma-eta-dec}
		 \dfrac{1}{\gamma^{k}}=\dfrac{\gamma}{\gamma-1}\left(\dfrac{1}{\gamma^{k}}-\dfrac{1}{\gamma^{k+1}}\right),\qquad \dfrac{1}{\eta^{k}}=\dfrac{\eta}{\eta-1}\left(\dfrac{1}{\eta^{k}}-\dfrac{1}{\eta^{k+1}}\right).
	\end{equation}
For any $k\geq 0$,
\begin{equation}\label{sig-t-w-v}
			\begin{array}{rcl}
			 \dsum_{i=1}^{m} 2\sigma^{k} t_i \E\nl \w_i^{k+1} - \overline{\v}_i^{k+1} \nr^2 
				~\overset{\eqref{lemma1}}{\le}~ 8 m \sigma^{k}\E\nl \W^{k+1} - \boldsymbol{\Pi}^{k+1}\nr_F^2~\overset{\eqref{lemma3}}{\le}~ \dfrac{8m\sigma^0 C_2}{\beta} \dfrac{\gamma^{k}}{\beta^{k}} =  \dfrac{8m\sigma^0 C_2}{\beta}  \dfrac{1}{\eta^{k}}.
		\end{array}
		\end{equation}
From Lemma \ref{bound-of-wik-vik}, we have  $\w_i^{k}\in \mathbb N(2\delta)$ and $\overline{\v}_i^{k}\in \mathbb N(2\delta)$, which   yields 
		\begin{equation}\label{bound-h}
		\nl\nabla f_i(\w_i^k)\nr^2\leq n\nl\nabla f_i(\w_i^k)\nr^2_\infty \overset{\eqref{def-eps}}{\le}  n \varepsilon_i^2(2\delta),\qquad	\nl \h_i^k \nr^2 =   \frac{\|\g_i^{k}\|^2}{(\sigma^k t_i)^2} \overset{\eqref{bd-g-v-k}}{\le} \frac{n \varepsilon_i^2(2\delta)}{(\sigma^k t_i)^2}.
		\end{equation}		
		Recall the descent sequence,
	\begin{equation}\label{def-H}
		H(\W^k,\V^k,\boldsymbol{\sigma}^k) 
		:= \E \sum_{i=1}^{m}\left(
		f_i(\w_i^k) + \frac{\sigma^k t_i}{2}\nl\w_i^k - \overline{\v}_i^k\nr^2\right).
	\end{equation}
	We decompose its change between two consecutive communication rounds as
	\begin{equation}\label{eq:H-decomp}
		\begin{aligned}
			&H(\W^{k+1},\V^{k+1},\boldsymbol{\sigma}^{k+1})
			- H(\W^{k},\V^{k},\boldsymbol{\sigma}^{k}) = \Delta H_\sigma^k + \Delta H_V^k + \Delta H_W^k,
		\end{aligned}
	\end{equation}
	where
	\begin{equation}\label{def-delta-H-blocks}
		\begin{aligned}
			\Delta H_\sigma^k
			&:= H(\W^{k+1},\V^{k+1},\boldsymbol{\sigma}^{k+1})
			- H(\W^{k+1},\V^{k+1},\boldsymbol{\sigma}^{k}), \\[1ex]
			\Delta H_V^k
			&:= H(\W^{k+1},\V^{k+1},\boldsymbol{\sigma}^{k})
			- H(\W^{k+1},\V^{k},\boldsymbol{\sigma}^{k}), \\[1ex]
			\Delta H_W^k
			&:= H(\W^{k+1},\V^{k},\boldsymbol{\sigma}^{k})
			- H(\W^{k},\V^{k},\boldsymbol{\sigma}^{k}).
		\end{aligned}
	\end{equation}
$\bullet$ For $\Delta H_\sigma^k$, it follows  
		\begin{equation}
			\label{H-sigma}
			\begin{array}{rcl} 
				\Delta H_\sigma^{k} 
				&\overset{\eqref{def-delta-H-blocks}}{=}& \E \dsum_{i=1}^m \frac{\bigl(\sigma^{k+1} - \sigma^{k}\bigr) t_i}{2} 
				\nl \w_i^{k+1} - \overline{\v}_i^{k+1} \nr^2  \\[3ex]
				&\overset{\eqref{SRDFL-CRstep}}{=}& 
				\E \dsum_{i=1}^m \frac{(\gamma-1)\sigma^{k} t_i}{2} 
				\nl \w_i^{k+1} - \overline{\v}_i^{k+1} \nr^2.
			\end{array}
		\end{equation}
$\bullet$  For $\Delta H_V^{k}$, we have
		\begin{equation}
			\label{H-V}
			\begin{array}{rcl}
				\Delta H_V^{k} 
				&\overset{\eqref{def-delta-H-blocks}}{=}& \E \dsum_{i=1}^m \frac{\sigma^{k} t_i}{2} 
				\left(\nl \w_i^{k+1}- \overline{\v}_i^{k+1} \nr^2 
				- \nl \w_i^{k+1} - \overline{\v}_i^{k} \nr^2 \right)  \\[3ex]
				&=& \E \dsum_{i=1}^m \frac{\sigma^{k} t_i}{2}  \left( 
				2 \left\langle \w_i^{k+1} - \overline{\v}_i^{k+1}, 
				\overline{\v}_i^{k} - \overline{\v}_i^{k+1} \right\rangle 
				-  \nl \Delta \overline{\v}_i^{k+1} \nr^2\right) \\[3ex]
				&\overset{\eqref{triangle-ineq}}{\le}& 
				\E \dsum_{i=1}^m 
				\left(\frac{2\sigma^{k} t_i }{3}
				\nl \w_i^{k+1} - \overline{\v}_i^{k+1} \nr^2 
				-  \frac{\sigma^{k} t_i}{8}
				\nl \Delta \overline{\v}_i^{k+1} \nr^2\right) .
			\end{array}
		\end{equation}
		$\bullet$  For $\Delta H_W^{k}$,   it follows from		
		$t_i\geq 1$ that
		\begin{equation}\label{fi-diff-base}
			\begin{array}{rcl} 
				f_i(\w_i^{k+1}) - f_i(\w_i^{k}) 
				&\overset{\eqref{lip-inequal}}{\le}  &
				\left\langle \nabla f_i(\w_i^{k}), \Delta \w_i^{k+1} \right\rangle 
				+ \dfrac{\sigma^{k}}{8} \nl\Delta \w_i^{k+1}\nr^2 \\[3ex]
				&\overset{\eqref{triangle-ineq}}{\le} &
				\dfrac{2}{\sigma^{k} t_i} \nl\nabla f_i(\w_i^{k})\nr^2
				+ \dfrac{\sigma^{k} t_i}{4}
				\nl\Delta \w_i^{k+1} \nr^2.
			\end{array}
		\end{equation}
		Moreover,
		\begin{equation}\label{quad-diff-base}
			\begin{array}{rcl} 
			 \nl \w_i^{k+1}- \overline{\v}_i^{k} \nr^2 
				- \nl \w_i^{k} - \overline{\v}_i^{k} \nr^2 
				 &=  &
				2\left\langle \w_i^{k+1} - \overline{\v}_i^{k}, \Delta \w_i^{k+1} \right\rangle 
				-  \nl\Delta \w_i^{k+1} \nr^2\\[2ex] 
				&\overset{ \eqref{local-updates}}{=}  &
				2\left\langle -\h_i^{k}, \Delta \w_i^{k+1} \right\rangle 
				-  \nl\Delta \w_i^{k+1} \nr^2\\  [2ex] 
				  &\overset{\eqref{triangle-ineq} }{\le} &	
				4  
				\nl\h_i^{k}\nr^2 
				- \dfrac{3}{4}
				\nl\Delta \w_i^{k+1} \nr^2.	\end{array}
		\end{equation}	
		Combining \eqref{fi-diff-base} and \eqref{quad-diff-base}, we obtain
		\begin{equation}
			\label{H-W}
			\begin{array}{rcl} 
				\Delta H_W^{k} 
				& \overset{\eqref{def-delta-H-blocks}}{=}& \E\dsum_{i=1}^m 
				\left(
				f_i(\w_i^{k+1}) - f_i(\w_i^{k}) 
				+ \frac{\sigma^{k} t_i}{2} 
				\left(\nl \w_i^{k+1} - \overline{\v}_i^{k} \nr^2 
				- \nl \w_i^{k} - \overline{\v}_i^{k} \nr^2\right)
				\right) \\[3ex]
				&\overset{(\ref{fi-diff-base},\ref{quad-diff-base})}{\le}& \E\dsum_{i=1}^m 
				\left(
				\frac{2}{\sigma^{k}t_i} \nl\nabla f_i(\w_i^{k})\nr^2
				-\frac{\sigma^{k} t_i}{8}
				\nl\Delta \w_i^{k+1} \nr^2  
				+ {2\sigma^{k} t_i} 
				\nl\h_i^{k}\nr^2 
				\right)\\[3ex]
				&\overset{\eqref{bound-h}}{\le}& 
				\E \dsum_{i=1}^m \left(
				\frac{ 4n \varepsilon_i^2(2\delta)}{ \sigma^{k} t_i} 
				-\frac{\sigma^{k}t_i}{8}
				\nl\Delta \w_i^{k+1} \nr^2  \right)\\[3ex]
				&\overset{\eqref{def-eps-t-r}}{\le}& 
				\dfrac{ 4mn \varepsilon^2(2\delta)}{ \sigma^{k} t_{\min}}  - \E \dsum_{i=1}^m  \frac{\sigma^{k}t_i}{8}
				\nl\Delta \w_i^{k+1} \nr^2.
			\end{array}			
		\end{equation}
Overall, using \eqref{H-sigma}, \eqref{H-V}, and \eqref{H-W}, we derive that
		\begin{equation}
			\begin{array}{rcl} 
			&&	H(\W^{k+1},\V^{k+1},\boldsymbol{\sigma}^{k+1}) - H(\W^{k},\V^{k},\boldsymbol{\sigma}^{k}) 
				\overset{\eqref{eq:H-decomp}}{=}\Delta H_\sigma^{k} + \Delta H_V^{k} + \Delta H_W^{k} 	\\[3ex]
				&\le& - \E\dsum_{i=1}^{m}\left(
				\frac{\sigma^{k} t_i}{8}\nl\Delta \w_i^{k+1} \nr^2 
				+ \frac{\sigma^{k} t_i}{8}\nl\Delta \overline{\v}_i^{k+1} \nr^2 \right) + \dfrac{ 4mn \varepsilon^2(2\delta)}{ \sigma^{k} t_{\min}}  + \E\dsum_{i=1}^{m} \frac{(1+\gamma)\sigma^{k} t_i}{2} \nl \w_i^{k+1} - \overline{\v}_i^{k+1} \nr^2  \\[3ex]
				& \overset{\eqref{sig-t-w-v}}{\le}&   - \E \dsum_{i=1}^{m}
				\left(
				 \frac{\sigma^{k} t_i}{8}\|\Delta \w_i^{k+1} \|^2 
				+ \frac{\sigma^{k} t_i}{8}\|\Delta \overline{\v}_i^{k+1} \|^2 
				\right) + \dfrac{ 4mn \varepsilon^2(2\delta)}{ \sigma t_{\min}}\dfrac{1}{ \gamma^{k}}   
				+ \dfrac{4(1+\gamma)m\sigma^0 C_2}{\beta}  \dfrac{1}{\eta^{k}}\\[3ex]
				& \overset{\eqref{gamma-eta-dec}}{=}& -\E \dsum_{i=1}^{m}
				\left(
				\frac{\sigma^{k} t_i}{8}\|\Delta \w_i^{k+1} \|^2 
				+ \frac{\sigma^{k} t_i}{8}\|\Delta \overline{\v}_i^{k+1} \|^2 
				\right) +  C \left( \dfrac{1}{\gamma^{k}} - \dfrac{1}{\gamma^{k+1}}\right)
				+ D \left( \dfrac{1}{\eta^{k}} - \dfrac{1}{\eta^{k+1}}\right),
			\end{array}
		\end{equation}	
		where   $C$ and $D$ are defined by (\ref{def-C3-C4}). \end{proof}

\begin{definition}[\bf $L^2(\Omega)$ space]
Let $(\Omega,\mathcal{F},\mathbb{P})$ be a probability space. 
Define
\[
L^2(\Omega) 
:= \left\{ X : \Omega \to \mathbb{R}^n \text{ measurable} 
\;\middle|\; 
\mathbb{E}\|X\|^2 < \infty \right\}.
\]
Two random variables $X$ and $Y$ are identified if 
$\mathbb{P}(X=Y)=1$. 
For simplicity, write $L^2:=L^2(\Omega)$. 
\end{definition}

\begin{definition}[\bf $L^2$-convergence]
Let $(\Omega,\mathcal{F},\mathbb{P})$ be a probability space. 
Let $\{X_k\}_{k\ge 1}$ and $X\in\mathbb{R}^n$ be random vectors such that
\[
\mathbb{E}\|X_k\|^2 < \infty 
\quad \text{and} \quad 
\mathbb{E}\|X\|^2 < \infty.
\]
We say that $X_k$ converges to $X$ in $L^2$ (or in mean square) if
\[
\lim_{k\to \infty} \mathbb{E}\|X_k - X\|^2 =0.
\]
\end{definition}

\begin{lemma}
\label{lemma:mean_convergence}
Let \(\{\mathbf{w}^k\}\subseteq \mathbb{N}(2\delta)\) be a sequence of random vectors.
Suppose there exist constants \(c>0\) and \(\gamma>1\) such that
\[
\mathbb{E}\|\mathbf{w}^{k+1}-\mathbf{w}^k\|^2 \le \frac{c^2}{\gamma^k}, \quad \forall k \ge 0,
\]
and let \(f:\mathbb{R}^d \to \mathbb{R}\) be continuously differentiable. Then the following statements are valid.
\begin{itemize}[leftmargin=12pt]
    \item[1)]  Sequence \(\{\mathbf{w}^k\}\) converges (to $\mathbf{w}^\infty$) in $L^2$ space and sequence  \(\{\mathbb{E}\mathbf{w}^k\}\) converges, i.e.,
 \begin{equation}\label{l2-convergence}
 \lim_{k\to\infty} \mathbb{E}\|\mathbf{w}^k-\mathbf{w}^\infty\|^2=0,\qquad \lim_{k\to\infty} \mathbb{E}\mathbf{w}^k=\E\mathbf{w}^\infty.
\end{equation}   
    \item[2)]  The nonlinear expectation converges to the function evaluated at the limit:
    \[
    \lim_{k\to\infty} \mathbb{E}|f(\mathbf{w}^k) -   f(\mathbf{w}^\infty)|=0.
    \]
    \item[3)] The sequence satisfies
       \[
    \mathbb{E} \|\mathbf{w}^k-\mathbf{w}^\infty\|^2=O(\gamma^{-k}),\qquad \E| f(\mathbf{w}^k)- f(\mathbf{w}^\infty)|  =O(\gamma^{-k/2}).
    \] 
\end{itemize}
\end{lemma}

\begin{proof} 1) By the triangle inequality and Cauchy-Schwarz:
\[
\|\mathbb{E} \mathbf{w}^{k+1}  - \mathbb{E} \mathbf{w}^k \| \le \mathbb{E}\|\mathbf{w}^{k+1}-\mathbf{w}^k\| \le \sqrt{\mathbb{E}\|\mathbf{w}^{k+1}-\mathbf{w}^k\|^2} \le \frac{c}{\gamma^{k/2}}.
\]
Since \(\sum_{k=0}^\infty \gamma^{-k/2} < \infty\), the sequence of means \(\{\mathbb{E}\mathbf{w}^k\}\) has summable increments and is thus Cauchy, thus it converges.  
 From the assumption, by Minkowski's inequality in  $L^2$,  for \(k>\ell\),
\[
\sqrt{\mathbb{E}\|\mathbf{w}^k - \mathbf{w}^\ell\|^2} \le \sum_{j=\ell}^{k-1} \sqrt{\mathbb{E}\|\mathbf{w}^{j+1}-\mathbf{w}^j\|^2}  \leq \sum_{j=\ell}^{k-1} \frac{c}{\gamma^{j/2}} \leq \frac{\sqrt{\gamma}}{\sqrt{\gamma}-1}\frac{1}{\gamma^{\ell/2}}\to0\qquad\text{as}~\ell\to\infty.
\]
Therefore, sequence \(\{\mathbf{w}^k\}\) is Cauchy in the \(L^2\) space. By the completeness of \(L^2\), there is \(\mathbf{w}^\infty\) satisfying the first condition in \eqref{l2-convergence}, 
which immediately yields
\begin{equation}\label{Ef-Ef-infty-0}
 \lim_{k\to\infty} \|\mathbb{E}\mathbf{w}^k-\E\mathbf{w}^\infty\| \leq \lim_{k\to\infty} \mathbb{E} \|\mathbf{w}^k- \mathbf{w}^\infty\| \leq \lim_{k\to\infty} \sqrt{\mathbb{E}\|\mathbf{w}^k-\mathbf{w}^\infty\|^2} =0 .
\end{equation}
2) It is noted that \(\{\mathbf{w}^k\}\) is bounded in $\mathbb{N}(2\delta)$ and \(f\) is continuously differentiable, so  \(f\) is Lipschitz continuous on  a bounded region containing $\mathbb{N}(2\delta)$ and $\mathbf{w}^\infty$. This implies that there is a constant $c>0$ such that
$$|f(\mathbf{w}^k)-f(\mathbf{w}^\infty)|\leq c\|\mathbf{w}^k-\mathbf{w}^\infty\|.$$
Taking the expectation leads to
\begin{equation}\label{Ef-Ef-infty}
\begin{array}{rcl}
 \lim\limits_{k\to\infty}\E | f(\mathbf{w}^k)- f(\mathbf{w}^\infty)| 
 \leq   \lim\limits_{k\to\infty} c\E\|\mathbf{w}^k-\mathbf{w}^\infty\| \overset{\eqref{Ef-Ef-infty-0}}{=}  0.
\end{array}
\end{equation}
3) Let $N>k$ and consider the difference,
\[
\Delta_N := \left| \mathbb{E}\|\mathbf{w}^N - \mathbf{w}^k\|^2 - \mathbb{E}\|\mathbf{w}^\infty - \mathbf{w}^k\|^2 \right|.
\]
We first show $\lim_{N\to\infty} \Delta_N = 0$. By linearity of expectation and the triangle inequality,
\begin{equation*}
\begin{array}{rcl}
\Delta_N &\le& \mathbb{E}\left| \|\mathbf{w}^N - \mathbf{w}^k\|^2 - \|\mathbf{w}^\infty - \mathbf{w}^k\|^2 \right|\\[2ex]
&\le& \mathbb{E}\Big(\left| \|\mathbf{w}^N - \mathbf{w}^k\| - \|\mathbf{w}^\infty - \mathbf{w}^k\| \right| ( \|\mathbf{w}^N - \mathbf{w}^k\| + \|\mathbf{w}^\infty - \mathbf{w}^k\|)\Big)\\[2ex]
&\le& \mathbb{E} \Big(\|\mathbf{w}^N - \mathbf{w}^\infty\|  ( \|\mathbf{w}^N - \mathbf{w}^k\| + \|\mathbf{w}^\infty - \mathbf{w}^k\|)\Big).
\end{array}
\end{equation*}
Taking expectations of the both sides of the above inequality leads to
\begin{equation}\label{bd-gap-N}
\begin{array}{rcl}
\Delta_N 
&\le& \mathbb{E}\Big( \|\mathbf{w}^N - \mathbf{w}^\infty\| \cdot \left( \|\mathbf{w}^N - \mathbf{w}^k\| + \|\mathbf{w}^\infty - \mathbf{w}^k\| \right) \Big) \\[2ex]
&\le& \sqrt{\mathbb{E}\|\mathbf{w}^N - \mathbf{w}^\infty\|^2} \cdot \sqrt{\mathbb{E} \left( \|\mathbf{w}^N - \mathbf{w}^k\| + \|\mathbf{w}^\infty - \mathbf{w}^k\| \right)^2  }\\[2ex]
&\le& \sqrt{\mathbb{E}\|\mathbf{w}^N - \mathbf{w}^\infty\|^2} \cdot \left(\sqrt{\mathbb{E}\|\mathbf{w}^N - \mathbf{w}^k\|^2} + \sqrt{\mathbb{E}\|\mathbf{w}^\infty - \mathbf{w}^k\|^2}\right)\\[2ex]
&\le& d \sqrt{\mathbb{E}\|\mathbf{w}^N - \mathbf{w}^\infty\|^2},
\end{array}
\end{equation}
where the second and third inequalities are from Cauchy-Schwarz inequality and  Minkowski's inequality in $L^2$, the last inequality holds because the second factor can be bounded by a constant $d$ independent of $N$ due to $\{\mathbf{w}^k\}\subseteq \mathbb{N}(2\delta)$. Then
\[
\lim_{N\to\infty} \Delta_N ~\overset{\eqref{bd-gap-N}}{\le}~  \lim_{N\to\infty} d   \sqrt{\mathbb{E}\|\mathbf{w}^N - \mathbf{w}^\infty\|^2} ~\overset{\eqref{l2-convergence}}{=}~ 0,
\]
that is
\begin{equation}\label{EwN-k=Ewinfty-k}
\lim_{N\to\infty} \mathbb{E}\|\mathbf{w}^N - \mathbf{w}^k\|^2 = \mathbb{E}\|\mathbf{w}^\infty - \mathbf{w}^k\|^2.
\end{equation}
From the telescoping representation
$$\mathbf{w}^N-\mathbf{w}^k = \sum_{j=k}^{N-1} (\mathbf{w}^{j+1}-\mathbf{w}^j),$$
apply Minkowski inequality in $L^2$,
\begin{equation*}
\begin{array}{rcl}
\sqrt{ \mathbb{E}\|\mathbf{w}^\infty - \mathbf{w}^k\|^2} &\overset{\eqref{EwN-k=Ewinfty-k}}{=}&  \sqrt{  \lim\limits_{N\to\infty} \E\|\mathbf{w}^k-\mathbf{w}^{N} \|^2}  \\[3ex]
& =& \lim\limits_{N\to\infty} \sqrt{ \E\|\mathbf{w}^k-\mathbf{w}^{N} \|^2}  \\[2ex]
&\leq& \lim\limits_{N\to\infty} \dsum_{j=k}^{N-1} \sqrt{\E\|\mathbf{w}^{j+1}-\mathbf{w}^j \|^2}  \\[2ex]
&\leq& \lim\limits_{N\to\infty} \dsum_{j=k}^{N-1} \frac{c}{\gamma^{j/2}} = \dfrac{\sqrt{\gamma}}{\sqrt{\gamma}-1}\frac{1}{\gamma^{k/2}}.
\end{array}
\end{equation*}
This immediately shows that
\begin{equation*}
 \E\|\mathbf{w}^k-\mathbf{w}^\infty \|^2 =O(\gamma^{-k}).
\end{equation*}
Using the above condition and \eqref{Ef-Ef-infty}, we derive that
\begin{equation*} 
 \E| f(\mathbf{w}^k)- f(\mathbf{w}^\infty)|  \leq    c\sqrt{\E\|\mathbf{w}^k-\mathbf{w}^\infty\|^2} =O(\gamma^{-k/2}).
\end{equation*}
The proof is finished.
\end{proof}

\subsection{\bf Proof of Theorem \ref{the:convergence}}

\begin{theorem}\label{the:convergence}
			Let $\{(\W^k, \V^k)\}$ be the sequence generated by Algorithm~\ref{algorithm-PaME} with Setup \ref{setup}.	Then the following statements hold  under Assumptions \ref{assump-gradientlip} and \ref{assump-double-stochasitic}.
	\begin{itemize}[leftmargin=14pt]
			\item[1)] 
			For any $i \in [m]$,
			\begin{equation}\label{gap-vanishing}
			 \lim\limits_{{k} \rightarrow \infty}\E\nl \Delta \w_i^{k}\nr 
			= \lim\limits_{{k} \rightarrow \infty}\E\nl \Delta \overline{\v}_i^{k}\nr 
			= \lim\limits_{{k} \rightarrow \infty}\E\|\w_i^{k} - \overline{\v}_i^{k}\|=0.
			\end{equation}
			\item[2)] Sequence $\{\boldsymbol{\varpi}^{{k}}\}$ converges to $\boldsymbol{\varpi}^{\infty}$ in the sense of $L^2$ convergence and expectation, namely,
			 \begin{equation} \label{E-w-k-L2}
 \lim_{k\to\infty} \mathbb{E}\|\boldsymbol{\varpi}^{{k}}-\boldsymbol{\varpi}^\infty\|^2=0,\qquad  \lim_{k\to\infty} \E\boldsymbol{\varpi}^{{k}}=\E\boldsymbol{\varpi}^{\infty}.
\end{equation}   Moreover, sequence  $\{(\E\W^{k}, \E{\mathbf{V}}^{k})\}$ converges and satisfies
				\begin{equation*}
			 \lim\limits_{{k} \rightarrow \infty} \E\W^{k} 
			= \lim\limits_{{k} \rightarrow \infty} \E\V^{k}
			= (\E\boldsymbol{\varpi}^{\infty},\cdots,\E\boldsymbol{\varpi}^{\infty})=:\W^\infty.
			\end{equation*}	 
			\item[3)] 
			Three sequences $\{H^{k}\}$, $\{\widetilde{H}^{k}\}$, and $\{\E f(\boldsymbol{\varpi}^{{k}})\}$ converge to the same value, namely,
			\begin{equation*}
			  \lim\limits_{{k} \rightarrow \infty} \widetilde{H}^{k} =\lim\limits_{{k} \rightarrow \infty} H^{k}  = \lim\limits_{{k} \rightarrow \infty} \E f(\boldsymbol{\varpi}^{{k}})=  \E f(\boldsymbol{\varpi}^{\infty}).
			\end{equation*} 			 
		\end{itemize}
	\end{theorem}
\begin{proof} 1) From \eqref{descent-H-H} and \eqref{lemma4-inequality}, sequence $\{\widetilde{H}^{{k}}\}$ is non-increasing and bounded from below. Therefore, it converges. Taking the limit of the both sides of  \eqref{descent-H-H} yield
\begin{equation*}
\lim_{k\to\infty}\E \dsum_{i=1}^{m}
				\left(
				\frac{\sigma^{k} t_i}{8}\|\Delta \w_i^{k+1} \|^2 
				+\frac{\sigma^{k} t_i}{8}\|\Delta \overline{\v}_i^{k+1} \|^2 
				\right)=0,
\end{equation*}
namely, for each $i\in[m]$,
\begin{equation}\label{limit-sigmak-gap}
\lim_{k\to\infty}\sigma^{k} \E    \|\Delta \w_i^{k+1} \|^2 =0,\qquad 
				  \lim_{k\to\infty}\sigma^{k}\E   \|\Delta \overline{\v}_i^{k+1} \|^2  =0,
\end{equation}
This immediately leads to the first two equations in \eqref{gap-vanishing}.
			The third equation in \eqref{gap-vanishing} holds due to
			\begin{equation}\label{sig-t-w-v-1}
			\begin{array}{rcl}
			 \dsum_{i=1}^{m} 2 t_i \E\nl \w_i^{k} - \overline{\v}_i^{k} \nr^2 
				~\overset{\eqref{lemma1}}{\le}~ 8 m \E\nl \W^{k} - \boldsymbol{\Pi}^{k}\nr_F^2~\overset{\eqref{lemma3}}{\le}~   \dfrac{8m C_2}{\beta^{k}}.
		\end{array}
		\end{equation}
2) By \eqref{limit-sigmak-gap}, we conclude that 
\begin{equation*} 
  \E    \|\Delta \w_i^{k+1} \|^2 = o(\sigma^{-k})=o\left(\gamma^{-k}\right),\qquad \forall i\in[m],
\end{equation*}
which indicates that there exists a constant $c_i>0$ such that
\begin{equation}\label{limit-sigmak-gap-1}
  \E    \|\Delta \w_i^{k+1} \|^2  \leq  \dfrac{c_i^2}{\gamma^{k}},\qquad \forall i\in[m].
\end{equation}
Recall \eqref{def-barw-tildew} that  $\boldsymbol{\varpi}^{k} =\frac{1}{m}\sum_{i=1}^m \w_i^{k}$  and Lemma \ref{bound-of-wik-vik} that $\w_i^{k}\in \mathbb{N}(2\delta)$ for all $ i\in[m]$ and $k\geq0$.  It is easy to check $\boldsymbol{\varpi}^{k}\in \mathbb{N}(2\delta)$ for all  $k\geq0$.  Moreover,
\begin{equation}\label{D-varpi-gamma}
\|\Delta \boldsymbol{\varpi}^{k+1}\|^2 = \left\| \frac{1}{m}\sum_{i=1}^m \Delta \w_i^{k+1}\right\|^2 \overset{\eqref{triangle-ineq}}{\leq} \frac{1}{m}\sum_{i=1}^m \left\|  \Delta \w_i^{k+1}\right\|^2 \overset{\eqref{limit-sigmak-gap-1}}{\leq} \dfrac{\sum_{i=1}^mc_i^2}{m}\dfrac{1}{\gamma^k}. 
\end{equation}
Therefore, by Lemma \ref{lemma:mean_convergence} 1),  \(\{ \boldsymbol{\varpi}^{k}\}\) converges to $\boldsymbol{\varpi}^{\infty}$ in the sense of $L^2$ convergence and \(\{ \E\boldsymbol{\varpi}^{k}\}\) converges, and they satisfy \eqref{E-w-k-L2}. By Lemma \ref{lemma:mean_convergence} 1) and condition \eqref{limit-sigmak-gap-1},  sequence  \(\{\mathbb{E}\mathbf{w}_i^k\}\) converges. 
It follows from \eqref{sig-t-w-v-1} that
	 \begin{equation*} 
			  \lim_{k\to\infty} \E\nl \w_i^{k} - \overline{\v}_i^{k} \nr^2 =0,
		\end{equation*}
which by  Jensen's inequality results in
	 \begin{equation*} 
		\lim_{k\to\infty} \nl \E\w_i^{k} - \E\overline{\v}_i^{k} \nr \leq   \lim_{k\to\infty} \E\nl \w_i^{k} - \overline{\v}_i^{k} \nr \leq 	  \lim_{k\to\infty} \sqrt{\E\nl \w_i^{k} - \overline{\v}_i^{k} \nr^2} =0.
		\end{equation*}
		Therefore,   $\lim_{k\to\infty}   \E\overline{\v}_i^{k} = \lim_{k\to\infty} \E\w_i^k$ for any $i\in[m]$, thereby
		$$\lim\limits_{{k} \rightarrow \infty} \E\W^{k} 
			= \lim\limits_{{k} \rightarrow \infty} \E\V^{k}.$$
Again by Jensen's inequality,  
	\begin{equation*}
	\begin{array}{rcl}
	  \lim\limits_{k\to\infty}   \|\E\w_i^k-\E\boldsymbol{\varpi}^{k}\| &\leq&     \lim\limits_{k\to\infty}  \sqrt{\mathbb{E}\|\w_i^k-\boldsymbol{\varpi}^{k}\|^2}\\[1ex]
	&\overset{\eqref{def-all-matrices}}{\leq}&   \lim\limits_{k\to\infty}  \sqrt{\mathbb{E}\|\W^k-\boldsymbol{\Pi}^k\|_F^2} \overset{\eqref{lemma3}}{=} 0.
\end{array}\end{equation*}
The above two facts suffice to
\begin{equation}\label{Ewi=Ew}
\lim_{k\to\infty} \E\w_i^k = \lim_{k\to\infty}\E\boldsymbol{\varpi}^{k}   =\E \boldsymbol{\varpi}^{\infty},\qquad \forall~i\in[m].
\end{equation}
 3) The convergence of sequence $\{\widetilde{H}^{{k}}\}$ and \eqref{def-tilde-H} can lead to the convergence of sequence $\{{H}^{{k}}\}$ and 
 \begin{equation}\label{H-H-f}
 \begin{array}{rcl}
 \lim\limits_{{k} \rightarrow \infty} \widetilde{H}^{k} =\lim\limits_{{k} \rightarrow \infty} H^{k}&\overset{\eqref{def-tilde-H}}{=}&\lim\limits_{{k} \rightarrow \infty} \E \dsum_{i=1}^{m}\left(
		f_i(\w_i^k) + \frac{\sigma^k t_i}{2}\nl\w_i^k - \overline{\v}_i^k\nr^2\right)\\
		&\overset{\eqref{sig-t-w-v}}{=}& \lim\limits_{{k} \rightarrow \infty} \E \dsum_{i=1}^{m} 
		f_i(\w_i^k).
 \end{array}
 \end{equation} 
 From 2), we have shown that \(\{ \boldsymbol{\varpi}^{k}\}\) converges (to $\boldsymbol{\varpi}^{\infty}$) in the sense of $L^2$ convergence, namely,
 $$ \lim_{k\to\infty}  \E \|\boldsymbol{\varpi}^{k} -\boldsymbol{\varpi}^{\infty}\|^2=0,$$
 which by Lemma \ref{lemma:mean_convergence} 2) contributes to
 $$ \lim\limits_{{k} \rightarrow \infty} \E   
		f(\boldsymbol{\varpi}^k)  = \E f(\boldsymbol{\varpi}^{\infty}),  $$
and
 \begin{equation}\label{E-wi-w-infty}
	\begin{array}{rcl}
	  \lim\limits_{k\to\infty}  \E \|\w_i^k-\boldsymbol{\varpi}^{\infty} \|^2 &\leq&     \lim\limits_{k\to\infty}  2\E \|\w_i^k-\boldsymbol{\varpi}^k \|^2 + \lim\limits_{k\to\infty}  2\E \|\boldsymbol{\varpi}^k -\boldsymbol{\varpi}^{\infty} \|^2 \\[2ex]
	   &=&     \lim\limits_{k\to\infty}  2\E \|\w_i^k-\boldsymbol{\varpi}^k \|^2\\[1ex]
	&\overset{\eqref{def-all-matrices}}{\leq}&   \lim\limits_{k\to\infty}  2\mathbb{E}\|\W^k-\boldsymbol{\Pi}^k\|_F^2 \overset{\eqref{lemma3}}{=} 0.
\end{array}\end{equation}
 This means sequence \(\{\mathbf{w}_i^k\}\) converges  to $\boldsymbol{\varpi}^{\infty}$ in the sense of $L^2$ convergence. Therefore, by Lemma \ref{lemma:mean_convergence} 2), we obtain
 \begin{equation*} 
 \lim\limits_{{k} \rightarrow \infty} \E 
		f_i(\w_i^k) = \E f_i(\boldsymbol{\varpi}^{\infty}),\qquad\forall~i\in[m].
 \end{equation*}
 This together with \eqref{H-H-f} gives rise to 
 \begin{equation*}  
 \lim\limits_{{k} \rightarrow \infty} \widetilde{H}^{k} =\lim\limits_{{k} \rightarrow \infty} H^{k} =  \lim\limits_{{k} \rightarrow \infty} \E \dsum_{i=1}^{m} 
		f_i(\w_i^k) = \E \dsum_{i=1}^{m} 
		f_i(\boldsymbol{\varpi}^{\infty}) = \E f(\boldsymbol{\varpi}^{\infty}). 
 \end{equation*} 
 This finishes the proof. \end{proof}

\subsection{\bf Proof of Theorem \ref{the:complexity}}

\begin{theorem}\label{the:complexity}
			Let $\{(\W^k, \V^k)\}$ be the sequence generated by Algorithm~\ref{algorithm-PaME} with Setup \ref{setup}.	Then  under Assumptions \ref{assump-gradientlip} and \ref{assump-double-stochasitic},
		\begin{equation*}
			 \E\|\W^{k} - \W^\infty\|^2_F= O(\gamma^{-{k}}) ,\qquad \E\|\V^{k} - \W^\infty\|^2_F = O(\gamma^{-{k}}),\qquad \E| f(\boldsymbol{\varpi}^{k}) -  f(\boldsymbol{\varpi}^\infty) | =O(\gamma^{-{k/2}}).
		\end{equation*}		
	\end{theorem}
\begin{proof}
	 The first result follows from \eqref{limit-sigmak-gap-1}, \eqref{E-wi-w-infty}, and Lemma \ref{lemma:mean_convergence} 3). The second result follows from 
	 \begin{equation*} 
	 \begin{array}{rcl}	 
\E\|\V^{k} - \W^\infty\|_F^2 &\leq &2	\E\|\W^{k} - \W^\infty\|^2_F + 2	\E\|\V^{k} - \W^k\|^2_F\\[2ex]
	&\overset{\eqref{lemma1}}{\leq}& 2	\E\|\W^{k} - \W^\infty\|^2_F +   \dfrac{8m}{t_{\min}}	\E\|\W^{k} - \boldsymbol{\Pi}^{k}\|^2_F\\ [2ex]
&\overset{\eqref{lemma3}}{\le}&	 2	\E\|\W^{k} - \W^\infty\|^2_F +   \dfrac{8mC_2}{t_{\min}}	  \dfrac{ 1 }{\beta^{k}} \\ [2ex]
&\le&	 2	\E\|\W^{k} - \W^\infty\|^2_F +   \dfrac{8mC_2}{t_{\min}}	  \dfrac{ 1 }{\gamma^{k}} \\ [2ex]
&=&
O(\gamma^{-{k}}).
		\end{array}\end{equation*}
The last result holds because of \eqref{D-varpi-gamma} and Lemma \ref{lemma:mean_convergence} 3).
\end{proof}

\subsection{\bf Proof of Theorem \ref{the:pame-gradient-stationarity}}

\begin{theorem}\label{the:pame-gradient-stationarity}
Let $\{(\W^{k}, {\mathbf{V}}^{k})\}$ be the sequence generated by Algorithm~\ref{algorithm-PaME} with Setup \ref{setup} and $T>0$ be the total number of iterations. Choose $\gamma$ to satisfy  \begin{equation}\label{condition-gamma-K}
\lim_{T\to\infty} \frac{\gamma-1}{1-\gamma^T} = 0.
\end{equation} Then under  Assumptions \ref{assump-double-stochasitic} and \ref{assump-gradientlip}, it holds
	\begin{equation*}\begin{aligned}
			&\E\nl\nabla f(\boldsymbol{\varpi}^{\infty})\nr^2= 0, \quad
			 \mathbb E
			\nl
			\nabla f(\boldsymbol{\varpi}^{T})
			\nr^2
			=
			O(\gamma^{-T}).
		\end{aligned}
	\end{equation*}		
\end{theorem}

\begin{proof}
	By the local update rule of Algorithm~\ref{algorithm-PaME}, for each
	$i\in[m]$, we have
	\begin{equation}\label{eq:grad-stat-local-update}
		\w_i^{k+1}
		=
		\overline{\v}_i^k
		-
		\frac{1}{\sigma^k m_i^k}\g_i^k,
		\qquad
		\g_i^k
		=
		\nabla f_i(\overline{\v}_i^k;\mathcal B_i^k).
	\end{equation}
	Hence,
	\begin{equation}\label{eq:grad-stat-residual}
		\g_i^k
		=
		\sigma^k m_i^k
		\left(
		\overline{\v}_i^k-\w_i^{k+1}
		\right).
	\end{equation}
	Taking conditional expectation with respect to the filtration
	$\mathcal F_k$ generated by the history before the stochastic gradient
	computation at iteration $k$, we obtain
	\begin{equation}\label{eq:grad-stat-full-gradient-residual}
		\nabla f_i(\overline{\v}_i^k)
		=
		\E\left[
		\sigma^k m_i^k
		\left(
		\overline{\v}_i^k-\w_i^{k+1}
		\right)
		\mid \mathcal F_k
		\right].
	\end{equation}
	Moreover,
	\begin{equation}\label{eq:grad-stat-residual-decomp}
		\overline{\v}_i^k-\w_i^{k+1}
		=
		\left(
		\overline{\v}_i^k-\w_i^k
		\right)
		+
		\left(
		\w_i^k-\w_i^{k+1}
		\right).
	\end{equation}
	Substituting \eqref{eq:grad-stat-residual-decomp} into
	\eqref{eq:grad-stat-full-gradient-residual} gives
	\begin{equation}\label{eq:grad-stat-local-gradient-decomp}
		\nabla f_i(\overline{\v}_i^k)
		=
		\E\left[
		\sigma^k m_i^k
		\left(
		\overline{\v}_i^k-\w_i^k
		\right)
		+
		\sigma^k m_i^k
		\left(
		\w_i^k-\w_i^{k+1}
		\right)
		\mid \mathcal F_k
		\right].
	\end{equation}
		Summing \eqref{eq:grad-stat-local-gradient-decomp} over $i=1,\ldots,m$
	and using Jensen's inequality, we have
	\begin{equation}\label{eq:grad-stat-local-sum-1}
			\E\nl
			\dsum_{i=1}^{m}
			\nabla f_i(\overline{\v}_i^k)
			\nr^2
			\leq
			2\E\nl
			\sigma^k
			\dsum_{i=1}^{m}
			m_i^k
			\left(
			\overline{\v}_i^k-\w_i^k
			\right)
			\nr^2
			+
			2\E\nl
			\sigma^k
			\dsum_{i=1}^{m}
			m_i^k
			\left(
			\w_i^k-\w_i^{k+1}
			\right)
			\nr^2.
	\end{equation}
It follows from $m_i^k\leq m_i\leq m$ that  
	\begin{equation}\label{eq:grad-stat-sum-vw}
		\nl
		\dsum_{i=1}^{m}
		m_i^k
		\left(
		\overline{\v}_i^k-\w_i^k
		\right)
		\nr^2
		\leq
		m^3
		\dsum_{i=1}^{m}
		\nl
		\overline{\v}_i^k-\w_i^k
		\nr^2,
	\end{equation}
	and
	\begin{equation}\label{eq:grad-stat-sum-dw}
		\nl
		\dsum_{i=1}^{m}
		m_i^k
		\left(
		\w_i^k-\w_i^{k+1}
		\right)
		\nr^2
		\leq
		m^3
		\dsum_{i=1}^{m}
		\nl
		\w_i^k-\w_i^{k+1}
		\nr^2.
	\end{equation}
	Combining \eqref{eq:grad-stat-local-sum-1},
	\eqref{eq:grad-stat-sum-vw}, and \eqref{eq:grad-stat-sum-dw}, we get
	\begin{equation}\label{eq:grad-stat-local-sum-2}
			\E\nl
			\dsum_{i=1}^{m}
			\nabla f_i(\overline{\v}_i^k)
			\nr^2
			\leq
			2m^3(\sigma^k)^2
			\dsum_{i=1}^{m}
			\E\nl
			\overline{\v}_i^k-\w_i^k
			\nr^2
			+
			2m^3(\sigma^k)^2
			\dsum_{i=1}^{m}
			\E\nl
			\w_i^k-\w_i^{k+1}
			\nr^2.
	\end{equation}
Next, since
	\begin{equation}\label{eq:grad-stat-global-gradient}
		\nabla f(\boldsymbol{\varpi}^k)
		=
		\dsum_{i=1}^{m}
		\nabla f_i(\boldsymbol{\varpi}^k),
	\end{equation}
	we have
	\begin{equation}\label{eq:grad-stat-global-decomp}
		\nabla f(\boldsymbol{\varpi}^k)
		=
		\dsum_{i=1}^{m}
		\nabla f_i(\overline{\v}_i^k)
		+
		\dsum_{i=1}^{m}
		\left(
		\nabla f_i(\boldsymbol{\varpi}^k)
		-
		\nabla f_i(\overline{\v}_i^k)
		\right).
	\end{equation}
	Let $\alpha_{\max}:=\max_{i\in[m]}\alpha_i$. By Assumption
	\ref{assump-gradientlip} and Jensen's inequality,
	\begin{equation}\label{eq:grad-stat-global-bound-1}
			\E\nl
			\nabla f(\boldsymbol{\varpi}^k)
			\nr^2
			\leq
			2\E\nl
			\dsum_{i=1}^{m}
			\nabla f_i(\overline{\v}_i^k)
			\nr^2
			+
			2m\alpha_{\max}^2
			\dsum_{i=1}^{m}
			\E\nl
			\boldsymbol{\varpi}^k-\overline{\v}_i^k
			\nr^2.
	\end{equation}
	Combining \eqref{eq:grad-stat-local-sum-2} and
	\eqref{eq:grad-stat-global-bound-1} yields
	\begin{equation}\label{eq:grad-stat-weighted-step}
			\dfrac{1}{\sigma^k}
			\E\nl
			\nabla f(\boldsymbol{\varpi}^k)
			\nr^2
			\leq
			4m^3\sigma^k
			\dsum_{i=1}^{m}
			\E\nl
			\overline{\v}_i^k-\w_i^k
			\nr^2
			+
			4m^3\sigma^k
			\dsum_{i=1}^{m}
			\E\nl
			\w_i^k-\w_i^{k+1}
			\nr^2
			+
			\dfrac{2m\alpha_{\max}^2}{\sigma^k}
			\dsum_{i=1}^{m}
			\E\nl
			\boldsymbol{\varpi}^k-\overline{\v}_i^k
			\nr^2.
	\end{equation}
	By Theorem~\ref{the:descent-property} and
	Lemma~\ref{bound-of-para},
	the sequence $\{\widetilde H^k\}$ is lower bounded and
	\begin{equation}\label{eq:grad-stat-delta-constant}
		\dsum_{k=1}^{\infty}
		\sigma^k
		\dsum_{i=1}^{m}
		\E\nl
		\w_i^{k+1}-\w_i^k
		\nr^2
		\leq
		\dfrac{8}{t_{\min}}
		\left(
		\widetilde H^1-\inf_k\widetilde H^k
		\right)
		=:
		C_{\Delta}
		<\infty .
	\end{equation}
	In addition, by \eqref{sig-t-w-v-1}, there exist
	constants $ \beta>\gamma>1, C_{\mathbf v}>0,$ and
	$C_{\boldsymbol{\varpi}}>0$ such that
	\begin{equation}\label{eq:grad-stat-vw-rate}
		\dsum_{i=1}^{m}
		\E\nl
		\overline{\v}_i^k-\w_i^k
		\nr^2
		\leq
		C_{\mathbf v}\beta^{-k},
	\end{equation}
	which further indicates that there exists 
	$C_{\boldsymbol{\varpi}}>0$ such that
	\begin{equation}\label{eq:grad-stat-varpi-v-rate}
		\dsum_{i=1}^{m}
		\E\nl
		\boldsymbol{\varpi}^k-\overline{\v}_i^k
		\nr^2 \leq  \frac{1}{m}\sum_{i=1}^m \E\nl
		  \w_i^{k}-\overline{\v}_i^k
		\nr^2
		\leq
		C_{\boldsymbol{\varpi}}\beta^{-k}.
	\end{equation}
	Since $\sigma^k=\sigma^0\gamma^k$ and $\beta>\gamma$, we have
	\begin{equation}\label{eq:grad-stat-vw-summable}
			\dsum_{k=1}^{\infty}
			\sigma^k
			\dsum_{i=1}^{m}
			\E\nl
			\overline{\v}_i^k-\w_i^k
			\nr^2
			\leq
			\sigma^0 C_{\mathbf v}
			\dsum_{k=1}^{\infty}
			\left(
			\dfrac{\gamma}{\beta}
			\right)^k
			=
			\sigma^0 C_{\mathbf v}
			\dfrac{\gamma/\beta}{1-\gamma/\beta}
			<\infty .
	\end{equation}
	Moreover,
	\begin{equation}\label{eq:grad-stat-varpi-v-summable}
			\dsum_{k=1}^{\infty}
			\dfrac{1}{\sigma^k}
			\dsum_{i=1}^{m}
			\E\nl
			\boldsymbol{\varpi}^k-\overline{\v}_i^k
			\nr^2
			\le
			\dfrac{C_{\boldsymbol{\varpi}}}{\sigma^0}
			\dsum_{k=1}^{\infty}
			\left(
			\dfrac{1}{\gamma\beta}
			\right)^k
			=
			\dfrac{C_{\boldsymbol{\varpi}}}{\sigma^0}
			\dfrac{1/(\gamma\beta)}{1-1/(\gamma\beta)}
			<\infty .
	\end{equation}
	Summing \eqref{eq:grad-stat-weighted-step} over $k$ and using
	\eqref{eq:grad-stat-delta-constant}, \eqref{eq:grad-stat-vw-summable},
	and \eqref{eq:grad-stat-varpi-v-summable}, we obtain
	\begin{equation}\label{eq:grad-stat-weighted-summable}
		\dsum_{k=1}^{\infty}
		\dfrac{1}{\sigma^k}
		\E\nl
		\nabla f(\boldsymbol{\varpi}^k)
		\nr^2
		\leq
		C_{\rm total}
		<\infty,
	\end{equation}
	where one may take
	\begin{equation}\label{eq:grad-stat-C-total}
			C_{\rm total}
			:=
			4m^3 C_{\Delta}
			+
			4m^3\sigma^0 C_{\mathbf v}
			\dfrac{\gamma/\beta}{1-\gamma/\beta}
			+
			\dfrac{2m\alpha_{\max}^2 C_{\boldsymbol{\varpi}}}{\sigma^0}
			\dfrac{1/(\gamma\beta)}{1-1/(\gamma\beta)} .
	\end{equation}
Since $\beta>\gamma>1$, there exists a constant $t>1$ such that
	$\beta\geq t\gamma$. Therefore,
	\begin{equation}\label{eq:grad-stat-ratio-bound}
		\dfrac{\gamma/\beta}{1-\gamma/\beta}
		\leq
		\dfrac{1}{t-1},
		\qquad
		\dfrac{1/(\gamma\beta)}{1-1/(\gamma\beta)}
		=
		\dfrac{1}{\gamma\beta-1}
		\leq
		\dfrac{1}{t-1}.
	\end{equation}
	Thus, $C_{\rm total}$ can be bounded by a constant independent of
	$\gamma$, namely,
	\begin{equation}\label{eq:grad-stat-C-uniform}
		C_{\rm total}
		\leq
		4m^3 C_{\Delta}
		+
		\dfrac{4m^3\sigma^0 C_{\mathbf v}}{t-1}
		+
		\dfrac{2m\alpha_{\max}^2 C_{\boldsymbol{\varpi}}}
		{\sigma^0(t-1)}
		=:
		C .
	\end{equation}
	Let
	\begin{equation}\label{eq:grad-stat-SK}
		S_T
		:=
		\dsum_{k=1}^{T}
		\dfrac{1}{\sigma^k}.
	\end{equation}
	Since $\sigma^k=\sigma^0\gamma^k$, we have
	\begin{equation}\label{eq:grad-stat-SK-form}
		S_T
		=
		\dsum_{k=1}^{T}
		\dfrac{1}{\sigma^0\gamma^k}
		=
		\dfrac{1-\gamma^{-T}}{\sigma^0(\gamma-1)}.
	\end{equation}
	Then, from \eqref{eq:grad-stat-weighted-summable}, we obtain
	\begin{equation}\label{eq:grad-stat-weighted-average}
		\dfrac{
			\dsum_{k=1}^{T}
			\dfrac{1}{\sigma^k}
			\E\nl
			\nabla f(\boldsymbol{\varpi}^k)
			\nr^2
		}{
			S_T
		}
		\leq
		\dfrac{C}{S_T}
		=		
		\dfrac{C\sigma^0(\gamma-1)}{1-\gamma^{-T}} .
	\end{equation}
	Equivalently, let $R_T$ be sampled from $\{1,2,\ldots,T\}$ according to
	\begin{equation}\label{eq:grad-stat-random-index}
		\mathbb P(R_T=k)
		=
		\dfrac{1/\sigma^k}{S_T},
		\qquad
		k=1,2,\ldots,T,
	\end{equation}
	then by \cite{Ghadimi13}, 
	\begin{equation}\label{eq:grad-stat-random-bound}
		\E\nl
		\nabla f(\boldsymbol{\varpi}^{R_T})
		\nr^2
		\leq
		\dfrac{C\sigma^0(\gamma-1)}{1-\gamma^{-T}} .
	\end{equation}
	We next show that the whole sequence converges to a stationary point with a
	linear rate in terms of the gradient norm. From
	\eqref{eq:grad-stat-random-bound} and condition \eqref{condition-gamma-K}, we have
	\begin{equation}\label{eq:grad-stat-random-zero}
		\E\nl
		\nabla f(\boldsymbol{\varpi}^{R_T})
		\nr^2
		\rightarrow 0,
		\qquad T\rightarrow\infty .
	\end{equation}
	This implies that there exists a subsequence
	$\{\boldsymbol{\varpi}^{k_j}\}_{j\geq1}$ with $k_j\rightarrow\infty$ such
	that
	\begin{equation}\label{eq:grad-stat-subsequence-zero}
		\E\nl
		\nabla f(\boldsymbol{\varpi}^{k_j})
		\nr^2
		\rightarrow 0.
	\end{equation}
	On the other hand, by the linear convergence result in
	Theorem~\ref{the:complexity}, there exist a constant 
	$C_{\infty}>0$ such that
	\begin{equation}\label{eq:grad-stat-varpi-linear-limit}
		\E\nl
		\boldsymbol{\varpi}^{k}
		-
		\boldsymbol{\varpi}^{\infty}
		\nr^2
		\leq \frac{1}{m}\sum_{i=1}^m \E\nl
		 \w_i^{k}
		-
		\boldsymbol{\varpi}^{\infty}
		\nr^2
		\leq
		C_{\infty}\gamma^{-k},
		\qquad k\geq1 .
	\end{equation}
	Let $\alpha$ denote the Lipschitz constant of $\nabla f$. By the
	Lipschitz continuity of $\nabla f$, we have
	\begin{equation}\label{eq:grad-stat-limit-bound}
		\begin{array}{rcl}
			\E\nl
			\nabla f(\boldsymbol{\varpi}^{\infty})
			\nr^2
			&\leq&
			2\E\nl
			\nabla f(\boldsymbol{\varpi}^{k_j})
			\nr^2
			+
			2\E\nl
			\nabla f(\boldsymbol{\varpi}^{k_j})
			-
			\nabla f(\boldsymbol{\varpi}^{\infty})
			\nr^2
			\\[2ex]
			&\leq&
			2\E\nl
			\nabla f(\boldsymbol{\varpi}^{k_j})
			\nr^2
			+
			2\alpha^2
			\E\nl
			\boldsymbol{\varpi}^{k_j}
			-
			\boldsymbol{\varpi}^{\infty}
			\nr^2 .
		\end{array}
	\end{equation}
	Letting $j\to\infty$ and using
	\eqref{eq:grad-stat-subsequence-zero} and
	\eqref{eq:grad-stat-varpi-linear-limit}, we obtain
	\begin{equation}\label{eq:grad-stat-limit-stationary}
		\E\nl
		\nabla f(\boldsymbol{\varpi}^{\infty})
		\nr^2
		=
		0.
	\end{equation}
	Therefore, the limit point $\boldsymbol{\varpi}^{\infty}$ is stationary in
	the $L^2$ sense. Finally, for any given $T\geq1$, by \eqref{eq:grad-stat-limit-stationary} and
	the Lipschitz continuity of $\nabla f$, we have
	\begin{equation}\label{eq:grad-stat-full-sequence-rate-proof}
		\begin{array}{rcl}
			\E\nl
			\nabla f(\boldsymbol{\varpi}^{T})
			\nr^2
			&=&
			\E\nl
			\nabla f(\boldsymbol{\varpi}^{T})
			-
			\nabla f(\boldsymbol{\varpi}^{\infty})
			\nr^2
			\\[2ex]
			&\leq&
			\alpha^2
			\E\nl
			\boldsymbol{\varpi}^{T}
			-
			\boldsymbol{\varpi}^{\infty}
			\nr^2
			\\[2ex]
			&\leq&
			\alpha^2 C_{\infty}\gamma^{-T}.
		\end{array}
	\end{equation}
	Hence, $\E\nl
		\nabla f(\boldsymbol{\varpi}^{T})
		\nr^2
		=
		O(\gamma^{-T}).$ 
	This completes the proof.
\end{proof}

%

\end{document}